\documentclass[11pt,a4paper]{scrartcl}
 \usepackage{natbib}
\usepackage[labelformat=empty]{subcaption}
\usepackage{comment}
\usepackage{lipsum}

\usepackage[font=small,labelfont=bf]{caption}
\usepackage
[
        a4paper,
        left=3cm,
        right=3cm,
        top=3cm,
        bottom=3cm,
]
{geometry}

\usepackage{setspace}
\usepackage{graphicx}
\graphicspath{{./figures}}

\usepackage{multicol,latexsym, array}
\usepackage{textcomp}
\usepackage[ansinew]{inputenc}
\usepackage{amsmath,amssymb, amsthm}
\usepackage[OT1]{fontenc}
\usepackage[english]{babel}
\usepackage{hyperref}

\hypersetup{
   colorlinks,
   linkcolor={blue},
   citecolor={blue},
   urlcolor={blue}
}

\usepackage{enumitem}
\usepackage{MnSymbol}
\usepackage{empheq}
\usepackage{rotating}
\usepackage{titletoc}
\usepackage{framed}
\usepackage{multirow,bigdelim}
\usepackage{bbm}
\usepackage[noend]{algpseudocode}
\usepackage{algorithm}
\usepackage{pgf}
\usepackage{pgfplots}
\usepackage{tikz}
\pgfplotsset{compat=1.17}
\usepackage{shadethm}
\usepackage{thmtools}
\usepackage{mdframed}
\usepackage{lipsum}
\usepackage{fixmath}
\usepackage{multicol}
\usepackage{footmisc}
\usepackage{microtype}
\usepackage{macros}

\usepackage{caption}
\usepackage{xargs}  
\usepackage{xcolor} 

\definecolor{DarkGray}{RGB}{90,90,90}

\usepackage[colorinlistoftodos,prependcaption,textsize=tiny]{todonotes}
\usepackage{cleveref}

\date{\today}

    {
    \paragraph{Acknowledgements:} 
    }

\usepackage{caption}
\usepackage{mathrsfs}
\usetikzlibrary{arrows}

\usepackage{xpatch}

\makeatletter
\xpatchcmd{\endmdframed}
  {\aftergroup\endmdf@trivlist\color@endgroup}
  {\endmdf@trivlist\color@endgroup\@doendpe}
  {}{}
\makeatother

\usepackage{nicefrac}
\usepackage{dsfont}

\usepackage{bigints}
\usepackage{mathtools}
\usepackage{fixltx2e}

\usepackage{mathtools}
\usepackage{color}
\usepackage{fancyhdr}
\usepackage{lastpage}

\usepackage{ifthen}

\newcommand{\SNN}{\operatorname{SNN}}

\newcommand{\coloneqq}{:=}

\renewcommand{\phi}{\varphi}

\newcommand{\bbf}{\bold{f}}

\numberwithin{equation}{section}

\newcommand{\footremember}[2]{
	\footnote{#2}
	\newcounter{#1}
	\setcounter{#1}{\value{footnote}}
}

\newcommand{\footrecall}[1]{
	\footnotemark[\value{#1}]
}

\newif\ifexclude
\excludefalse  
\newcommand{\excludePart}[1]{\ifexclude \else #1 \fi}

\DeclareRobustCommand{\texttt}[1]{\textup{\ttfamily #1}}

\theoremstyle{plain}
\newtheorem{theorem}{Theorem}[section]
\newtheorem{lemma}[theorem]{Lemma}
\newtheorem{example}[theorem]{Example}
\newtheorem{corollary}[theorem]{Corollary}
\newtheorem{proposition}[theorem]{Proposition}

\theoremstyle{remark}
\newtheorem{remark}[theorem]{Remark}

\theoremstyle{definition}
\newtheorem{definition}[theorem]{Definition}

\begin{document}
\title{On the Universal Representation Property of Spiking Neural Networks}
\author{
  \begin{tabular}{cc}
\begin{tabular}{c}
  Shayan Hundrieser\footremember{twente}{\scriptsize \;\;Department of Applied Mathematics, University of Twente, Enschede, The Netherlands} \\[-1ex]
  \footnotesize{\href{mailto:s.hundrieser@utwente.nl}{s.hundrieser@utwente.nl}}
\end{tabular}&
\begin{tabular}{c}
  Philipp Tuchel\footremember{bochum}{\scriptsize \;\;Faculty of Mathematics, Ruhr University Bochum, Bochum, Germany} \\[-1ex]
  \footnotesize{\href{mailto:philipp.tuchel@ruhr-uni-bochum.de}{philipp.tuchel@ruhr-uni-bochum.de}}
\end{tabular}\\[4ex]
\begin{tabular}{c}
  Insung Kong\footrecall{twente} \\[-1ex]
  \footnotesize{\href{mailto:insung.kong@utwente.nl}{insung.kong@utwente.nl}}
\end{tabular}&
\begin{tabular}{c}
  Johannes Schmidt-Hieber\footrecall{twente} \\[-1ex]
  \footnotesize{\href{mailto:a.j.schmidt-hieber@utwente.nl}{a.j.schmidt-hieber@utwente.nl}}
\end{tabular}
\end{tabular}\vspace{0.2cm}
}

\date{}
\maketitle

\vspace{-0.5cm}

\begin{abstract}
\noindent 
Inspired by biology, spiking neural networks (SNNs) process information via discrete spikes over time, offering an energy-efficient alternative to the classical computing paradigm and classical artificial neural networks (ANNs). In this work, we analyze the representational power of SNNs by viewing them as sequence-to-sequence processors of spikes, i.e., systems that transform a stream of input spikes into a stream of output spikes. We establish the universal representation property for a natural class of spike train functions.  Our results are fully quantitative, constructive, and near-optimal in the number of required weights and neurons. The analysis reveals that SNNs are particularly well-suited to represent functions with few inputs, low temporal complexity or compositions of such functions. The latter is of particular interest, as it indicates that deep SNNs can efficiently capture composite functions via a modular design. As an application of our results, we discuss spike train classification. Overall, these results contribute to a rigorous foundation for understanding the capabilities and limitations of spike-based neuromorphic~systems.
\end{abstract}

\excludePart{
  \vspace{0.5cm}
  
  }


%

%

\section{Introduction}
Brain-inspired (neuromorphic) computing is a broad and extremely active field \citep{PfeifferPfeil2018, FrenkelEtAl2023, braininspiredSurvey}. Interest in this domain is largely driven by the limitations of classical computers, specifically the high energy demand caused by the \emph{von Neumann bottleneck}. This bottleneck arises because memory and processing units are physically separated and must constantly exchange data; while processing speeds have increased exponentially, memory access latencies have not improved at the same pace \citep%
{hennessy2011computer}.

To address this, neuromorphic chips aim to more closely mimic the human brain by combining memory and processing within the same units (neurons), processing information in a distributed and localized manner. A milestone in the development of neuromorphic computing is the 2023 release of the IBM NorthPole chip that, in some scenarios, has been reported to significantly outperform comparable GPUs both in terms of processing speed and energy consumption \citep{NorthPolePaper}. 

\emph{Spiking neural networks} (SNNs)  offer significant energy efficiency. Unlike classical systems that require a steady current, SNNs operate on an event-driven basis, consuming energy only when a spike occurs. While current SNN performance still lags behind standard Artificial Neural Networks (ANNs), recent advancements aim to bridge this gap \citep{8891809, Yin2021, Stanojevic2024, zhu2024spikegpt}. ANNs have received tremendous attention from mathematicians in the past years and the advancement of neuromorphic computing asks for a similar program to unravel the properties of SNNs.

A central point in pursuing this topic concerns the decoding scheme: Should information be encoded in the relative frequency of spikes or in the precise timing of the spikes?
Empirical evidence in biological systems argues for the latter. 
 The argument rests on the processing-speed in biological systems: the primate brain can recognize objects within 100ms-150ms. If this requires 10 layers, then every layer has 10ms-15ms to process the signal. During such a short time span, a neuron can release at most two spikes, see \cite{THORPE2001715, TAHERKHANI2020253} for references. We therefore adopt the viewpoint that the coding is in the spiking times of the neurons, also referred to as the spike trains of the neurons. 

In this work, we consider feedforward SNNs consisting of an input layer, potentially multiple hidden layers, and an output layer. The input neurons receive input spike trains from the environment. All other neurons have an associated membrane potential that is formalized in \Cref{sec:SNN_formalization}. Spikes occur if the membrane potential of the neuron exceeds a fixed threshold value (integrate-and-fire). SNNs can therefore be viewed as functions mapping spike trains to spike trains. This work investigates which  spike train to spike train functions can be (efficiently) represented by SNNs.

\begin{figure}[t!]
    \centering
    \includegraphics[width=\linewidth]{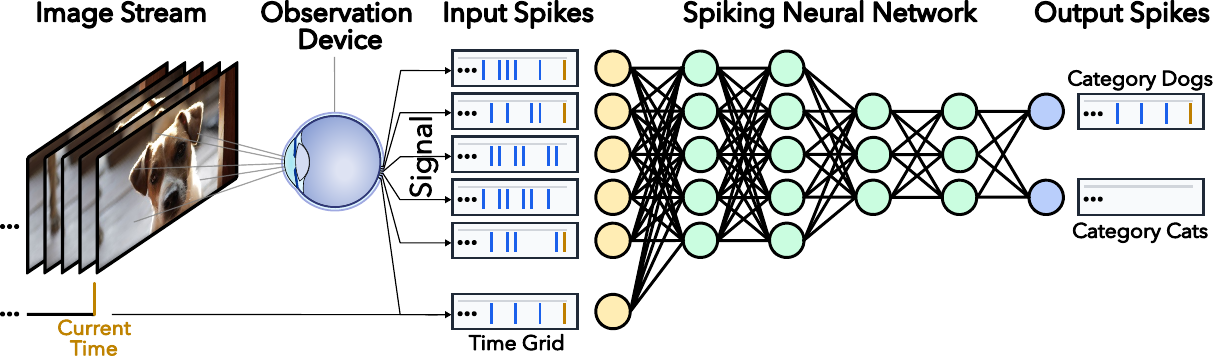}
    \caption{\textbf{Schematic depiction of a spiking neural network acting as a classifier}. An image stream is captured by a device (here represented by an eye) and encoded into input spike trains, where higher input signals increase the probability of a spike in the corresponding spike train.  These signals are processed by the SNN over a time grid. The resulting output spikes determine the class label, distinguishing between categories such as `Dogs' and `Cats'. In \Cref{sec:classification_spike_trains}, we provide an explicit architecture for this.
    }
    \label{fig:classifier_architecture}
\end{figure}

\subsection{Contributions: Expressiveness of SNNs}

Most of the existing mathematical theory on SNNs has been derived for time-to-first-spike
(TTFS) coded SNNs. TTFS assumes that the information is encoded in the time lag
until a neuron spikes for the first time, see \Cref{subsec:related_work}. In contrast, we work under the \emph{sequence-to-sequence (STS)} coding paradigm. This approach interprets SNNs as a system mapping input spike trains to output spike trains, thereby admitting the capacity to process heterogeneous temporal data. Our main contributions can be summarized as follows: %

\begin{itemize}

\item \textbf{Universal Representation of Temporal Functions:} We show that any spike train to spike train function,
that is \emph{causal} (the output at time $t$ only depends on the input up to time $t$) and exhibits some notion of \emph{finite memory},
can be \emph{exactly} represented by an SNN (see \Cref{thm:universal_representation_single_input} and \Cref{thm:function_representation} for the precise statements). The proof is constructive and provides insights into the expressiveness of feedforward SNNs. The result is a structural analog of the universal approximation theorem for ANNs \citep{cybenko1989approximation, hornik1991approximation, leshno1993multilayer, pinkus1999approximation}.

\item \textbf{Complexity Analysis and Lower Bounds:} We establish near-optimality of the required network size by providing nearly matching lower bounds. Specifically, for functions depending on the previous $m$ spikes and a single input, the construction requires $O(\sqrt{m})$ neurons (\Cref{thm:universal_representation_single_input}), whereas for $d\geq 2$ input neurons we require $2^{O(md)}$ neurons (\Cref{thm:function_representation}). Both bounds are optimal in a worst-case sense up to lower order terms 
(Theorems \ref{cor_lo_numpara_train} and \ref{cor_lo_numpara_functional}). This shows an interesting separation between  single-input and  multiple-input SNNs. Refined results are obtained under additional sparsity assumptions on the functions. 

\item \textbf{Compositional Functions can Avoid the Curse of Dimensionality:} While the worst-case bound for the number of neurons is exponential in the dimension, we demonstrate that functions possessing a \emph{compositional structure} can be represented with significantly fewer neurons (\Cref{thm:function_representation_composite}). This result is inspired by analogous results for ANNs \citep{10.1214/19-AOS1875, 10.1214/20-AOS2034} and offers theoretical intuition regarding which functions are ``simple'' or ``difficult'' for SNNs to learn, confirming that SNNs are particularly efficient at representing compositional functions composed of short memory or low-dimensional functions.

\item \textbf{Classification and Examples:} We showcase various relevant functions which are representable by SNNs (\Cref{sec:examples_functions}). This includes among others Boolean logical gates, periodic functions, spike amplifiers, and memory mechanisms. 
Furthermore, we discuss classification tasks, demonstrating how SNNs could be used to classify spike trains (\Cref{sec:classification_spike_trains}).  

\end{itemize}
 Collectively, these results confirm that the membrane potential dynamics of a spiking neuron naturally induce a form of temporal memory. %
This reinforces the potential of neuromorphic systems to overcome the von Neumann bottleneck by integrating memory and computation within the neuron itself.

\subsection{Related work on (Spiking) Neural Networks}\label{subsec:related_work}

The previously considered time-to-first-spike (TTFS) coded SNNs are functions that map vectors with non-negative entries to vector with non-negative entries. Distinctive features of this SNN model are that every input neuron can spike at most once and that every neuron has its own learnable delay parameter that models the travel time of the signal between different neurons. Computation of Boolean functions for a specific class of TTFS coded SNNs has been considered in \cite{MAASS19971659}. This model also takes the refractory period after a spike into account. Given that every input node can emit at most one spike, one wonders whether it is possible to configure the SNN weights in such a way that the output spikes if and only if the Boolean function returns one (the output is not TTFS coded). While all weights can be chosen to be one (\cite{MAASS19971659}, p.\ 1663) only the delay parameters have to be adjusted. It is shown that for specific Boolean functions, one SNN unit is sufficient.   For the same task, the number of hidden units in an artificial neural network needs to grow polynomially in the input size. Under a local linearity assumption, \cite{MaasUAP1997} shows that single-spike SNNs can arbitrarily well approximate ANNs with clipped ReLU activation function $\sigma_{\gamma}(x)=\min(\gamma, \max(x,0)),$ where the upper bound $\gamma>0$ is a hyperparameter. This implies then that this class of SNNs satisfies the universal approximation property. These results are also summarized in \cite{NIPS1996_f93882cb}. \cite{MAASS199926} show that the VC dimension of a TTFS coded SNN is quadratic in the number of learnable delays. \cite{schmitt1999vc} derives the refined VC dimension bound $O(WL\log(WL))$ with $W$ the number of parameters and $L$ the depth.

In a recent line of work, \cite{singh2023expressivity} proposes to use an encoder to map the inputs to the first spiking times of the SNN and a decoder to map them back to the outputs. It is shown that this specific class of SNNs generates continuous piecewise linear functions and that they can represent deep ReLU networks with one input.  \cite{2024arXiv240404549N} demonstrate how such SNNs with only positive weights can approximate standard function classes with optimal approximation rates. The restriction to non-negative weights seems to cause a logarithmic growth of the metric entropy of deeper SNNs, see \cite{2024arXiv240404549N}, Remark 20. 

\cite{2025arXiv250518023N} considers a temporal dynamic where the activations and the potential are updated in each time step. Via an encoder and decoder this is then transformed into a vector-to-vector map. They also prove a universal approximation theorem and study the number of regions on which these SNNs are piecewise constant. Similarly, \cite{2025arXiv250414015D, Fishell2025.10.31.685901} study the complexity of this class of SNNs through the lens of a new concept called `causal pieces'. \cite{2025arXiv250302013F} provides an overview of the mathematical results on this TTFS model class. 

It is important to highlight the main differences between TTFS coded SNNs and our approach. TTFS coded SNNs are functions whereas in this work, we think of SNNs as operators that map spike trains to spike trains. As a consequence we obtain a universal representation theorem instead of a universal approximation theorem. While  TTFS coding relies heavily on the learnable delay parameters, the only parameters in our approach are the network weights and firing happens if the membrane potential exceeds a given threshold. This models the behavior of cortical integrate-and-fire neurons. Because of the possibility to use gradients, more recent papers such as \cite{2024arXiv240404549N} focus instead on SNNs with ReLU type activations. 

Gelenbe neural networks are another abstraction of SNNs. Distinctive properties are integer-valued potentials and spiking of the neurons at random times. The signal is assumed to be in the rate/intensity at which the neurons spike. For these networks, the universal approximation property has been derived in \cite{GelenbeEtAl1999,GelenbeEtAl1999b}. 

While not directly linked to this article, a considerable amount of work has been devoted to the learning of SNNs, see e.g. 
\cite{TAHERKHANI2020253} for a summary. Moreover, an excellent introductory overview of SNNs is provided by \cite{2023arXiv230310780M}.

\subsection{Notation}

We introduce some basic mathematical notation here. We set $\RR_+\coloneqq (0, \infty)$, $\RRpbar \coloneqq [0, \infty]$,  $\NN \coloneqq \{1, 2, \dots\}$, and $\NN_0 \coloneqq \{0\} \cup \NN$. 
Furthermore, for $n\in \NN$ we write $[n]\coloneqq \{1, \dots, n\}$ and, by extension,  write $[\infty] \coloneqq \NN$. The cardinality of a set  $A$ is denoted by $|A|\in \NN_0 \cup \{\infty\}$, and we write $2^{A}$ to denote its power set. For a real number $a\in \RR$ we write $a_+ \coloneqq \max(0, a)$. Given two sets $A,B$ in $\RR^d$ we denote by $A\oplus B \coloneqq \{a + b: a\in A, b\in B\}$ their Minkowski sum (sumset). The indicator function is denoted by $\mathds{1}$. Given sequences $(a_n)_{n\in \NN},(b_n)_{n\in \NN}\in (0,\infty)$, we write $b_n = O(a_n)$ (resp. $b_n = \Omega(a_n)$) if there exists a universal constant $C>0$ such that $b_n \leq C a_n$ (resp.\ $b_n \geq C a_n$) for all $n\in \NN$. Throughout the whole paper, we use the convention that the logarithm is considered to be base 2, i.e., $\log(x) = \log_2(x)$ for $x>0$.

\section{Spiking Neural Networks (SNNs)}\label{sec:SNN_formalization}

We first formalize the spiking mechanism in terms of the membrane potential of a single neuron (Section \ref{subsec:membrane_potential}) and subsequently explain how multiple neurons can be combined to form a feedforward SNN (Section \ref{subsec:feed_forward_SNN}). Finally, we provide a perspective to recast an SNN as a recurrent artificial neural network, which involves Heaviside and clipped ReLU activation functions (Section \ref{subsec:connection_SNN_ANN}). %
\begin{definition}\label{def:spike_train}
    A \emph{spike train} $\calI\subseteq (0, \infty)$ is a countable set of points (modeling spike times) such that  any bounded time interval contains only finitely many elements, i.e.,  $|\calI\cap [0, t]|< \infty$ for every $t\in (0, \infty)$. Equivalently, $\calI \subseteq (0, \infty)$ is a spike train if it has no accumulation points. Assuming $\calI \neq \emptyset$, we further define for $j\in \{1, 2, \dots, |\calI|\}$, $$\calI_{(j)}:= \tau_j \quad \text{ where } \calI = \{\tau_1, \tau_2, \dots \} \text{ with } \tau_1 < \tau_2 < \dots %
    $$ as the $j$-th spike of $\calI$. The class of all spike trains is denoted by $\mathbb{T}.$ 
\end{definition}

\subsection{Neuronal spiking model}\label{subsec:membrane_potential}

For a description of the spiking mechanism, we consider one output (postsynaptic) neuron receiving signal from $d$ input (presynaptic) neurons; see Figure~\ref{fig:snn_CUnit_model} for a schematic representation. The inputs emit spikes at certain times which are then weighted with real-valued coefficients $w_1, \dots, w_d$, each associated with one edge, and contribute to a certain membrane potential of the output neuron. Positive weights correspond to excitatory connections, and negative weights correspond to inhibitory connections. However, the membrane potential cannot attain negative values.

\begin{figure}[t!]
    \centering
    \includegraphics[width=\textwidth, trim = {1mm 1mm 1mm 1mm},clip]{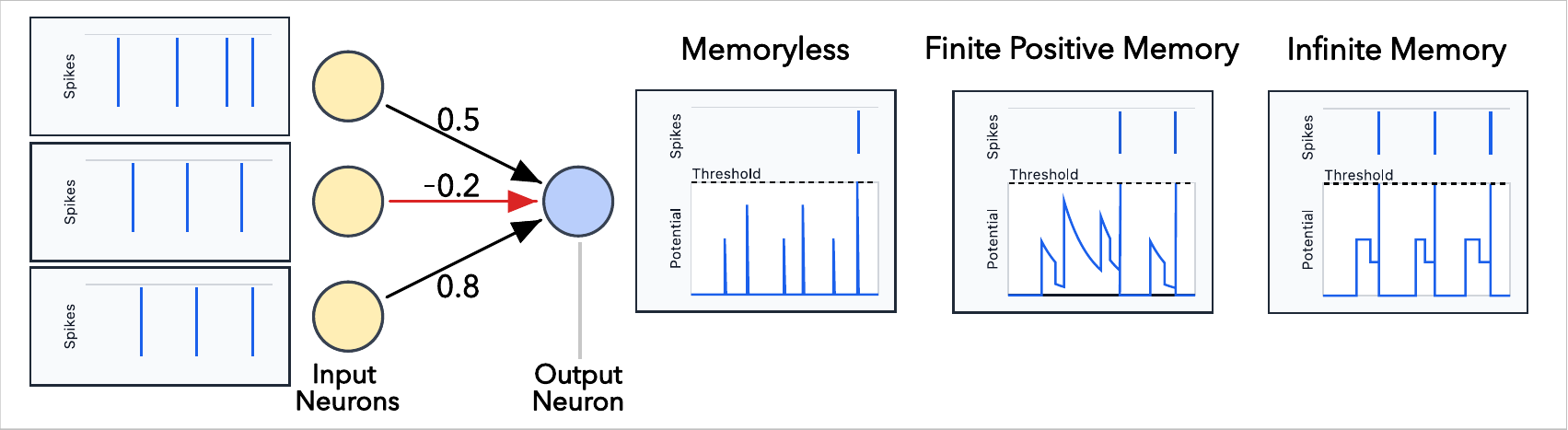}
    \caption{\textbf{%
    A neuron receiving signal from three input neurons and its potential.} The spike trains of all neurons and the membrane potential of the output neuron are displayed for no memory ($h = 0$), finite positive memory ($h \in (0, \infty)$), and infinite memory ($h = \infty$). Each input neuron emits spikes that are weighted and stored inside the membrane potential at the output neuron. A postsynaptic spike occurs when the membrane potential strictly exceeds a threshold value after which the potential is reset to zero. %
    } 
    \label{fig:snn_CUnit_model}
\end{figure}

For each presynaptic neuron $N_j$, with $j\in \{1, \dots, d\}$, denote by $\calI_j\subseteq (0, \infty)$ the corresponding spike train. 
The membrane potential $P\colon [0, \infty) \to [0, \infty)$ is recursively defined as follows.  
For $t = 0$, we set $P(0) \coloneqq 0$. For $t >0$ and presynaptic spikes $\calI_1, \dots, \calI_d$, we define $$\prev{t} \coloneqq \prev{t}(\calI_1, \dots, \calI_d,t) := \max\{s \in (\{0\} \cup \calI_1\cup \dots \cup \calI_d), s<t\}$$ 
as the last presynaptic spike time before $t.$ If no presynaptic spike occurred, this value is set to zero. In particular, since $\calI_1, \dots, \calI_d$ are spike trains, $\prev{t}$ is well-defined and satisfies $\prev{t} <t$ for every $t>0$. 
 Subsequently, the \emph{membrane potential} $ {P}\colon [0, \infty) \to [0, \infty)$ and the \emph{output spike train} $\calS\in \TT$ are defined via the equations 
\begin{align}\label{eq:potential_definition}
    P(t) :=& \bigg(  \mathds{1}\big({P}(\prev{t})\leq 1\big)  {P}(\prev{t})e^{-\frac{t-\prev{t}}h}   + \sum_{j : t \in \calI_j} w_j\bigg)_{+},\quad \calS \coloneqq  \big\{s\in (0, \infty) \colon P(s)>1\big\}.
\end{align}
The definition of $P(t)$ is recursive, as it depends on the function value at $t^{-}<t.$ If the memory coefficient $h$ is large, then the potential decays slowly between input spikes, meaning that spikes from the distant past still have a significant impact on the current potential. Conversely, if $h$ is small, then the potential decays quickly and only recent spikes have a significant impact on the potential. To extend the formalism to $h = 0$ and $h = \infty$, we use the convention that 
\begin{align*}
    \exp\Big(-\frac{x}{h}\Big) := \begin{cases}
        0, &\text{if } h = 0 \text{ and } x>0,\\
        1, &\text{if } h = \infty \text{ or } h= x = 0. \end{cases}
\end{align*}

The above construction is formally well-defined as confirmed in \Cref{prop:potential_well_defined} below, and implies that an output spike can only occur at the same time as an input spike, we refer this to as \emph{domination}. 

\begin{definition}[Domination]\label{def:domination}
    A spike train $\calI\subseteq (0, \infty)$ is said to be dominated by another spike train $\calI'\subseteq (0, \infty)$ if $\calI \subseteq \calI'$, and this is denoted as $\calI \dom \calI'$. Further, given a spike train $\calI$, we denote by $\TTT{\calI}$ the set of all spike trains that are dominated by $\calI$. 
\end{definition}

\begin{proposition}\label{prop:potential_well_defined}
The membrane potential $P\colon [0, \infty) \to [0,\infty)$ in \eqref{eq:potential_definition} is well-defined. Further, $\calS$ in \eqref{eq:potential_definition} is dominated by $\bigcup_{i \in [d]} \calI_i$ and, thus, is a spike train by \Cref{def:spike_train}.
\end{proposition}

The proof of \Cref{prop:potential_well_defined} is given in Section \ref{app:proofs_section2}. Notably, 
replacing $P(s)>1$ in \eqref{eq:potential_definition} by $P(s)\geq 1$ constraints the expressiveness of the SNN when the interarrival time of the input spikes is large, see Remark \ref{rmk:strict_threshold}. %
The above spiking mechanism deviates from biological neurons in the sense that, immediately after a spike occurred, the neuron can build up potential and spike again. Incorporating such a refractory period likely delimits the expressiveness of SNNs and further complicates the mathematical analysis. 

\subsection{Feedforward SNNs}\label{subsec:feed_forward_SNN}

To formalize feedforward SNNs we introduce the SNN (computational) unit. This recasts the nodewise computation of an output spike train from the input spike trains as a function. %

\begin{definition}[SNN computational unit] \label{def_snn_CUnit}    
    For memory coefficient $h\in [0,\infty]$ and weight vector $\mathbf{w} = (w_1,\ldots,w_p)^{\top} \in \RR^p$, we define the \textit{SNN (computational) unit} $\phi_{h,\mathbf{w}}: \mathbb{T}^p \to \mathbb{T}$ as  
    $$\phi_{h,\mathbf{w}}(\calI_1, \dots, \calI_p) := \calS
    $$
    with the output spike train $\calS$ defined by \eqref{eq:potential_definition}.
\end{definition}

Several SNN units can be arranged into an SNN layer. For a given input dimension $p_{\operatorname{in}}$ and output dimension $p_{\operatorname{out}}$,
the SNN layer is characterized by a weight matrix $W := (\mathbf{w}_1, \ldots, \mathbf{w}_{p_{\operatorname{out}}})^{\top} \in \RR^{p_{\operatorname{out}} \times p_{\operatorname{in}}}$,
where the $i$-th component of $\mathbf{w}_j$ represents the weight of the connection from the $i$-th input neuron to the $j$-th output neuron.

\begin{definition}[SNN layer]
    For $h \in [0,\infty]$ and 
    a weight matrix $W := (\mathbf{w}_1, \ldots, \mathbf{w}_{p_{\operatorname{out}}})^{\top} \in \RR^{p_{\operatorname{out}} \times p_{\operatorname{in}} }$,
    we define the \textit{SNN layer} $\boldsymbol{\psi}_{h,W}: \mathbb{T}^{p_{\operatorname{in}}} \to \mathbb{T}^{p_{\operatorname{out}}}$ as
    $$\boldsymbol{\psi}_{h,W} (\calI_1, \dots, \calI_{p_{\operatorname{in}}}) := 
    \Big(\phi_{h,\mathbf{w}_1}(\calI_1, \dots, \calI_{p_{\operatorname{in}}}),\ldots,\phi_{h,\mathbf{w}_{p_{\operatorname{out}}}}(\calI_1, \dots, \calI_{p_{\operatorname{in}}})\Big).$$
\end{definition}

\begin{remark}[Excitatory and inhibitory neurons]
    In neuroscience, neurons are often classified as either excitatory or inhibitory based on the type of effect they have on their postsynaptic targets. Excitatory neurons increase the likelihood of the postsynaptic neuron firing an action potential, while inhibitory neurons decrease this likelihood. In our framework, this distinction can be modeled by assigning to each neuron either exclusively positive or exclusively negative outgoing weights. To simplify the notation, however, we do not impose this constraint in the formal definition of SNNs. Nonetheless, this constraint can be easily incorporated by at most doubling the number of neurons in a network. 
\end{remark}

Stacking $L$ SNN layers on top of each other yields then a feedforward SNN with $L-1$ hidden layers. In analogy with feedforward ANNs, the neurons after the $\ell$-th SNN layer are jointly referred to as $\ell$-th hidden layer for $\ell=1,\ldots,L-1.$ The input neurons are referred to as input layer or $0$-th layer and the output neurons as output layer or $L$-th layer. The \textit{architecture} of a feedforward SNN is given by a pair $(L,\mathbf{p})$, where $L$ denotes the number of stacked SNN layers and $\mathbf{p}=(p_0,p_1,\ldots,p_L)$ is the width vector, consisting of the number of input neurons $p_0,$ the number of output neurons $p_{L}$ as well as the number of neurons $p_1,\ldots,p_{L-1}$ in each of the corresponding $L-1$ hidden layers. For $\ell \in [L]$, the $(\ell - 1)$-st layer is connected to the $\ell$-th layer by the SNN layer
$\boldsymbol{\psi}_{h,W_{\ell}}: \mathbb{T}^{p_{\ell-1}} \to \mathbb{T}^{p_{\ell}}$. 

The total number of parameters in this SNN architecture is $\sum_{\ell=1}^L p_{\ell-1} p_{\ell}$.
We additionally impose the constraint that among these parameters at most $s \leq \sum_{\ell=1}^L p_{\ell-1} p_{\ell}$ are non-zero.
This means that there are at most $s$ weights between neurons. 
To formalize this, we write $\|A\|_0$ to denote the number of non-zero entries of a matrix $A.$

\begin{definition}[Feedforward SNN]\label{def:feedforwardSNN}
    For memory coefficient $h \in [0,\infty]$, $L \in \mathbb{N}$ SNN layers, width vector $\mathbf{p}=(p_0,p_1,\ldots,p_L)$ and maximum number of non-zero parameters (sparsity) $s \in \mathbb{N}$, we define the class of feedforward SNNs (with fixed architecture) as 
    \begin{align*}
    \SNN_h(L, \mathbf{p}, s) 
    := \bigg\{ \bbf: \mathbb{T}^{p_{0}} \to \mathbb{T}^{p_{L}},\bbf(\calI_1, \dots, \calI_{p_0}) :=   \boldsymbol{\psi}_{h, W_L} \circ \ldots \circ \boldsymbol{\psi}_{h, W_1}(\calI_1, \dots, \calI_{p_0}),& \\
    W_{\ell} \in \RR^{p_{\ell} \times p_{\ell-1}}, \sum_{\ell=1}^L \|W_{\ell}\|_0 \leq s & \bigg\}.
    \end{align*}
\end{definition}

These SNN classes satisfy the composition property
\begin{align} \label{property_comp}
    \begin{split}
    &\Big\{\bbf' \circ \bbf : \, \bbf \in \SNN_h\big(L,(p_0,\ldots,p_{L}),s\big), \, 
    \bbf' \in  \SNN_h\big(L',(p_{L},\ldots,p_{L+L'}),s'\big)\Big\} \\
    &= \SNN_h\big(L+L',(p_0,\ldots,p_{L+L'}),s+s'\big).
    \end{split}
\end{align}
As a collection of functions on spike trains, this space is not convex (in any traditional sense) and does not admit any additive or linear structure. Moreover, given the definition of feedforward SNNs, and noting that the output spikes of an SNN unit are  dominated by its input, we conclude an analogous property for feedforward SNNs: an output spike can only occur if an input spike occurred at the same time.

\begin{corollary}\label{cor:inputPotential}
    For any feedforward SNN $\bbf = (f_1,\ldots, f_k)$ with $d$ inputs and $k$ output,
    \begin{align*}
       \textstyle \bigcup_{i \in [k]} f_i(\calI_1, \dots, \calI_d)\dom \bigcup_{j \in [d]} \calI_j \quad \text{ for all } \;\;\calI_1, \dots, \calI_d \in \TT.
    \end{align*}
\end{corollary}

\subsection{Interpreting SNNs as ANNs with State-Space Block}\label{subsec:connection_SNN_ANN}

A natural question in the study of SNNs is how they relate to artificial neural networks (ANNs). In the following, we discuss how SNNs can be interpreted as a certain class of recurrent ANNs with specific activation functions and state-space blocks, see also \citep{casco-rodriguez2025on}. This connection deviates from the approach by \citet{STANOJEVIC202374}, who consider a conversion from ANNs to SNNs using time-to-first-spike coding.

To derive the correspondence, we need to restrict the input spike trains to a regular grid, i.e., we consider only spike trains dominated by $\calN_{\delta} \coloneqq \{n \delta: n \in \mathbb{N}\}$ for some $\delta>0$. Under this assumption,  we can encode the input spike trains $\calI_1, \dots, \calI_d$ via the respective sequential input vector $(\bx^{(t)})_{t \in \mathbb{N}}$ with
$\bx^{(t)} := (\mathds{1}(t\delta \in \calI_j))_{j \in [d]}$. This allows us to represent an SNN layer using a recurrent ANN with a state-space block consisting of two hidden neurons. One hidden neuron has the (strict) Heaviside activation function $\mathds{1}(x>1)$, and the other hidden neuron uses the clipped ReLU activation function $x \mathds{1}(0 \leq x \leq 1)$, see Figure \ref{SNN_RNN_equ} for a schematic depiction. The Heaviside neuron models the output spike train, whereas the clipped ReLU neuron encodes the potential; see Appendix \ref{app:SNNasANNs} for details.

\begin{figure}[t!]
    \centering
    \includegraphics[width = 1\linewidth]{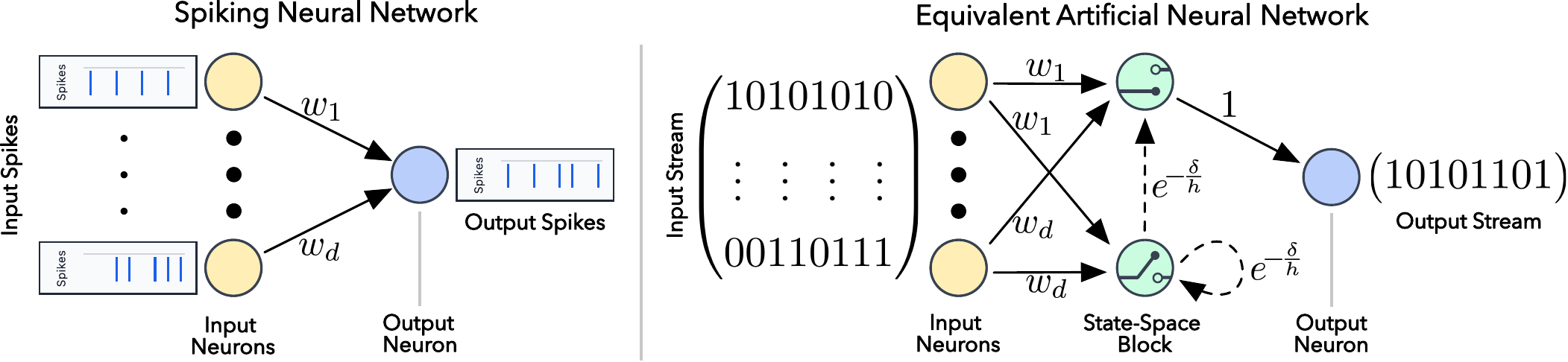}
    \caption{\textbf{SNN computational unit  and equivalent  recurrent ANN with state-space block.} In both SNN and ANN, the yellow and blue neurons denote input and output neurons, respectively. For the SNN unit, the input consists of spike trains $\calI_1, \dots, \calI_d$ which are all dominated by $\calN_{\delta}$. 
    For the ANN, the input consists of an input stream of binary vectors, $(\bx^{(t)})_{t \in \mathbb{N}}$ which encodes the input spike trains via  $\bx^{(t)} := (\mathds{1}(t\delta \in \calI_j))_{j \in [d]}$. 
    The state-space block of the ANN consists of two hidden green neurons, one hidden neuron has activation function $ \mathds{1}(x>1)$ and generates the output spike train and the other hidden neuron has activation function $ x \mathds{1}(0 \leq x \leq 1)$ and encodes the potential. The output of the latter neuron is delayed by the time unit $\delta$ (dashed connections). %
    }
    \label{SNN_RNN_equ}
\end{figure}

While this argument is for SNNs with input spike times contained in a regular grid, the possibility for asynchronicity of input spike trains provides SNNs with additional flexibility and representational power. 

In the degenerate case of no memory $h = 0$, the membrane potential immediately drops to zero after a spike occurred. The output spikes depend only on the input spikes at the same time. In this case, the conversion shows that a feedforward SNN is then exactly equivalent to an ANN with the (strict) Heaviside activation function $x \mapsto \mathds{1}(x > 1),$ the same network weights, and no biases. 
See \citet{khalife2024neural} and \citet{kong2025expressivity} for the expressivity of ANNs with the Heaviside activation function.

\section{Universal Representation of Functions of Spike Trains}\label{sec:functional_representation}

This section analyzes the representational power of Spiking Neural Networks (SNNs). We specifically focus on functions that map $d$ input spike trains to a single output spike train and satisfy some natural constraints that are introduced below. Explicit representation theorems are provided in \Cref{subsec:Univ_Rep1} for a single input ($d = 1$) and in \Cref{subsec:Univ_Rep2Plus} for multiple inputs ($d \geq 2$). Building on these results, we then discuss implications for the representation of compositional functions in \Cref{subsec:composite_functions}.

\begin{definition}\label{def:properties}
    Let $F\colon \TT^d \to \TT$ be a function of $d\in \NN$ input spike trains. 
    \begin{enumerate}
        \item[$(i)$] The function $F$ is called \emph{\inputdominated} if for any $\calI_1, \dots, \calI_d\in\TT$ it holds that
        \[          F(\calI_1, \dots, \calI_d)\dom \textstyle\bigcup_{i \in [d]}\displaystyle\calI_i. \]
        \item[$(ii)$] The function $F$ is called \emph{causal} if for all $t>0$ and $\calI_1, \dots, \calI_d\in \TT$ it holds that
        \begin{align*}
            F(\calI_1, \dots, \calI_d) \cap [0,t] = F({\calI}_1\cap [0,t], \dots, {\calI}_d\cap [0,t]) \cap [0,t].
        \end{align*}
        \item[$(iii)$]
        The function $F$ has the \emph{monotone scaling property}        
        if for any strictly 
        increasing bijection $\phi:(0,\infty) \to (0,\infty)$ and any 
        $\calI_1,\dots,\calI_d\in\TT$ it holds that
        \[
          F(\phi(\calI_1), \dots, \phi(\calI_d)) 
          = \phi(F(\calI_1, \dots, \calI_d))
        \]
    \end{enumerate}
\end{definition}

Intuitively, an \inputdominated\ function can only have an output spike if an input spike occurred at that time. Causality ensures that the output at time $t$ depends only on input spike that have occurred up to time $t$ but not on any future input spikes. While the term causal has various meanings across different fields of mathematics, we chose it because of the close similarity with causal systems in control theory and the recently introduced related notion of `causal pieces' \citep{2025arXiv250414015D}. 

The output of an \inputdominated, causal function can be very sensitive to perturbations of the input spike times. To delimit this behavior, we impose the monotone scaling property which entails that the output is independent of the size of the gaps between consecutive input spikes. 
What matters is whether a neuron spikes before, at the same time, or after another neuron -- and this remains invariant under monotone transformations. %

These properties are natural in the context of SNNs. Indeed, any feedforward SNN $f$ with memory $h\in [0, \infty]$ is \inputdominated\ and causal. Further, if $h = \infty$, then $f$ satisfies the monotone scaling property. This follows from the fact that the potential is constant between spikes, and thus only the ordering of the spike timings matters. For $h < \infty$, the potential depends on the precise spike timings, and thus  in general $f$ does not satisfy the monotone scaling property. Remarkably though, even for the regime of finite $h\in (0, \infty)$, SNNs can still exactly represent functions satisfying the monotone scaling property. 

\subsection{Single-Input Universal Representation} \label{subsec:Univ_Rep1}

We first focus on \inputdominated, causal functions $F\colon \TT \to \TT$ of a single input satisfying the monotone scaling property. As a consequence of this property, such a function is fully determined by the output for the spike train $\calI = \NN$. Indeed, for any $\calI \in \TT$, let $\pi \colon \calI \to \NN$ be the ordering function defined by $\pi(\calI_{(j)}) = j$ for $j \in [|\calI|]$, then %
\begin{align}\label{eq:single_input_reduction}
    F(\calI)  = \pi^{-1}\big(F(\pi(\calI))\big) = \pi^{-1}\big(F(\NN) \cap [|\calI|]\big).
\end{align}
Given that the function is fully determined by $F(\NN)$, we first turn our attention to the special cases when $F(\NN)$ is finite or periodic. In both cases, $F(\NN)$ is fully characterized by finitely many spikes. 

\begin{definition}[Finite and periodic functions]
    An \inputdominated, causal function $F\colon \TT \to \TT$ satisfying the monotone scaling property is called $m$-\emph{finite} for $m\in \NN$ if $F(\NN) = F([m])$. %
    Further, $F$ is called $m$-\emph{periodic} for $m\in \NN$ if 
    \begin{align*}
        F(\NN) = F([m])\oplus \{0,m, 2m ,\ldots\},
    \end{align*}
    with $\oplus$ denoting the Minkowski sum (sumset). Further, if $F$ is $m$-finite (resp.\ $m$-periodic), then it is said to be $r$-sparse for $r\in [m]$ if $|F([m])| \leq r$. The class of all functions which are $m$-finite (resp.\ $m$-periodic) and $r$-sparse is denoted by $$\FF_{\mathrm{fin}}(m,r)\quad \ \big(\text{resp.} \ \ \FF_{\mathrm{per}}(m,r)\big).$$
\end{definition}
The following theorem states that both finite and periodic \inputdominated, causal functions satisfying the monotone scaling property can be exactly represented by a feedforward SNNs with finitely many neurons and positive memory. %

\begin{theorem}[Single-input universal representation theorem]\label{thm:universal_representation_single_input}
    Let $m\in \NN, r \in [m]$ with $h\in (0,\infty]$. Then, there exists an architecture $(L,\bp)$ with $L = O(1+\log(m))$, $\|\bp\|_1 = O(\sqrt{m})$ and a positive integer $s = O(\sqrt{m} + r)$ such that for any $F \in \FF_{\mathrm{fin}}(m,r)$ there exists a feedforward SNN $f \in \SNN_{h}(L,\bp,s)$ with memory $h$ such that
    \begin{align}\label{eq:universal_representation_single_input}
        f(\calI) = F(\calI), \quad \text{for all } \calI \in \TT.
    \end{align}
    If $m = 4^{q}$ for $q\in \NN_0$, and possibly $h = 0$ if $m = 1$, then \eqref{eq:universal_representation_single_input} also holds for $F\in \FF_{\mathrm{per}}(m,r)$.  %
\end{theorem}

The proof is fully constructive and detailed in \Cref{subsec:proofs_universal_representation_single_input}. In particular, a feasible assignment of the SNN parameters that ensure representation can be explicitly derived from the represented function. The number of neurons scales with $\sqrt{m}$ rather than $m$, whereas the number of weights scales with $r + \sqrt{m}$. Thus, fewer output spikes of the function enable a representation by a more sparsely connected SNN. 

Complementary to the upper bounds we also derive a lower bound showing that the number of non-zero network parameters $s$ cannot be improved by more than an $O(\log(m+1))$-factor when $r \geq \Omega(\sqrt{m}).$

\begin{theorem}[Lower bound for the number of SNN parameters for single-input universal representation] \label{cor_lo_numpara_train}
     Let $m \in \NN$, $r \in [m]$, $L \in \NN, \bp \in \NN^{L+1}, s\in \NN$ and $h\in (0,\infty]$.     
     If every $F \in \FF_{\mathrm{fin}}(m,r)$ can be represented by an element of $\SNN_{h}(L,\bp,s)$, then     
    $$s \geq \frac{r}{5 \log(m+1)}.$$
    The analogous statement also holds for 
    $\FF_{\mathrm{per}}(m,r)$.
\end{theorem}

The proof relies on quantitatively upper bounding  the expressive power of feedforward SNNs for a fixed input spike train (\Cref{thm:SNN_expressivity}) and lower bounding the number of distinct functions in $\FF_{\mathrm{fin}}(m,r)$ and $\FF_{\mathrm{per}}(m,r)$. Details are provided in \Cref{subsec:proofs_cardinality_upper_bound}. In particular, for $r = m$, we require $s = \Omega(m/\log(m))$ weights.  

Let us now comment on multiple aspects of the universal representation theorem. %
\begin{enumerate}
    \item The construction is independent of the memory coefficient $h\in (0, \infty]$. Intuitively one may suspect that larger memory increases the expressive power. However, for monotone scaling SNNs, rescaling the time axis shows that the expressive power of SNNs is identical for all $h\in (0, \infty)$ and in particular no worse than that for $h = \infty$. 
    \item %
    All network weights in the SNN construction are integer valued, which appears to be necessary in some instances (see \Cref{rmk:strict_threshold}) to ensure that the SNN satisfies monotone scaling. This observation could have interesting consequences for learning. If all input spikes occurred in some bounded interval, then the weights could also be chosen as real-valued without changing the SNN output. Insofar, integer weights ensure full representative power of an SNN unit for large interarrival times. We return to this point in \Cref{rmk:strict_threshold}. 
    \item Depth is necessary as it is impossible to represent certain \inputdominated, causal, monotone scaling functions which are $3$-finite with a \emph{shallow}, i.e., single hidden layer, feedforward SNN (\Cref{prop:onelayer_counter}). However, for any $m\in \NN$ and $h \in (0, \infty)$, all $m$-finite functions $F\in \FF_{\mathrm{fin}}(m,m)$ can be represented on the first $m$-input spikes by a shallow SNN with $m$ hidden neurons as long as the interarrival time between consecutive input spikes are sufficiently homogeneous (\Cref{prop:onelayer_positive}). 
    \item The number of negative weights in the construction in \Cref{thm:universal_representation_single_input} scales with order $\sqrt{m}$, whereas the number of positive weights is proportional to $r+\sqrt{m}$. For $r=\Omega(m),$ the gap in the rate is $\Omega(\sqrt{m}).$ Interestingly, also in biological neural networks more excitatory neurons (positive weights) than inhibitory neurons (negative weights) occur, with the typical ratio $4:1$ for cortical neurons \citep{Markram2004, Lehmann2012, Sahara4755}. In \Cref{thm:log_T_neg_weights} we construct an  SNN with $m$ output neurons and  $O(m (1+\log^2 m))$ total neurons and non-zero weights, of which only $2\lceil\log(m)\rceil$ are negative such that for any finite input spike train $\calI$ with cardinality $m$, the $j$-th output neuron $N_j$ for $j\in [m]$ spikes precisely at $\calI_{(j)}$. This implies that $m$-finite functions (on spike trains with cardinality $m$) can also be implemented with SNNs with similar architectural constraints. 
    Conversely, at least two negative weights are necessary to ensure that certain outputs are representable, see \Cref{lem:two_negative_weights}.
    \item The current construction in \Cref{thm:universal_representation_single_input} is optimized for SNNs with no upper bound on the out-degree (number of descendants of a  node in the underlying graph). %
    Under the additional assumption that every neuron is connected to at most $q\in \NN$, $q\geq 4$,  neurons in the next layer the required number of neurons to represent $m$-finite or $m$-periodic functions changes to $O(\max(m/q, \sqrt{m}))$  (\Cref{thm:universal_representation_spike_trains_bounded_out_degree}). %
    In particular, by \Cref{cor_lo_numpara_train} and \Cref{prop:connection_neuron_lower_bound} the neuron count for general $q$ cannot be improved up to logarithmic terms in $m$.%
\end{enumerate}

\subsection{Multiple-Input Universal Representation} \label{subsec:Univ_Rep2Plus}

We now focus on multi-input functions $F\colon \TT^d \to \TT$ with $d \geq 2$ (although the following is still sensible for $ d=1$) which are \inputdominated, causal and satisfy the monotone scaling property. Unlike the single-input case, such functions are not fully determined by $F(\NN, \dots, \NN)$, as this only describes the output if all input spikes happen at the same time. To ensure finite representability using SNNs, we therefore impose the following assumption on the behavior of the function, which restricts the dependence of the output spike at time $t$ to only a bounded number of past input spikes.

\begin{definition}[Markovian memory]
    An \inputdominated, causal function $F\colon \TT^d\to \TT$ satisfies the monotone scaling property has  \emph{Markovian memory} of order $m\in \NN$, or is said to be $m$-Markovian, if for all $t\in \NN$ and  $\calI_1, \dots, \calI_d\in \TTT{\NN}$ with  $\bigcup_{i = 1}^d \calI_i \cap (0,t] = [t]$ it~holds 
    \begin{align*}
        F(\calI_1, \dots, \calI_d)\cap\{t\} = \big(F(\calI_1 \cap (t-m,t], \dots, \calI_d \cap (t-m,t])\big)\cap\{t\}.
    \end{align*}
    If $m = 1$, the function is called \emph{memoryless}. For $r\in[2^{md}]$, the function $F$ is called $r$-\emph{sparse}~if $$\left|\left\{ \textstyle (\calI_1, \dots, \calI_d) \in \TTT{[m]}^d  \colon F(\calI_1, \dots, \calI_d) \neq \emptyset, \bigcup_{i \in [d]}\calI_i = [m'] \text{ for } m'\in [m]\right\}\right|\leq r.$$ The class of all $m$-Markovian functions $F$ which are $r$-sparse is denoted by $$\FF_{\mathrm{mm}}(d,m,r).$$ 
\end{definition}

Intuitively, the output of an $m$-Markovian function can only depend on the last $m$ input spikes from each input spike train to determine whether there is an output spike at time~$t$. This is natural as biological neurons typically have only limited memory, and thus the output at time $t$ cannot depend on arbitrarily distant input spikes. 

SNNs can also exactly represent causal, monotone scaling functions with bounded Markovian memory: 

\begin{theorem}
[Multiple-input universal representation theorem]\label{thm:function_representation}
Let $d,m \in \NN, r \in [2^{md}]$ with $h\in (0,\infty]$ or possibly $h =0$ if $m= 1$. Then, there exists an architecture $(L,\bp)$ and a positive integer $s$ satisfying $L = O(1 +\log(m))$,  $\|\bp\|_1 = O(dm^2+r)$, and $s = O(dm^2r)$ such that for any function $F\in \FF_{\mathrm{mm}}(d,m,r)$, there exists a feedforward SNN $f \in \SNN_h(L, \bp, s)$~with
    \begin{align}\label{eq:representation_causal_function}
        f(\calI_1, \dots, \calI_d) = F(\calI_1, \dots, \calI_d) \quad \text{ for all }\;\;\calI_1, \dots, \calI_d \in \mathbb{T}. 
    \end{align}
\end{theorem}

The proof is provided in \Cref{subsec:proofs_universal_representation_multiple_inputs} and is also entirely constructive. It relies on building a memory-mechanism (using $O(dm)$ neurons) within the SNN to store the relevant past input spikes, and subsequently representing a binary function with binary inputs (using $O(r)$ neurons) by essentially constructing Boolean operations. The latter is based on the classical sum-of-products (disjunctive normal form) representation of Boolean functions, and it exploits that SNNs can efficiently represent the \texttt{AND} and \texttt{OR} functions with an arbitrary number of inputs. 

Complementary to the representation result, we also derive a lower bound on the required number of network weights to represent all functions in $\FF_{\mathrm{mm}}(d,m,r)$. 

\begin{theorem}[Lower bound for the number of SNN parameters for multiple-input universal representation] 
\label{cor_lo_numpara_functional}
     Let $d,m \in \NN$ with $d \geq 2$, $r \in[2^{md}]$, $L \in \NN, \bp \in \NN^{L+1}, s\in \NN$ and $h\in (0,\infty]$.     
     If every $F \in \FF_{\mathrm{per}}(m,r)$ can be represented by an element of $\SNN_{h}(L,\bp,s)$, then
    $$s \geq \frac{r \land (2^{d}-1)^{m}}{5md}.$$  
\end{theorem}

The proof of the lower bound is similar to the single-input case and based on a counting argument provided in \Cref{subsec:proofs_cardinality_upper_bound}. 
Let us now discuss some more aspects of the result.
\begin{enumerate}
    \item In the absence of any sparsity, we have $r=2^{md}$ which is substantially larger compared to $r\leq m$ in the single-input case. In conjunction with the lower bounds on the network connectivity, this highlights the optimality of the construction in \Cref{thm:function_representation} up to polynomial factors in $m$ and $d$ and   
    the presence of a \emph{curse of dimensionality} in the SNN representation of the class $\FF_{\mathrm{mm}}(d,m,r=2^{md}).$ However, it is not just the dimensionality $d$ which leads to large architectures but also the memory parameter $m$.  This is to be expected since for multiple inputs, the number of possible input spike patterns that can be realized within the memory window scales exponentially in $md$. Meanwhile for a single input, by the monotone scaling property, all input spike patterns with $m$ spikes are essentially equivalent and lead to similar outputs (up to the exact spiking time). One structural assumption which mitigates the curse, however, is that $F$ yields an output spike on a small fraction of input spike patterns. 
    \item In Theorem \ref{thm:function_representation}, the bound on the number of required network weights is of the same order as the number of neurons.  We believe that by allowing for more intricate connectivity patterns and at the expense of a significantly more involved proof, the number of neurons could be reduced further. The lower bound in \Cref{cor_lo_numpara_functional} combined with \Cref{prop:connection_neuron_lower_bound} however asserts that the number of neurons can be at best reduced to $\Omega(\sqrt{\frac{r \wedge ((2^d-1)^m-1)}{dm}})$ which still scales exponentially in $dm$ for maximal $r$.  %
    \item Similar to the single-input case, the construction in the proof of \Cref{thm:function_representation} does not depend on the choice of the memory parameter $h\in (0, \infty]$ and utilizes integer weights only. If one were to impose an upper bound on the interarrival time between consecutive spikes in the input spike trains, then the construction could also be carried out for real-valued weights without changing the representability result, see \Cref{rmk:strict_threshold}. 
		Notably, the use of real weights allows SNNs to be provable more expressive compared to SNNs with the same architecture and integer weights, 
		 see \Cref{thm:counterexample-integer-weights}. Characterizing the full representational power of SNNs with real-valued weights is an interesting direction for future research.
\end{enumerate}

We close this subsection with two remarks on how the framework of \inputdominated, causal, monotone scaling functions with Markovian memory can be utilized to model time-dependent and periodic functions. 
\begin{remark}[Time-dependent functions]\label{rmk:time_dependent_functions}
Functions satisfying the monotone scaling property are fully determined by the order at which the input spikes arrive. By incorporating an exogenous spike train, one can still obtain functions which depend on the actual timing of the spikes and not only on their ordering. To see  this, take a fixed spike train $\calD$ and consider
\begin{align*}
    (\calI_1, \dots, \calI_d) \mapsto F_{\calD}(\calI_1, \dots, \calI_d, \calD)
\end{align*} 
for some $m$-Markovian function $F_{\calD}\colon \TT^{d+1} \to \TT$. The output spike train of this map will depend on the actual timing of the input spike trains. This construction plays an important role in the proofs. For concrete examples, see the \texttt{IS-EQUAL}, and \texttt{CEIL}$_m$ function from \Cref{sec:examples_functions}, as well as the \texttt{CLASSIFIER}$_{\calD, \mfR}$ function in \Cref{sec:classification_spike_trains}.

\end{remark}

\begin{remark}[Periodic functions] By Lemma \ref{lem:finite_and_markovian}, the class of \inputdominated, causal, Markovian functions excludes periodic functions. Nevertheless, composing a function of finite Markovian memory with a periodic function can yield a function which is periodic from some spike onwards. For example, composing an
    $m_P$-periodic function $F_M\colon \TT^2\to \TT$ with an $m_M$-Markovian function $F_P\colon \TT\to \TT$, yields a function $F: \TT \to \TT, F(\calI) := F_M(\calI, F_P(\calI))$ which is $m_P$-periodic from the $(m_M +1)$-st overall spike onwards (Lemma \ref{lem:mm_periodic_composition}). %
\end{remark}

\subsection{Efficient Representation of Compositional Sparse Functions}\label{subsec:composite_functions}

As a consequence of the representation results for single and multiple inputs, we also obtain refined representation statements for compositions of monotone scaling functions. To this end, we consider for $\bd \coloneqq (d_0, d_1, \dots, d_q)$ with $q \in \NN$, $d_0 = d$, and $d_q = 1$ functions of the form
\begin{align*}
    F = F_q \circ F_{q-1} \circ \dots \circ F_1,
\end{align*}
with \inputdominated, causal, monotone scaling  functions $F_\ell \colon \TT^{{d_{\ell-1}}} \to \TT^{{d_{\ell}}}$ for $\ell = 1, \dots, q.$ According to the lower bound in \Cref{cor_lo_numpara_functional}, the number of required weights to represent a general function in $F_\ell$ grows at least with order $\Omega(d_{\ell} 2^{d_{\ell-1} m_\ell}/(d_{\ell-1} m_\ell))$ if $F_\ell$ is $m_\ell$-Markovian. Sparsity assumptions can reduce the number of network weights. However, sparsity is not preserved if one artificially enlarges the input dimension $d_{\ell-1}$ of $F_\ell$ by adding irrelevant input spike trains:

\begin{example}
Let $F\in \FF_{\textup{mm}}(2,m,r)$ for $m\in \NN, r\in [2^{2m}]$,  and define for $d\in \NN$ the function $G\colon \TT^{d+2} \to \TT, G(\calI_1, \dots, \calI_{d+2}) = F(\calI_1, \calI_2)$. Then, $G$ is not $m$-Markovian, since the output will depend on spikes in $\calI_1$ and $\calI_2$ which may not be among the last $m$ spikes of $\calI_1 \cup \dots \cup \calI_{d+2}$. Further, for any $\calI_1, \calI_2 \in \TTT{[m]}$ with $F(\calI_1, \calI_2) \neq \emptyset$, it follows for any $\calI_3, \dots, \calI_{d+2} \in \TTT{[m]}$ that $G(\calI_1, \calI_2, \dots, \calI_{d+2}) \neq \emptyset$, resulting in $G$ being at best $2^{md}$-sparse. 
\end{example}

To remedy this issue, which would suggest unnecessarily large architectures to represent functions with few relevant inputs, 
we assume that for each $\ell \in [q]$, there exist integers $t_{\ell-1}\in \{0, \dots, d_{\ell-1}\}$, $m_\ell \in \NN$, $r_\ell \in \NN_0$ with $r_\ell \leq m_{\ell}$ if $t_{\ell-1}\leq 1$ and $r_{\ell}\leq 2^{m_{\ell} t_{\ell-1}}$ if $t_{\ell-1}\geq 2$ such that $F_{\ell} = (F_{\ell, k})_{k\in [d_{\ell}]}$ and each component function satisfies
\begin{align*}
    F_{\ell, k}\in \FF(t_{\ell-1}, m_\ell, r_\ell) \coloneqq \begin{cases}
        \FF_{\mathrm{per}}(4^{\lceil\log_4(m_{\ell})\rceil}, r_\ell), & \text{ if } t_{\ell-1} = 0,\\
        \FF_{\mathrm{fin}}(m_{\ell}, r_\ell), & \text{ if } t_{\ell-1} = 1,\\
        \FF_{\mathrm{mm}}(t_{\ell-1}, m_\ell, r_\ell), & \text{ if } t_{\ell-1} \geq 2.
    \end{cases}
\end{align*}
 This means that each component only depends on at most $\max(t_{\ell-1},1)$ inputs, and that the output is characterized by the last $O(m_{\ell})$ time-points.  The corresponding function class of compositional functions is denoted by 
\begin{align*}
    \FF(q, \bd, \bt, \bm, \br) \coloneqq \Big\{F = F_q \circ \dots \circ F_1\;\;: \;\; &F_\ell = (F_{\ell, k})_{k\in [d_{\ell}]} \colon \TT^{{d_{\ell-1}}} \to \TT^{{d_{\ell}}},
   \\
    &F_{\ell, k}\in \FF(t_{\ell-1}, m_\ell, r_\ell) \text{ for all } k\in [d_{\ell}], \ell \in [q]\Big\},
\end{align*}
where $\bt \coloneqq (t_0, t_1, \dots, t_{q-1})\in \NN^q_0,  \bm \coloneqq (m_1, m_2, \dots, m_q)\in \NN^{q}$, $\br \coloneqq (r_1, r_2, \dots, r_{q})\in \NN_0^{q}$. In Sections \ref{sec:examples_functions} and \ref{sec:classification_spike_trains}, we illustrate how this compositional framework can be used to model relevant functions for spike train processing. 

\begin{theorem}[SNN representation of compositional functions]\label{thm:function_representation_composite}
Let $q\in \NN$, $\bd, \bt, \bm, \br$ be as above with $d_0 = d$ and $d_q = 1$ with $h \in (0, \infty]$. Then, there exist parameters $L = O(\sum_{\ell=1}^{q}1 + \log(m_\ell))$, $\bp\in \NN^L$, $s\in \NN$  with $\|\bp\|_1 = O\left(\sum_{\ell = 1}^q d_{\ell} \left[m_{\ell}^2(\max(1,t_{\ell})) + r_\ell  \right] \right)$ and %
    $
    s =  O\left(\sum_{\ell = 1}^q d_{\ell} m_{\ell}^2\max(1,t_{\ell}) r_\ell \right)
    $
such that for any function $F\in \FF(q, \bd, \bt, \bm, \br)$ there exists a feedforward SNN $f \in \SNN_h(L, \bp, s)$ with
    \begin{align*}%
        f(\calI_1, \dots, \calI_d) = F(\calI_1, \dots, \calI_d) \quad \text{ for all }\;\;\calI_1, \dots, \calI_d \in \mathbb{T}. 
    \end{align*}
\end{theorem}

The proof of \Cref{thm:function_representation_composite} is a direct consequence of Theorems \ref{thm:universal_representation_single_input} and \ref{thm:function_representation} by constructing for each component function $F_{\ell, k}$ an SNN according to the respective theorem's architecture and subsequently stacking the SNNs on top of each other. 

Let us summarize the key insights of Theorem \ref{thm:function_representation_composite}. %
\begin{enumerate}
    \item By the universal representation theorems, a function is efficiently representable if it has small memory or few inputs. \Cref{thm:function_representation_composite} shows that also the composition of functions with small memory and many inputs with a function that has large memory and few relevant inputs can be efficiently represented. 
    \item %
    A key conceptual benefit of the compositional framework is that periodic functions with a large period can be generated by composing periodic functions with a short period. In particular, this allows efficiently representing functions which depend on a rather long history of input spikes using functions which all admit small memory. For this to be possible, information from past spikes needs to be repeated and successively condensed into a few spikes so that they can be passed on to later components. This process may be interpreted as extracting and temporarily storing relevant features from the past spike patterns.
    \item Composing functions also allows processing information from many input spike trains by stacking functions on top of each other, where each component function only depends on a small subset of the inputs. As small SNNs can represent the standard Boolean logical operations (\Cref{subsec:expl:memoryless_functions}), an example of such a compositional structure are SNNs that emulate Boolean circuits.

\end{enumerate}

\section{Examples of Efficiently Representable Functions}
\label{sec:examples_functions}

In the following, we provide several examples for \inputdominated, causal, functions satisfying the monotone scaling property and admitting an efficient SNN representation. %
The representation of all these functions is detailed in Appendix \ref{app:proofs_section4}, some figures below already illustrate the constructions.

\subsection{Finite or Periodic Single-Input Functions}\label{subsec:expl:single_input_functions}
The following functions operate on a single-input spike train and are either finite or periodic. Specific constructions for SNNs representing these functions are illustrated in \Cref{fig_skip}.
\begin{itemize}
    \item \texttt{SKIP} function: Given an input spike train $\calI \in \mathbb{T}$, its output is $\calO \coloneqq \calI$. This function describes the identity function on spike trains, and can be implemented with any depth, which is useful to synchronize the depth of different building blocks within an SNN. It is $1$-periodic and $1$-sparse. 
    \item \texttt{ODD/EVEN} function: Given an input spike train $\calI \in \mathbb{T}$, this function outputs two spike trains $\calO_{1}^{2}, \calO_{2}^{2} \in \mathbb{T}$,  defined by 
    \begin{align*}
        \calO_{1}^{2} \coloneqq \{\calI_{(2n-1)} : n\in \NN, 2n-1 \leq |\calI|\}, \quad \calO_{2}^{2} \coloneqq \{\calI_{(2n)} : n\in \NN, 2n \leq |\calI|\}, 
    \end{align*}
    where $\calO_1^2$ (resp.\ $\calO_2^2$) denotes the output of the \texttt{ODD} (resp.\ \texttt{EVEN}) component. When we refer to the \texttt{ODD} and \texttt{EVEN} function separately, we mean the single-output spike train function with the former or the latter output from the display above, respectively.  
    Both the \texttt{ODD} and \texttt{EVEN} function have period of order $m = 2$ and are $1$-sparse. %
    \item \texttt{CLOCK}$_m$ function for $m = 2^q$ with $q\in \NN$: Given an input spike train $\calI \in \TT$, this operation outputs $m$ spike trains $\calO_1^m, \dots, \calO_m^m \in \mathbb{T}$, defined by
    \begin{align*}
        \calO_j^{m} \coloneqq \{\calI_{(m  k+ j)} : k\in \NN_0, mk+j \leq |\calI| \} \quad \text{ for } j\in [m].
    \end{align*}
    Each output is $m$-periodic and $1$-sparse. The function can be constructed using a binary tree of depth $q$ where each node represents am \texttt{ODD/EVEN} function, and all \texttt{ODD} outputs and \texttt{EVEN} output per layer are group together and sorted. The spike train realized at a given node in the tree can be exactly characterized, see \Cref{lem:odd_even_path}. %
        \item \texttt{SPIKE}$_m$ function for $m\in \NN$: Given an input spike train $\calI \in \mathbb{T}$, the output spike train is defined as $\calO \coloneqq\{\calI_{(m)}\}$ if $|\calI| \geq m$ and $\calO \coloneqq \emptyset$ else. In other words, $\calO$ contains only the $m$-th spike of the input spike train. This function is $m$-finite and $1$-sparse. 
        \item \texttt{REPRESENT}$_\br$ given a binary vector $\br = (r_1, \dots, r_m)\in \{0,1\}^m$ for $m \in \NN$: For an input spike train $\calI \in \TT$ the output spike train is given by 
		\begin{align*}
        		\calO \coloneqq \{\calI_{(k)} \colon k\in [m], k \leq |\calI|, r_k = 1\}.
        \end{align*}
        This outputs a spike train whose spiking pattern is determined by $\br$ on the first $m$ spikes of $\calI$. This function is $m$-finite and is sparse with index $\|\br\|_1$. 
\end{itemize}

\begin{figure}[t!]
    \centering
    \includegraphics[width=\linewidth]{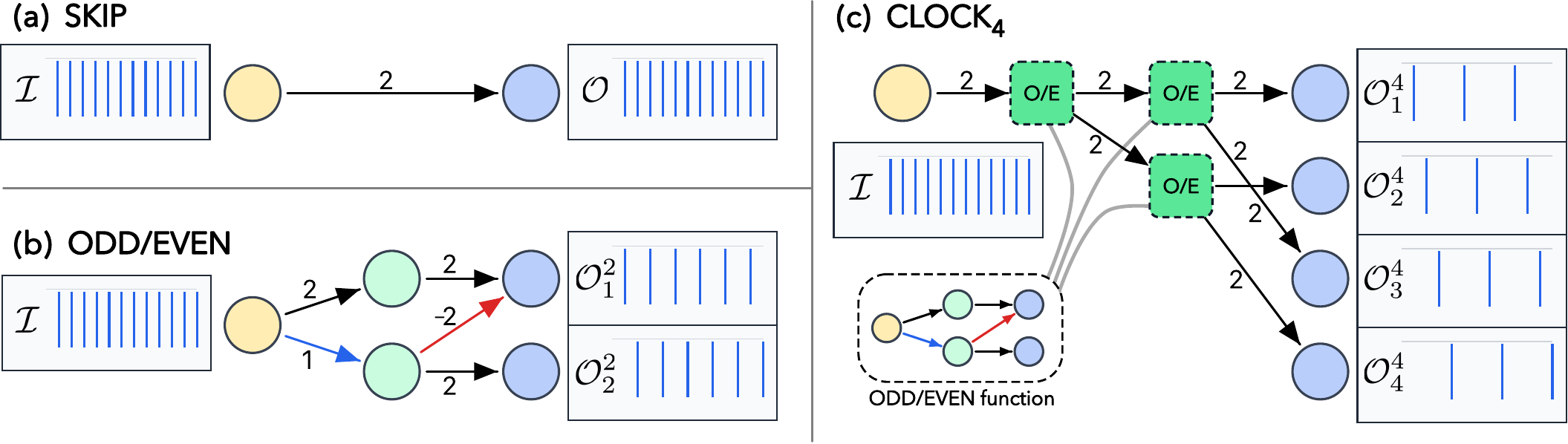}
    \caption{\textbf{SNNs implementing periodic functions}. Explicit construction via SNNs for $(a)$  \texttt{SKIP} function, $(b)$ \texttt{ODD/EVEN} function, and $(c)$ \texttt{CLOCK}$_4$ function with corresponding input and output spike trains. Input, hidden, and output neurons are depicted in yellow, green, and blue, respectively.
    }
    \label{fig_skip}
\end{figure}

\begin{remark}[On the relevance of the strict threshold]\label{rmk:strict_threshold}
    The SNN representation of the \texttt{ODD/EVEN} function crucially relies on the convention in \eqref{eq:potential_definition} that a spike is released if the membrane potential \emph{strictly} exceeds the threshold $1.$ Looking at the bottom hidden neuron (in green) in \Cref{fig_skip}$(b)$, this neuron receives weight $1$ from the input neuron and thus spikes at every second input spike. For an input spike train $\calI,$ the membrane potential of this neuron is $1$ at times $\calI_{(1)},\calI_{(3)},\ldots$ and therefore does not trigger a spike. At times $\calI_{(2k)}$ for $k \in \NN$ the membrane potential is $1 + e^{-(\calI_{(2k)} - \calI_{(2k-1)})/h} > 1$, resulting in a spike.  Choosing instead the threshold $\geq 1$, the weight would need to be chosen strictly smaller than one to avoid spiking at every input spike time, however for memory coefficient $h\in (0,\infty)$ and sufficiently large interarrival times between input spikes no output spike would be released. The issue could be mitigated by imposing an upper bound on the interarrival times.
\end{remark}

\subsection{Multiple-Input Memoryless Functions}\label{subsec:expl:memoryless_functions}
The following functions are memoryless and can be represented by SNNs with a bounded number of hidden neurons that is independent of the number of input neurons. In particular, the first three functions correspond to standard Boolean logical operations on spike trains, see \Cref{fig_logic} for explicit constructions, while the last two functions are useful for comparing spike trains. We also mention here that the \texttt{OR}- and \texttt{AND}-functions with arbitrarily many inputs can be represented by an SNN with a bounded number of neurons.  
\begin{itemize}
        \item \texttt{OR} function: For two inputs, the \texttt{OR} function is defined as $\TT^2\to \TT, (\calI, \calJ)\mapsto \calI \cup \calJ$. The $d$-component \texttt{OR} function is defined as $\TT^d\to \TT, (\calI_1, \dots, \calI_d)\mapsto  \bigcup_{j \in [d]}\calI_j$.
        \item \texttt{AND} function: For two inputs the \texttt{AND} function is defined as $\TT^2\to \TT, (\calI, \calJ)\mapsto \calI \cap \calJ$. The $d$-component \texttt{AND} function is defined as $\TT^d\to \TT, (\calI_1, \dots, \calI_d)\mapsto  \bigcap_{j \in [d]}\calI_j$.
        \item \texttt{MINUS} function: The \texttt{MINUS} function in terms of two input spike trains,  is defined as $\TT^2 \to \TT, (\calI, \calJ) \mapsto \calO \coloneqq \calI\backslash \calJ$, which removes all spikes from $\calJ$ within $\calI$. 
        \item \texttt{XOR} function: The \texttt{XOR} function is defined as $\TT^2 \to \TT, (\calI, \calJ) \mapsto \calO \coloneqq (\calI \setminus \calJ) \cup (\calJ \setminus \calI)$. The \texttt{XOR} outputs spikes if and only if one of the two input spike trains $\calI$ or $\calJ$ spike. 
        \item \texttt{IS-EQUAL} function: We define the \texttt{IS-EQUAL} function using three input spike trains as $\TT^3 \to \TT, (\calI, \calJ, \calD) \mapsto \big(\calI\cap \calJ \cap \calD\big) \cup \big((\calD \setminus \calI)\cap (\calD\setminus \calJ)\big)$. The output spikes if either both input spike trains $\calI$ and $\calJ$ simultaneously spike, or neither $\calI$ nor $\calJ$ spike but the dominating spike train $\calD$ does.
    \end{itemize}

\begin{figure}[t!]
    \centering
    \begin{subfigure}[b]{0.3\linewidth}
        \centering

\includegraphics[height=.6\linewidth, trim ={1mm 1mm 1mm 1mm},clip]{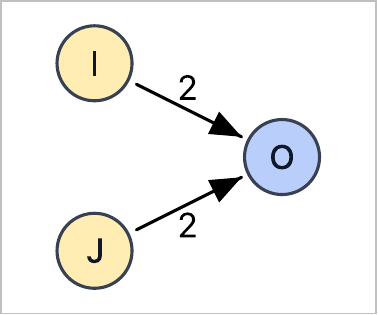}
        \caption{(a) \texttt{OR} function}
        \label{fig:or}
    \end{subfigure}
    \begin{subfigure}[b]{0.3\linewidth}
        \includegraphics[height=.6\linewidth, trim ={1mm 1mm 1mm 1mm},clip]{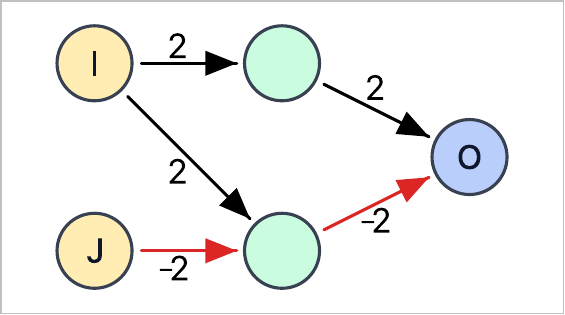}
        \caption{\quad \quad (b) \texttt{AND} function }
        \label{fig:and}
    \end{subfigure}
    \hspace{1cm}
    \begin{subfigure}[b]{0.3\linewidth}
    
    \includegraphics[height=.6\linewidth, trim ={1mm 1mm 1mm 1mm},clip]{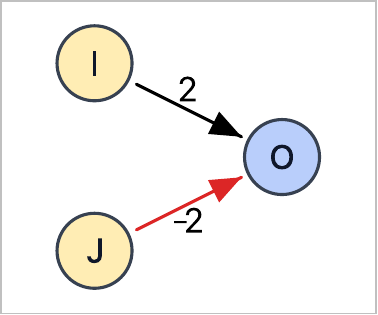}
        \caption{\!\!\!\!\!\!\!\!\!\!\!\!\!\!\!\!\!\!(c) \texttt{MINUS} function}
        \label{fig:minus}
    \end{subfigure}
    \caption{\textbf{SNNs implementing Boolean operations}. Explicit constructions via SNNs for the \texttt{OR}, \texttt{AND}, and \texttt{MINUS} function based on two inputs.}%
    \label{fig_logic}
\end{figure}

\subsection{Multiple-Input Compositional Functions with Bounded Memory} \label{subsec:expl:bounded_memory_functions}
All subsequent functions are defined in terms of some parameter $m\in \NN$ and are compositions of functions, each such function is either finite, periodic or Markovian with memory $O(m)$. %
\begin{itemize}
    \item \texttt{CEIL}$_m$ function: For two input spike trains $\calI, \calD \in \mathbb{T}$, the output is defined by $$\calO \coloneqq \{\calD_{(n)} : |(\calD_{(n-1)}, \calD_{(n)}] \cap \calI|\geq m, n \in [|\calD|-1]\}$$ if $|\calD|\geq 1$ and $\calO =\emptyset$ else,  with the convention that $\calD_{(0)} := 0$. Hence, there is an output spike at time $\calD_{(n)}$ if $\calI$ spikes at least $m$ times in the interval $(\calD_{n-1}, \calD_n]$ from $\calI$. This operation has Markovian memory $m+1$.
    \item \texttt{IS-APPROX-EQUAL}$_m$ function: Given three input spike trains $\calI, \calJ, \calD\in \TT$, the output $\calO$ is defined by the composition of the following operations for $k\in [m]$, 
    \begin{align*}
        \quad \quad \quad \quad \quad\calI_{k} &\coloneqq \texttt{CEIL}_k(\calI, \calD) \quad &
        \calJ_{k} &\coloneqq \texttt{CEIL}_k(\calJ, \calD), \quad \quad \quad \quad \quad \\
         \quad \quad \quad \quad \quad\calH_k &\coloneqq \texttt{IS-EQUAL}(\calI_{k}, \calJ_{k}, \calD)\quad 
        &\calO  &\coloneqq \texttt{AND}(\calH_1, \dots, \calH_m).\quad \quad \quad \quad \quad
    \end{align*}
    An output spike occurs at time $\calD_{(n)}$ if and only if $$\min\big(|\calI \cap (\calD_{(n-1)}, \calD_{(n)}]|,m\big) = \min\big(|\calJ \cap (\calD_{(n-1)}, \calD_{(n)}]|, m\big),$$  
    that is, both spike trains admit either the same number of spikes $\leq m$ or both spike trains exceed $m$ spikes for a time resolution determined by $\calD$. As a composition of operations with Markovian memory $m+1$ and $1$, this operation has again Markovian memory $m+1$.
    \item \texttt{DELAY}$_m$ function: For two input spike trains $\calI, \calD\in \TT$, with union $\calJ \coloneqq \calI \cup \calD$, the output spike train $\calO$ is defined by 
    \begin{align*}
        \calO \coloneqq \{J_{(n+m-1)} : n\in \NN, J_{(n)} \in \calI\}.
    \end{align*}
    This function delays every spike from $\calI$ by $m$ spikes with respect to the reference spike train $\calI \cup \calD$. This operation has Markovian memory $m$ as the output only depends on whether the spike that occurred $m-1$ positions before the current spike from $\calJ$ was from $\calI$ or not.
    \item \texttt{REPEAT}$_m$ function: For two input spike trains $\calI, \calD\in \TT$, with union $\calJ \coloneqq \calI \cup \calD$, the output spike train $\calO$ is defined by 
    \begin{align*}
        \calO \coloneqq \big\{J_{(n+k)} : n\in \NN, k\in \{0, \dots, m\}, J_{(n)} \in \calI\big\} = \bigcup_{k \in [m+1]} \texttt{DELAY}_k(\calI, \calJ).
    \end{align*}
    This function outputs a spike train that is dominated by $\calJ$ and dominates $\calI$, such that every input spike from $\calI$ is repeated exactly m times, and thus has Markovian memory $m+1$.
    \item \texttt{IF-THEN}$_{m}$ function with condition length $m\in \NN$: Given an input spike train $\calI\in\TT$, a binary vector $\br = (r_1, \dots, r_m)\in \{0,1\}^m$, and an exogenous spike train $\calD \in \TT$, the output spike train $\calO$ is defined by %
    \begin{align*}
         \calI^{\calD} \coloneqq \texttt{CEIL}_1(\calI, \calD),\qquad 
         \calR &\coloneqq \texttt{REPRESENT}_{\br}(\calD),\qquad 
         \calF\coloneqq \texttt{REPRESENT}_{\mathbf{1}_m}(\calD),\\
        &\!\!\!\!\!\!\!\!\!\!\!\!\!\!\!\!\!\calO \coloneqq  \texttt{SPIKE}_m(\texttt{IS-EQUAL}(\calI^{\calD}, \calR, \calF))
    \end{align*}
    where $\mathbf{1}_m= (1, \dots, 1) \in \{0, 1\}^m$. 
    The operation checks if for the first $m$ spikes of $\calD$, the ceiling counterpart of the input spike train $\calI$ with respect to $\calD$ coincides with the reference spike train $\calR$ defined by the binary vector $(r_1, \dots, r_m)$. Only if this condition is fulfilled, the output spikes once. This function is a composition of $m$-finite and an $m$-Markovian functions. %
    Notably, by attaching a \texttt{REPEAT}$_k$ function to the output, one can further amplify the classifier with $k$ additional output spikes. 
    \item \texttt{MEMORY}$_m$ function: Given $d$ input spike trains $\calI_1, \dots, \calI_d \in \TT$, this operation outputs $m \cdot d$ output spike trains $(\calO_{j}^{i})_{i\in [d]; j \in [m]}$ defined as
    \begin{align}
        \calO_{j}^{i} &\coloneqq \texttt{DELAY}_{j}\left(\calI_i, \textstyle \bigcup_{k=1}^{d} \calI_k\displaystyle\right) \quad \text{ for } i\in [d], j\in [m].
        \label{eq:123}
    \end{align}
    Specifically, $\calO_{j}^{i}$ outputs at a given time the $j$-th last spike from $\calI_i$ with respect to the union of all input spike trains. The \texttt{MEMORY}$_m$ function is crucial to prove the multiple-input universal representation theorem, see \Cref{fig:functional_representation_architecture} in \Cref{subsec:proofs_universal_representation_multiple_inputs} for a schematic depiction. %
\end{itemize}

\section{Classification of Spike Trains}\label{sec:classification_spike_trains}

In this section, the goal is to detect whether specific reference spike patterns occur in a given input spike train. If yes, then a classifier SNN should output a spike, and otherwise no spike should be produced. We believe that this is a crucial property for SNN-driven pattern recognition, temporal sequence processing, and decision-making. %

In the following, we first mathematically formalize the classification task and introduce the classifier function. We only focus on the single-input case. The extension to multiple-input spike trains is straightforward, see \Cref{rmk:multipleInputClassification}. Afterward, we show that the classifier function can be efficiently represented by SNNs and discuss for which classes the representation is simple/hard. 

We follow the framework of \inputdominated, causal, monotone scaling functions introduced in \Cref{sec:functional_representation}. However, to allow for time-dependent functions, that is, functions whose output at time $t$ may depend on the absolute time $t$ and not only on the relative timing of the input spikes, we additionally consider an exogenous spike train $\calD \in\TT$ which serves as a time reference or temporal grid as outlined in \Cref{rmk:time_dependent_functions}. For a given $m\in \NN$, we consider a (non-empty) \emph{class of reference patterns} $\mfR \subseteq \{0,1\}^m.$  Each element $R\in \mfR$ is a binary vector of length $m$ encoding a specific pattern of one's and zero's, encoding the presence/absence of a spike. The respective classifier function should then assess whether the most recent $m$ spikes from the input spike train correspond to one of the reference patterns in $\mfR$ and is defined as follows. 

\begin{definition}[Classifier function]
    Given $\calD\in\TT,$ a positive integer $m\leq |\calD|$, a class of reference patterns $\mfR \subseteq \{0,1\}^m,$ and an input spike train $\calI \in \TT$, the output $\calO$ of the \texttt{CLASSIFIER}$_{\calD, \mfR}$ function is defined by 
    \begin{align*}
       \calO \coloneqq \left\{\calD_{(n)} : \Big(\mathds{1}(\calD_{(n-m+i)} \in \calI^{\calD}) \Big)_{i\in [m]}\in \mfR, n \in [|\calD|], n\geq m \right\}, \qquad  \calI^\calD \coloneqq \texttt{CEIL}_1(\calI, \calD).
    \end{align*}
\end{definition}

The above function can be interpreted as follows. At each time point $D_{(n)}$ from the dominating spike train $\calD$, the function assesses the last $m$ spikes from the ceiling counterpart $\calI^\calD$ of the input spike train $\calI$ with respect to $\calD$. If the resulting binary vector indicating whether a spike occurred at these $m$ time points belongs to the reference class $\mfR$, then the function outputs a spike at time $D_{(n)}$. Otherwise, no spike is produced. As a function of both input $\calI$ and dominating spike train $\calD$, \texttt{CLASSIFIER}$_{\calD, \mfR}$ is \inputdominated, causal, satisfies the monotone scaling property, and has $m$-Markovian memory.

In the following, we show that this function can be efficiently represented by SNNs. %

\begin{theorem}[Classifier universal representation theorem]\label{thm:classifier_universal_representation}
Let $\calD\in \TT$, $m\in \NN$ with $m\leq |\calD|$, $r\in[2^{m}]$ and $h\in (0,\infty]$. Then, there exists an architecture $(L, \bp)$ and a positive integer $s$ satisfying $L = O(1+\log(m))$, $\|\bp\|_1 = O(m^2+r),$ and $s = O(m^2 +m r)$ such that for any $\mfR\subseteq {\{0,1\}^m}$ with $|\mfR| \leq r$, there exists an SNN $f\in \SNN_{h}(L,\bp,s)$ with memory $h$, called classifier SNN, such that
\begin{align*}
    \textup{\texttt{CLASSIFIER}}_{\calD, \mfR}(\calI)  = f(\calI, \calD) \quad \text{ for all } \calI \in \TT.
\end{align*}
\end{theorem}

The classifier universal representation theorem states that a pattern class $\mfR$ with elements of length $m$ can be efficiently implemented if either the function itself or its complement contains only a few patterns. In particular, the index $r$ can be interpreted as a sparsity index for the classifier function. Notably, we also see that depth of the network is completely independent of the complexity of the class $\mfR$ and only logarithmically dependent on the length $m$ of the patterns in $\mfR$. Conversely, this suggests that deep SNNs are capable of classifying extremely long patterns efficiently.

The proof of \Cref{thm:classifier_universal_representation} is similar to the proof of the multiple-input universal representation theorem. It relies on first utilizing the \texttt{CEIL}$_1$ operation to construct $\calI^\calD$, then composing it with the memory module  \texttt{MEMORY}$_m$ to store the last $m$ spikes from the input spike train $\calI$, and then constructing a memoryless function which assesses whether the outputs of the \texttt{MEMORY}$_m$  correspond to a pattern in $\mfR$. Details are deferred to \Cref{subsec:proofs_universal_representation_classifiers}. 

To demonstrate the implications of the above representation result, we provide some examples for specific classes $\mfR$ and discuss the complexity of the resulting classifier SNNs.

\begin{example} We apply \Cref{thm:classifier_universal_representation} to specific classes of reference patterns $\mfR\subset \{0,1\}^m$. This ensures that in all considered cases, there exists a classifier SNN with depth of order $L = O(1+\log(m))$. Below we provide bounds for the width and the number of parameters.
    \begin{enumerate}
        \item (Singleton class) If $\mfR$ consists of a single pattern, then, $\|\bp\|_1 + s = O(m^2).$
        \item (Translational invariant class) Given $\mfR_0\subset \{0,1\}^m$ consider the translational invariant class $\mfR \coloneqq 
         \{ (R_{i+k \bmod m})_{i\in [m]} \colon R\in \mfR_0,\, k = 0,\dots,m-1 \}.$ Then, $|\mfR| \leq m \cdot |\mfR_0|,$ $\|\bp\|_1=O(m^2+m\cdot |\mfR_0|)$ and $s=O(m^2\cdot |\mfR_0|)$ parameters.
        \item (Ball environment class) Given $\mfR_0 \subset \{0,1\}^m$ and a radius parameter $b\in \{1, \dots, m\}$, consider the ball environment class $\mfR \coloneqq 
         \{ R' \in \{0,1\}^m \colon \sum_{i=1}^m |R_i - R_i'| \leq b \text{ for some } R\in \mfR_0\}.$ Then $|\mfR| \leq |\mfR_0|\sum_{j=0}^b \binom{m}{j} =:M,$ $\|\bp\|_1 = O(m^2+ M),$ and $s=O(mM).$  %
        \item (Complementary class) Let $\mfR = \{0,1\}^m \backslash \mfR_0$ for some pattern class $\mfR_0$. Then, the classifier SNN $f_{\mfR}$ can be realized by negating the output of the classifier SNN $f_{\mfR_0}$ using the \texttt{MINUS}-function (introduced in \Cref{subsec:expl:memoryless_functions}) whenever there is an input spike. Hence, efficient representation is possible if $\min(|\mfR|, 2^{d}- |\mfR|)$ is small. 
    \end{enumerate}
\end{example}

We conclude this section with three remarks on extending the classification framework.

\begin{remark}[Multiple-input spike train classification]\label{rmk:multipleInputClassification}
The classification framework can be naturally extended to multiple-input spike trains. In this case, the classifier function would assess whether the most recent $m$ spikes from each input spike train correspond to a multidimensional reference pattern in a given class $\mfR$. Moreover, the resulting classifier SNN would scale with order $L = O(1+\log(m))$ with $\|\bp\|_1 =  O(m^2 d +r)$  where $r\coloneqq |\mfR|$, see \Cref{rmk:multipleInputClassification_explicit}. %
In the worst-case the size of $\mfR$ (and its complement) scales as $\Omega(2^{md})$, which leads to an architecture that is exponential in $m$ and $d$. Combinatorial arguments suggest that this is unavoidable. However, under more structure in $\mfR$, say, if $\mfR$ is characterized by some compositional nature, more efficient SNN representations are  possible.
\end{remark}

\begin{remark}[Post-processing after classification]
The classifier SNN considered here spikes exactly once if it identifies that the input spike pattern of the most recent spikes belongs to the reference class $\mfR$. This output can be immediately used to trigger further processing in a downstream task. Via the \texttt{REPEAT}$_k$ function introduced in \Cref{subsec:expl:bounded_memory_functions}, the output can also be repeated $k\in \NN$ times to reinforce the classification decision. This approach would also reconcile the two perspectives described in the introduction how the brain may process information, that is, whether information is encoded in the timing of the spike trains or the frequency of spikes. On the one hand, for a spike train to be classified to lie in a specific class, the specific spiking times of the input spike train are relevant. On the other hand, once an input spike train is identified to belong to this class, an optimal classifier SNN spikes as often as possible. Meanwhile, the classifier should be able to stop spiking (and thus forget) after some time, which is captured by the fact that the classification output is only repeated finitely many times.
\end{remark}

\section{Proofs for Main Results}\label{sec:proofs_main_results}

This section is devoted to the proofs of the universal representation results and the corresponding lower bounds stated in \Cref{sec:functional_representation}. 
As the single-input universal representation theorem (\Cref{thm:universal_representation_single_input}) and the classifier result (\Cref{thm:classifier_universal_representation}) both utilize the multiple-input universal representation theorem (\Cref{thm:function_representation}), we first prove the latter. Afterward, we prove the lower bounds for single- and multiple-input functions (Theorems \ref{cor_lo_numpara_train} and \ref{cor_lo_numpara_functional}). 
To simplify the exposition, we modularize the proofs via technical lemmata, whose proofs are deferred to the appendix. 

\subsection{Proof of the Multiple-Input Universal Representation Theorem}\label{subsec:proofs_universal_representation_multiple_inputs}

We first show \Cref{thm:function_representation}. Prerequisites are SNN representations of the \texttt{MEMORY}$_m$-function from \Cref{subsec:expl:bounded_memory_functions} and of multivariate Boolean, i.e., memoryless functions; see \Cref{fig:functional_representation_architecture} for a schematic depiction of the full proof. The proofs to these lemmata are given in~\Cref{subsec:proof:constructions}.

\begin{lemma}[Memory function]\label{lem:memory_module}
    Let $d, m\in \NN$ and let $h\in (0,\infty]$. Then, there exists a feedforward SNN $f \in \SNN_h(L, \bp, s)$ with depth $L  = O(1+\log(m))$, $\|\bp\|_1 = O(dm^2)$ neurons, $s = O(dm^2)$ non-zero weights, $p_0 = d,$ and $p_L = dm $ such that
    \begin{align*}
        f(\calI_1, \ldots, \calI_d) = \textup{\texttt{MEMORY}}_m(\calI_1, \ldots, \calI_d) \quad \text{ for all } \calI_1, \ldots, \calI_d \in \TT.
    \end{align*} 
\end{lemma}

\begin{lemma}[Memoryless Boolean function representation]\label{prop:boolean_function_representation}
    Let $d, r \in \NN$ and $h\in [0,\infty]$. Then, there exists a four-layer architecture $(4, \bp)$ and a positive integer $s$ satisfying $\|\bp\|_1 = O(d+r)$, and $s = O(dr)$ such that for any Boolean function $F: \{0,1\}^d \to \{0,1\}$ with $F(\mathbf{0}) = 0$ and $|\{\by \in \{0,1\}^d : F(\by)=1\}| = r$, there exists a feedforward SNN $f \in \SNN_h(4, \bp, s)$ such that for all $\calI_1, \ldots, \calI_d \in \TT$ it holds 
    \begin{align*}
        f(\calI_1, \ldots, \calI_d) = \Big\{t \in \textstyle \bigcup_{j\in [d]} \calI_j \displaystyle : F\big(
    \mathds{1}(t \in \calI_1), \mathds{1}(t \in \calI_2), \ldots, \mathds{1}(t \in \calI_d)
    \big) = 1\Big\}.
    \end{align*}
\end{lemma}

The SNN construction to prove \Cref{prop:boolean_function_representation} is inspired by the so-called \emph{sum-of-products} representation of Boolean functions and involves the implementation of \texttt{OR}- and \texttt{AND}-function with multiple inputs, as well as the \texttt{MINUS} function. In particular, the fact that SNNs can implement the \texttt{OR}- and \texttt{AND}-function with multiple inputs in constant depth is crucial to ensure that the overall function is represented with constant network depth.

We can now prove the multiple-input universal representation theorem. 
\begin{figure}
\centering
\includegraphics[width=\textwidth]{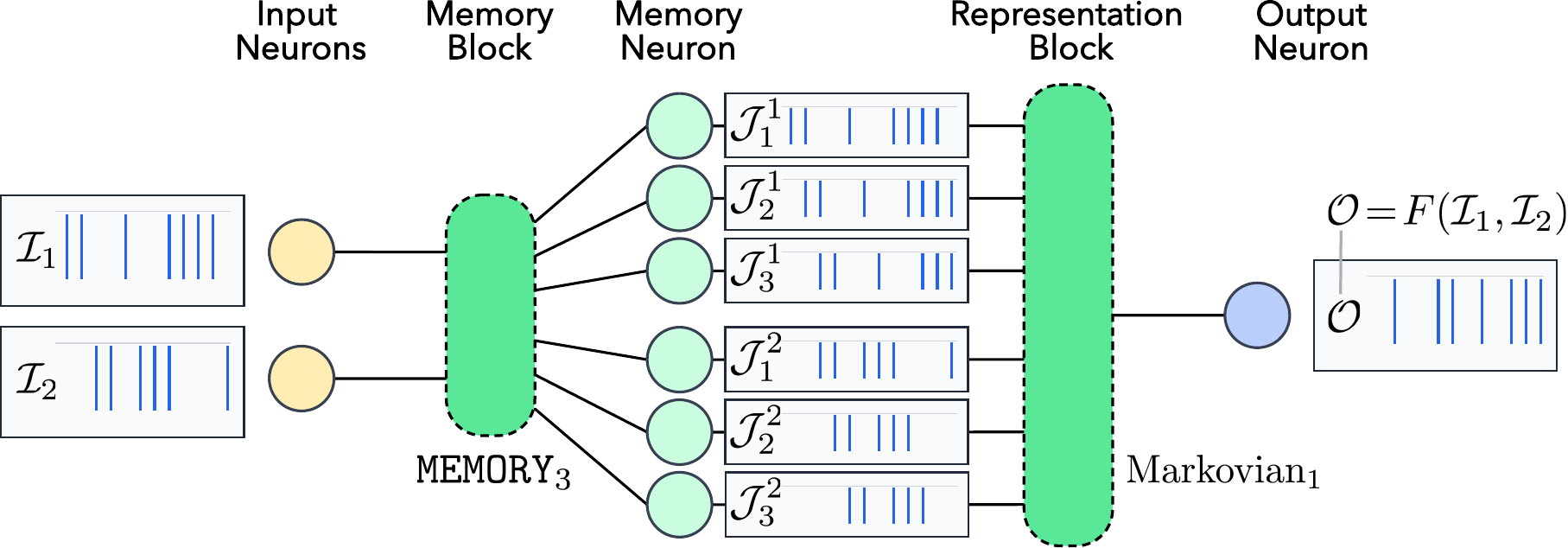}
\caption{\textbf{Schematic depiction of SNN architecture for representation of a function~$F$.} Input spike trains are first processed by a memory function (see \Cref{sec:examples_functions}) which outputs the input spike trains delayed by the appropriate index (in the figure, the index is~$3$). The delayed spike trains then serve as the input for a  Markovian$_1$-function, i.e., a~memoryless function, to produce the desired output spike train.}
\label{fig:functional_representation_architecture}
\end{figure}

\begin{proof}[Proof of Theorem \ref{thm:function_representation}]
    We first treat the case of $h\in (0, \infty]$, and discuss the memoryless case $h = 0$ at the end of the proof. Let $F\in \FF_{\mathrm{mm}}(d,m,r)$ be an $m$-Markovian function. %
      
    To show the assertion for $h \in (0, \infty]$, we denote the  spike trains for the $d$ inputs by $\calI_1, \ldots, \calI_d \in \TT$ and apply the function \texttt{MEMORY}$_m$ from Lemma \ref{lem:memory_module} to them. By \eqref{eq:123}, we end up with $dm$ neurons with spike trains $(\calJ_j^i)_{i\in[d], j\in[m]}$ where
    $$
    \calJ_j^i = \texttt{DELAY}_j\textstyle\left(\calI_i, \bigcup_{k\in [d]} \calI_k\right).
    $$
    Further, define the binary function $J_j^i(t) \coloneqq \mathds{1}(t \in \calJ_j^i)$ for $t\in (0,\infty)$. By causality and the monotone scaling property, we may, without loss of generality, assume that $\calI_1,\ldots,\calI_d\in\TTT{\NN}$, $\bigcup_{i=1}^d \calI_i = \NN$, and $t\in\NN$. Since $F$ is $m$-Markovian and \inputdominated, for any $t$, the indicator $\mathds{1}(t\in F(\calI_1,\ldots,\calI_d))$ depends only on the spike trains $(\calI_i\cap(t-m,t])_{i\in[d]}$, which are completely determined by the binary vector $(J_1^1(t), \ldots, J_m^1(t), \ldots, J_1^d(t), \ldots, J_m^d(t))$. Hence, there exists a Boolean function $F_\mathrm{B}: \{0,1\}^{md}\to \{0,1\}$ with $F_{\mrB}(\mathbf{0}) = 0$ such that for all $t\in (0, \infty)$ it holds
    \begin{align*}\label{eq:boolean_function_representation}
    t \in F(\calI_1, \ldots, \calI_d)
    \quad\iff\quad
    F_\mathrm{B}\big(J_1^1(t), \ldots, J_m^1(t), \ldots, J_1^d(t), \ldots, J_m^d(t)\big) = 1.
    \end{align*}
    We now use \Cref{prop:boolean_function_representation} (applied with input dimension $md$) to represent $F_\mathrm{B}$ with an SNN $f_\mathrm{B}$ with constant depth $4$, $O(md + r_\mathrm{B})$ neurons, and $O(md r_\mathrm{B})$ non-zero weights, where $r_\mathrm{B} := |\{y \in \{0,1\}^{md} : F_\mathrm{B}(y) = 1\}|$. Since $F \in \FF_{\mathrm{mm}}(d,m,r)$ is $r$-sparse, we have $r_\mathrm{B} \leq r$. Composing the memory module \texttt{MEMORY}$_m$ from \Cref{lem:memory_module} with the SNN $f_B$ which represents $F_B$ yields an SNN $f$ whose output neuron spikes at time $t$ if and only if $t \in F(\calI_1, \ldots, \calI_d)$.  Hence, $f(\calI_1, \ldots, \calI_d) = F(\calI_1, \ldots, \calI_d)$. By \Cref{lem:memory_module} and \Cref{prop:boolean_function_representation}, the  network $f$ has depth $L = O(1+\log(m))$, a total number of neurons $\|\bp\|_1 = O(dm^2 + r)$, and at most $s = O(dm^2 r)$ non-zero weights. 
    
    Finally, for the case $h = 0$ and $m= 1$,
    one can directly apply \Cref{prop:boolean_function_representation} to the input spike trains $\calI_1, \dots, \calI_d$ to obtain an SNN with the desired representation property.
\end{proof}

\subsection{Proof of the Single-Input Universal Representation Theorem}\label{subsec:proofs_universal_representation_single_input}

For the proof of \Cref{thm:universal_representation_single_input} we rely on the following lemma, which shows that the \texttt{SPIKE}$_1$ and \texttt{CEIL}$_1$ functions can be implemented by small feedforward SNNs.  

\begin{lemma}\label{lem:SPIKES_module}
For any $h\in (0, \infty]$, there exist feedforward SNNs $f\in \SNN_{h}(L_1, \bp_1, s_1)$ and $g\in \SNN_{h}(L_2, \bp_2, s_2)$ with bounded architecture, 
$\sum_{i \in [2]}L_i + \|\bp_i\|_1 + s_i = O(1)$, such that 
\begin{align*}
    f(\calI) = \textup{\texttt{SPIKE}}_{1}(\calI), \quad  g(\calI, \calD) = \textup{\texttt{CEIL}}_1(\calI,\calD) \quad  \text{for all } \calI, \calD \in \TT. 
\end{align*}
\end{lemma}

The statement is a direct consequence of \Cref{lem:drop_spike}$(i)$ and \Cref{lem:ceil_snn}. We now proceed with the proof of \Cref{thm:universal_representation_single_input}. \Cref{fig:SNN_3_new} depicts the architecture of the constructed feedforward SNN.

\begin{figure}
    \centering
\includegraphics[width=\textwidth]{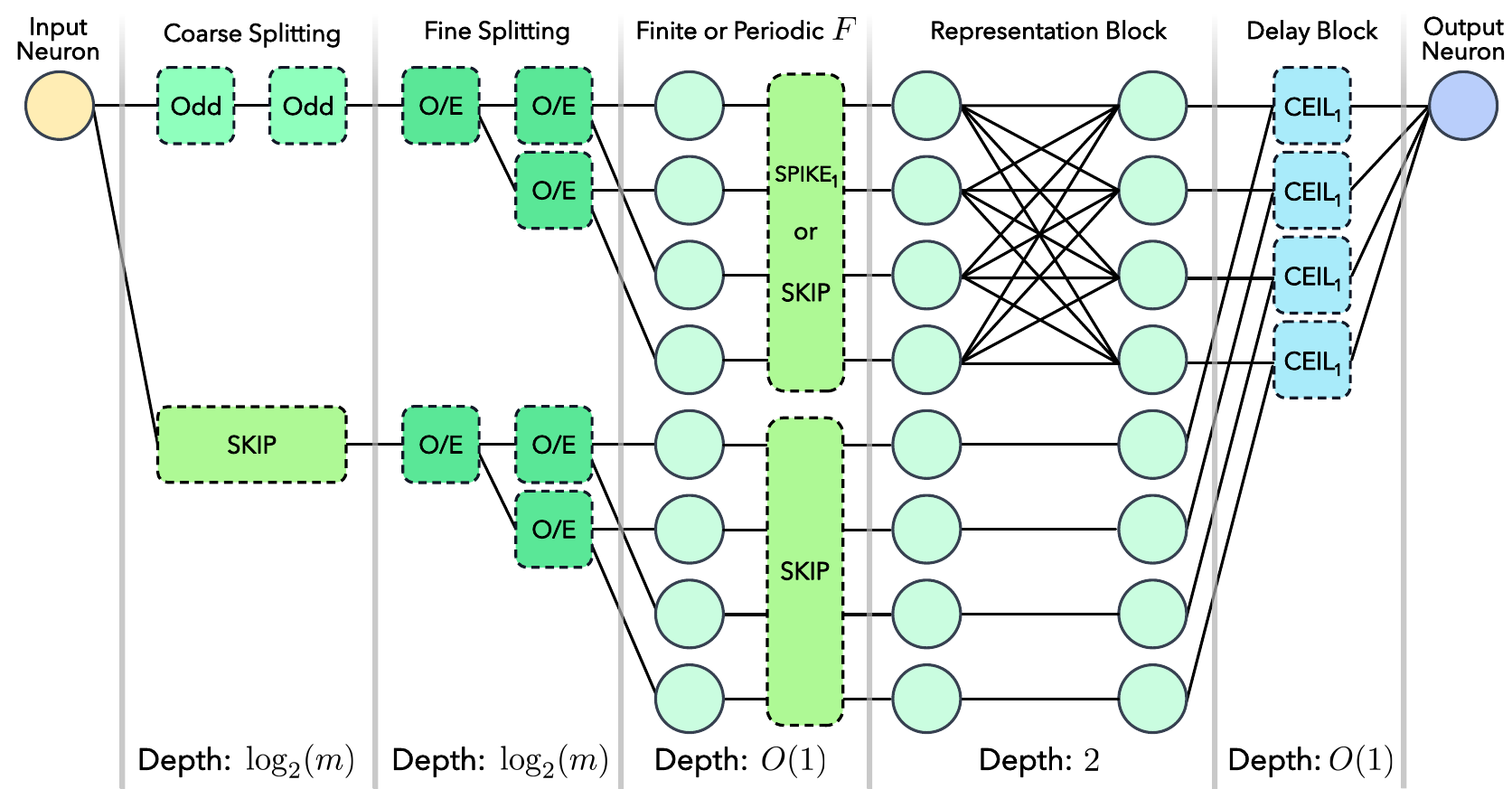}
\caption{\textbf{Architecture of the single-input SNN from proof of \Cref{thm:universal_representation_single_input} for $m = 16$.}  The architecture involves several blocks which corresponds to the different steps in the proof. If the function $F$ to be represented  is finite (resp.\ periodic), then the third block contains a componentwise \texttt{SPIKE}$_1$ (resp.\ \texttt{SKIP}) function. The exact representation of $F$ is encoded in the weights in the top representation block, some of which will be zero, all other depicted edges admit weight equals to $2$. The depth describes the neuron layers in each block. Note that in the first two blocks, since each \texttt{ODD/EVEN} module contains two layers of neurons, the depth is twice the number of modules.
}
\label{fig:SNN_3_new}
\end{figure}

\begin{proof}[Proof of \Cref{thm:universal_representation_single_input}]

    If $m= 1$ and $F$ is $1$-periodic, then one can use for $h\in [0,\infty]$ a SNN with a single output neuron that is directly connected to the input neuron with weight $2$ if $F(\NN) = \NN$  or weight $0$ if $F(\NN) = \emptyset$. 

    The construction for the SNN in the more general case consists of several different steps. %
    For the proof, we define $\overline m \coloneqq 4^{\lceil \log_4(m) \rceil} \geq m $. In particular, any function which is $m$-finite and $r$-sparse is in particular also $\overline m$-finite and $r$-sparse. In the periodic case, we assumed $m=4^q$, thus $\overline m = m$.  Further, without loss of generality, we can assume in the following that the input spike train $\calI$  fulfills $|\calI| = \infty$. Given that SNNs are causal, this implies that the same assertions apply if $\calI$ consists of only finitely many spikes. 

    \textit{Step 1: Coarse splitting of the input spike train.}
    We first apply the composition of $\log(\overline m)/2$ (which is an integer by definition of $\overline m$) \texttt{ODD} functions to the input neuron and denote the output neuron by $v_{1,1}$. In parallel, we attach enough \texttt{SKIP} connections to the input neurons until the depth is synchronized, the corresponding output neuron is denoted by $v_{1,2}$. The spike trains of $v_{1,1}$ and $v_{1,2}$ are denoted by $\calV_{1,1}$ and $\calV_{1,2}$, respectively. By construction,
    \begin{align*}
        \calV_{1,1} = \{ \calI_{(\sqrt{\overline m} \cdot k + 1)} : k\in \NN_0\}, \quad \calV_{1,2} = \calI. 
    \end{align*}
    This requires depth $O(\log(\overline m))$ and $\sqrt{\overline m}$ neurons and non-zero weights.

    \textit{Step 2: Fine splitting.} 
    We now apply a \texttt{CLOCK}$_{\sqrt{\overline m}}$ function to the neurons $v_{1,1}$ and $v_{1,2}$, which leads to $2\sqrt{\overline m}$ neurons $v_{2,1}, \dots, v_{2,2\sqrt{\overline m}}$ with spike trains $\calV_{2,1}, \dots, \calV_{2,2\sqrt{\overline m}}$, where
    \begin{align*}
        \calV_{2,j} &= \{ \calI_{(\overline m k + \sqrt{\overline m} (j-1) + 1)} \colon k \in \NN_0\}, \quad  \calV_{2,\sqrt{\overline m} + j} = \{ \calI_{(\sqrt{\overline m} \cdot k + j)} : k\in \NN_0 \}\quad \text{ for } j \in [\sqrt{\overline m}].
    \end{align*}
    For $j \in [\sqrt{\overline m }]$, each spike train $\calV_{2,j}$ only spikes once in a period of length $\overline m$,  while each spike train $\calV_{2,\sqrt{\overline m} + j}$ spikes $\sqrt{\overline m}$ times in a period of length $\overline m$.
    This requires depth $O(\log(\overline m)),$ as well as $O(\sqrt{\overline m})$ neurons and non-zero weights.

    \textit{Step 3: Finite or periodic functions.} If the function $F$ to be represented is $\overline m$-finite, then we apply the \texttt{SPIKE}$_1$ function from \Cref{lem:SPIKES_module} to each neuron $v_{2,1}, \dots, v_{2,\sqrt{\overline m}},$ denote the output neurons by $v_{3,1}, \dots, v_{3,\sqrt{\overline m}}$ and their spike trains by $\calV_{3,1}, \dots, \calV_{3,\sqrt{\overline m}}$. 
    For the other neurons $v_{2, \sqrt{\overline m}+1}, \dots, v_{2,2\sqrt{\overline m}}$, we attach \texttt{SKIP} connections and denote the spike trains by  $\calV_{3,\sqrt{\overline m}+1}, \dots, \calV_{3,2\sqrt{\overline m}}$. 
    Meanwhile, if the function to be represented is $\overline m$-periodic, then we instead attach \texttt{SKIP} connections to all neurons $v_{2,1}, \dots, v_{2,2\sqrt{\overline m}}$, and keep the notation identical. Notably, for periodic functions this step could also be omitted entirely. 
    
    The construction gives spike trains $\calV_{3,1}, \dots, \calV_{3,2\sqrt{\overline m}}.$  For $j\in [\overline m]$,
    \begin{align*}
        \calV_{3,j} &= \begin{cases}
        \{ \calI_{(\sqrt{\overline m} (j-1) + 1)} \}, & \text{ if } F \text{ is $\overline m$-finite}, \\
        \{ \calI_{(\overline m k + \sqrt{\overline m} (j-1) + 1)} : k\in \NN_0 \}, & \text{ if } F \text{ is $\overline m$-periodic},
        \end{cases} \quad \text{ and } \quad 
        \calV_{3,\sqrt{\overline m}+j} = \calV_{2,\sqrt{\overline m}+j},
    \end{align*}
    and it needs $O(1)$ layers and $O(\sqrt{\overline m})$ neurons and non-zero weights.
    
    \textit{Step 4: Representation block.} We now compose the neurons $v_{3,1}, \dots, v_{3,\sqrt{\overline m}}$ with another layer of neurons to obtain $v_{4,1}, \dots, v_{4,\sqrt{\overline m}}$. The weight $w_{ji}$ between neuron $v_{3,i}$ and neuron $v_{4,j}$ is set equal to $2$ if  $(i-1)\sqrt{\overline m}+j \in F([\overline m])$ and otherwise set to zero. This ensures that the spike train $\calV_{4,j}$ of neuron $v_{4,j}$ for $j \in [\sqrt{\overline m}]$ is given by 
    \begin{align*}
        \calV_{4,j} = \begin{cases}
        \{ \calI_{(\sqrt{\overline m}(i-1)+1)} :i \in [\sqrt{\overline m}],  (i-1)\sqrt{\overline m}+j \in F([\overline m]) \}, & \text{ if } F \text{ is $\overline m$-finite}, \\
        \{ \calI_{(\overline m k + \sqrt{\overline m}(i-1)+1)} :k \in \NN_0, i \in [\sqrt{\overline m}],  (i-1)\sqrt{\overline m}+j \in F([\overline m]) \}, & \text{ if } F \text{ is $\overline m$-periodic}.
        \end{cases}
    \end{align*}
    In parallel, we attach a \texttt{SKIP} connection to each neuron $v_{3,\sqrt{\overline m}+1}, \dots, v_{3,2\sqrt{\overline m}}$ and denote the spike trains by $\calV_{4,\sqrt{\overline m}+1}, \dots, \calV_{4,2\sqrt{\overline m}}$. 
    
    Overall, the representation block involves a single layer, $O(\sqrt{\overline m})$ neurons and $O(\sqrt{\overline m}+r)$ non-zero weights where $r$ denotes the sparsity index of $F$.

    \textit{Step 5: Delay block.} We now aim to shift the spike trains of $v_{4,1}, \dots, v_{4,\sqrt{\overline m}}$ to the correct position, as they can currently only spike whenever $\calV_{1,1}$ spikes. To this end, we apply to each pair $v_{4,j}$ and $v_{4,\sqrt{\overline m}+j}$ for $j\in [\sqrt{\overline m}]$ the \texttt{CEIL}$_1$ function from \Cref{lem:SPIKES_module}, where the respective inputs are selected as $\calI \coloneqq \calV_{4,j}$ and $\calD \coloneqq \calV_{4,\sqrt{\overline m}+j}$. We denote the resulting output neuron by $v_{5,j}$ with spike train $\calV_{5,j}$, which by construction is given for each $j \in [\sqrt{\overline m}]$ by 
    \begin{align*}
        \calV_{5,j} = \begin{cases}
        \{ \calI_{(\sqrt{\overline m}(i-1)+j)} :i \in [\sqrt{\overline m}],  (i-1)\sqrt{\overline m}+j \in F([\overline m]) \}, & \text{ if } F \text{ is $\overline m$-finite}, \\
        \{ \calI_{(\overline m k + \sqrt{\overline m}(i-1)+j)} :k \in \NN_0, i \in [\sqrt{\overline m}], (i-1)\sqrt{\overline m}+j \in F([\overline m]) \}, & \text{ if } F \text{ is $\overline m$-periodic}.
        \end{cases}
    \end{align*}
    This step requires $O(1)$ layers, $O(\sqrt{\overline m})$ neurons and $O(\sqrt{\overline m})$ non-zero weights.

    \textit{Step 6: Output. }
    Consider a single output neuron $o$ by attaching weight $2$ from each neuron $v_{5,j}$ for $j\in [\sqrt{\overline m}]$ to $o$. The spike train $\calO$ of $o$ is given by 
    \begin{align*}
        \calO = \begin{cases}
        \{ \calI_{(t)} : t\in F([\overline m]) \}, & \text{ if } F \text{ is $\overline m$-finite}, \\
        \{ \calI_{(\overline m k + t)} : k\in \NN_0, t\in F([\overline m]) \}, & \text{ if } F \text{ is $\overline m$-periodic},
        \end{cases}
    \end{align*}
    which equals $F(\calI)$ by construction. Collectively, this construction requires $O(\log(\overline m)) = O(1+\log(m))$ layers, $O(\sqrt{\overline m}) = O(\sqrt{m})$ neurons and $O(\sqrt{\overline m} + r) = O(\sqrt{m} + r)$ non-zero weights. This concludes the proof.
\end{proof}

\subsection{Proof of Classifier Universal Representation Theorem}\label{subsec:proofs_universal_representation_classifiers}

\begin{proof}[Proof of \Cref{thm:classifier_universal_representation}]

    We first rewrite \texttt{CLASSIFIER}$_{\calD, \mfR}$ into a composition of simpler functions. 
    Let $\calI\in \TT$ be the input spike train and let $\calD$ be the exogenous dominating spike train. In the following, we first show that there exists a $1$-Markovian function $F_{\mfR}:\TT^{2m} \to \TT$, fully determined by $\mfR$, such that    
    \begin{equation}
        \texttt{CLASSIFIER}_{\calD, \mfR}(\calI) = F_{\mfR}\Big(\texttt{MEMORY}_m(
        \texttt{CEIL}_1\big(\calI, \calD\big), \calD)\Big) \quad \text{ for all }\calI \in \TT.\label{eq:class_comp}
    \end{equation}
    To define $F_{\mfR}$, note that the output of the memory function $\texttt{MEMORY}_m(\calI^{\calD}, \calD)$, with $\calI^{\calD}\coloneqq \texttt{CEIL}_1(\calI, \calD)$,  comprises $2m$ output spike trains  $\calV_j^{i}$ for $i\in[2]$ and $j\in[m]$ where, by domination $\calI^{\calD}\dom \calD$, it holds for all $n\in \NN$ that
    \begin{align*}
        \sum_{j \in [m]} \mathds{1}(\calD_{(n)} \in \calV_j^{2}) \begin{cases}
        \leq m-1 & \text{ if } n \leq m-1, \\
        = m &\text{ if } n \geq m.
        \end{cases}
    \end{align*}
Moreover, for all $n\in \NN$, $n\geq m$ and $j\in[m]$ the following equivalence is met, 
    $$\calD_{(n)} \in \calV_j^{1} \quad \Longleftrightarrow \quad \calD_{(n-m+j)} \in \calI^{\calD}.$$
    The function $F_{\mfR}$ is then defined as a function acting on the $2m$ input spike trains $\calJ_j^{i}$ for $i\in[2]$ and $j\in[m]$ via
    \begin{align*}
        \mathds{1}\left(t \in  F_{\mfR}(\calJ_1^{1}, \ldots, \calJ_m^{1}, \calJ_1^{2}, \ldots, \calJ_m^{2})\right) = \mathds{1}\left(
        (\mathds{1}( t \in \calJ_j^{1}))_{j \in [m]} \in \mfR \right) \cdot \mathds{1}\left(\textstyle\sum_{j \in [m]} \mathds{1}(t \in \calJ_j^{2}) =m\right)
    \end{align*} 
    for all $t \in (0, \infty)$. In particular, note that $F_{\mfR}$ is $1$-Markovian since the output at time $t$ only depends on the input spike trains at time $t$, and that it is $|\mfR|$-sparse. 
    Combining the above three displays, we conclude that \eqref{eq:class_comp} is satisfied. 

    For the representation of the $\texttt{CLASSIFIER}_{\calD, \mfR}$-function with a feedforward SNN we first apply a feedforward SNN representing the  $\texttt{CEIL}_1$-function to the two input neurons. In parallel, we add sufficiently many \texttt{SKIP}-connections to the neuron with spike train $\calD$. By \Cref{lem:SPIKES_module} this architecture has $O(1)$ depth, $O(1)$ neurons, and $O(1)$ non-zero weights. Next, we attach the memory module $\texttt{MEMORY}_m$ from \Cref{lem:memory_module}, which has $O(1+\log(m))$ layers, $O(m^2)$ neurons, and $O(m^2)$ non-zero weights. Finally, we apply the function $F_{\mfR}$ to the output of the memory module. By \Cref{prop:boolean_function_representation}, $F_{\mfR}$ can be represented by a feedforward SNN with $O(1)$ depth, $O(m+ |\mfR|)$ neurons, and $O(m |\mfR|)$ non-zero weights. Combined, this yields a feedforward SNN representing $\texttt{CLASSIFIER}_{\calD, \mfR}$ with $O(1+\log(m))$ depth, $O(m^2 + |\mfR|)$ neurons, and $O(m^2 + m |\mfR|)$ non-zero weights.
    \end{proof}

    \begin{remark}\label{rmk:multipleInputClassification_explicit}
    Our proof also naturally extends to the multivariate setting with $d$ input spike trains. In this case, the classifier function would be characterized for a set $\mfR \subseteq \{0,1\}^{m\times d}$ by 
    \begin{align*}
        \texttt{CLASSIFIER}_{\calD, \mfR}(\calI_1, \ldots, \calI_d) = F_{\mfR}\left(\texttt{MEMORY}_m(\texttt{CEIL}_1(\calI_1, \calD), \ldots, \texttt{CEIL}_1(\calI_d, \calD), \calD)\right)
    \end{align*}
    where $F_{\mfR}:\TT^{2 m d} \to \TT$ is a $1$-Markovian function defined analogously as in the proof of \Cref{thm:classifier_universal_representation}. In particular,  the \texttt{MEMORY}$_m$ function needs depth $O(1+\log(m))$ with $O(m^2d)$ neurons and non-zero weights (\Cref{lem:memory_module}), while $F_{\mfR}$ can be represented by a feedforward SNN with $O(1)$ depth, $O(m d + |\mfR|)$ neurons, and $O(m d |\mfR|)$ non-zero weights (\Cref{prop:boolean_function_representation}). 
    \end{remark}

\subsection{Proof for the Complexity Upper Bound of SNN Computational Units}\label{subsec:proofs_cardinality_upper_bound}

To prove the lower bounds in Theorems  \ref{cor_lo_numpara_train} and \ref{cor_lo_numpara_functional}, we first derive an upper bound on the number of distinct functions that a class of feedforward SNNs with fixed architecture can represent over a given set of input spike trains. A comparison with lower bounds on the number of distinct functions within the function classes $\FF_{\mathrm{fin}}(m,r)$, $\FF_{\mathrm{per}}(m,r)$, and $\FF_{\mathrm{mm}}(d,m,r)$, yields then lower bounds on the number of required SNN network parameters to represent all functions within these classes.%

\begin{theorem}[Expressiveness of feedforward SNNs] \label{thm:SNN_expressivity}
    Consider a finite collection of $d$-tuples of spike trains $\bbI\subseteq \TT^d$ and set $n\coloneqq |\bbI|$. Enumerate its elements by $\bbI = \left\{ \smash{(\calI^{(i)}_1, \dots, \calI^{(i)}_{d})}, i \in [n] \right\},$ define $T_i := |\bigcup_{j\in[d]} \smash{\calI^{(i)}_j}|$, and set 
    $T_{\operatorname{sum}} := \sum_{i\in [n]} (T_i+1)^2$. 
    Then, for any $h\in [0, \infty]$, $L \in \mathbb{N}$, $\mathbf{p} := (d,p_1,\ldots,p_L) \in \mathbb{N}^{L+1}$, and $s \in \mathbb{N}$ with $s \geq \sum_{\ell=1}^{L} (p_{\ell-1} \vee p_{\ell})$, 
    it holds 
    \begin{align*}
       \big|\left\{\bbf\colon \bbI \to \mathbb{T}^{p_L} \colon \bbf \in \SNN_h(L,\mathbf{p}, s)  \right\}\big|
       \leq \big(8 e s^2 T_{\operatorname{sum}}\big)^{s}.
    \end{align*}
\end{theorem}

If $\II$ consists of a single spike train $(d = 1)$ with $m\in \NN$ spikes, the bound simplifies to $\big(8 e s^2 (m+1)^2\big)^{s}$. This confirms that the number of distinct functions grows at most polynomially in $m$ and exponentially in $s$. Likewise, if $\II$ contains multiple spike trains, the expressiveness scales polynomially in $n$ and the total number of spikes among $\calI_1, \dots, \calI_d$, but it scales exponentially in the number of weights. 
The proof of \Cref{thm:SNN_expressivity} is detailed in \Cref{subsec:proof:SNN_expressivity}. 

We now prove the lower bound on the required number of SNN parameters for single-input universal representation.

\begin{proof}[Proof of \Cref{cor_lo_numpara_train}]
By the monotone scaling property, it suffices to consider the input spike train $\calI := [m]$.
For every $\calO \subseteq [m]$ with $|\calO| \leq r$, there exists $F \in \FF_{\mathrm{fin}}(m,r)$  such that $F(\calI) = \calO$.
This means that the cardinality of possible output spike trains of feedforward SNNs in $\SNN_{h}(L,\bp,s)$ with input $\calI$ should be at least $ \sum_{k=0}^r \binom{m}{k}$.  
    By applying Theorem~\ref{thm:SNN_expressivity} with $d=p_L=1$, $\bbI = \{ \calI \}$ and $T_{\operatorname{sum}} = (m+1)^2$, we get
    $$\big(8 e s^2 (m+1)^2\big)^{s} \geq \left|\big\{f(\calI): f \in \SNN_h(L, \mathbf{p}, s)\big\}\right| \geq \left|\big\{\calO \subset [m], |\calO| \leq r \big\}\right| = \sum_{k=0}^r \binom{m}{k} \geq 2^r,$$
    where we used $\tbinom{m}{k} \geq \tbinom{r}{k}$ for the last inequality. This yields $s\in \NN$, and since $8e\leq 21.75<25=5^2$, we have 
        $s \geq \frac{r}{\log(8e s^2 (m+1)^2)} \geq \frac{r}{2 \log(5 s (m+1))}.$ 
    We thus arrive, since $r\leq m$, at
    \begin{align*}
        s &\geq \frac{m+1}{5} \mathds{1}\left(s\geq \frac{m+1}{5} \right)  + \frac{r}{2 \log\big(5 s (m+1)\big)} \mathds{1}\left(1 \leq s < \frac{m+1}{5} \right)\\
        &\geq \frac{m+1}{5} \mathds{1}\left(s\geq \frac{m+1}{5} \right)  + \frac{r}{4 \log (m+1)} \mathds{1}\left(1 \leq s < \frac{m+1}{5} \right)\\
        &\geq \frac{m+1}{5} \land \frac{r}{4 \log (m+1)}\geq \frac{r}{5 \log (m+1)}.     
    \end{align*} 
    The assertion for $\FF_{\mathrm{per}}(m,r)$ can be obtained in the same way.
\end{proof}

\begin{proof}[Proof of \Cref{cor_lo_numpara_functional}]
We consider the set of input spike trains 
    $$\II := \Big\{(\calI_1, \ldots, \calI_d) \in \TTT{[m]}^d: \textstyle\bigcup_{j=1}^d \calI_j = [m]\Big\}.$$ 
If $(\calI_1, \ldots, \calI_d) \in \II$, 
then for each $t \in [m]$ at least one of the input spike trains $\calI_1,\ldots, \calI_d$ should include $t$, and hence the number of possible values of $(\mathds{1}(t \in \calI_j))_{j \in [d]}$ is $2^d-1$.
This yields $|\II| = (2^d-1)^m$.
    By applying Theorem \ref{thm:SNN_expressivity} with $d_0=d$, $d_L=1,$ and
    $T_{\operatorname{sum}} = \sum_{\calI \in \bbI} (m+1)^2 = (2^d-1)^m (m+1)^2,$
    we get
    $$\Big|\big\{f\colon \bbI \to \mathbb{T}, f \in \SNN_h(L,\mathbf{d}, s)  \big\}\Big| \leq \big(8 e s^2 (2^d-1)^m (m+1)^2 \big)^{s}.$$
    Let $a(r) := r \land (2^{d}-1)^{m}$. According to Lemma \ref{lem:cardinality_causal_functions_2}  below  we have 
    $$\Big|\big\{\FF : \II \to \TT, (\calI_1, \ldots, \calI_d) \mapsto F(\calI_1, \ldots, \calI_d) \;\colon\;  F \in \FF_{\mathrm{mm}}(d,m,r)\big\}\Big| \geq 2^{a(r)}.$$
    Hence, for universal representation it is necessary that 
    $s \log(8e s^2 (2^d-1)^m (m+1)^2) \geq a(r)$, which implies $s\in \NN$ by $d \geq 2$. 
    We thus conclude, since $\log a(r) \leq dm$, that
    \begin{align*}
        s &\geq \frac{a(r)}{\sqrt{8e}(m+1)} \mathds{1}\left( s \geq \frac{a(r)}{\sqrt{8e}(m+1)} \right)
        + \frac{a(r)}{\log\big(8e s^2 (2^d-1)^m (m+1)^2 \big)} \mathds{1}\left( 1\leq s < \frac{a(r)}{\sqrt{8e}(m+1)} \right) \\
&\geq \frac{a(r)}{\sqrt{8e}(m+1)} \mathds{1}\left( s \geq \frac{a(r)}{\sqrt{8e}(m+1)} \right)
        + \frac{a(r)}{md + 2\log a(r)} \mathds{1}\left( 1\leq s < \frac{a(r)}{\sqrt{8e}(m+1)} \right) \\
        &\geq \frac{a(r)}{\sqrt{8e}(m+1)}
        \land \frac{a(r)}{md + 2\log a(r)}\geq \frac{a(r)}{5md}.      &\qedhere
    \end{align*}
\end{proof}

\begin{lemma}\label{lem:cardinality_causal_functions_2}
    For $d,m \in \NN$ with $d \geq 2$, the input spike train class
    $\II := \{(\calI_1, \ldots, \calI_d) \in \TTT{[m]}^d: \bigcup_{j=1}^d \calI_j = [m]\},$
    and any     
    $r \in [2^{md}]$, 
    $$\Big|\big\{F : \II \to \TT, (\calI_1, \ldots, \calI_d) \mapsto F(\calI_1, \ldots, \calI_d), F \in \FF_{\mathrm{mm}}(d,m,r)\big\}\Big| \geq 2^{r} \land 2^{(2^{d}-1)^{m}}.$$ 
\end{lemma}  

The proof of \Cref{lem:cardinality_causal_functions_2} is deferred to \Cref{subsec:app:proof_cardinality}. 

We finally show that lower bounds on the number of required nonzero weights in an SNN can be directly translated to lower bounds on the number of required neurons.

\begin{lemma}\label{prop:connection_neuron_lower_bound}     
     Given a feedforward SNN with architecture $(L,\bp= (p_0,\ldots,p_L))$ and number of (non-zero) network weights $s$.
\begin{enumerate}
    \item If $L \in \{1,2\}$, then $\|\bp\|_1 > s/(p_0+p_L)$.
    \item If $L \geq 3$, then $\|\bp\|_1 > \sqrt{2s}$. %
    \item If the number of outgoing edges from each neuron (out-degree) is upper bounded by $r$, then $\|\bp\|_1 > s/r$.
    \end{enumerate}
\end{lemma}
\begin{proof}
Assertion $(i)$ follows from $s \leq p_0 p_1 \leq (p_0+p_1)^2$ for $L=1$ and $s \leq p_0 p_1 + p_1 p_2 < (p_0+p_2)(p_0+p_1+p_2)$ for $L=2$. Assertion $(ii)$ is a consequence of $2s \leq 2p_0 p_1 + \ldots + 2p_{L-1} p_L < (p_0+\ldots+p_L)^2.$ Finally, Assertion $(iii)$ follows from    
    $s \leq p_0 r + p_1 r + \ldots + p_{L-1} r < r (p_0+\ldots+p_L)$.
\end{proof}

\appendix
\allowdisplaybreaks

\section{Proofs on Properties of Spiking Neural Networks}

\subsection{On the Definition of the Membrane Potential}\label{app:proofs_section2}

\begin{proof}[Proof of \Cref{prop:potential_well_defined}]
    Define $\calI \coloneqq \calI_1 \cup \dots \cup \calI_d$ and %
    let $\calI_{(|\calI|+1)}\coloneqq \infty$ if $|\calI| < \infty$. %
    
    We prove both claims at once by mathematical induction on the intervals $[0,\calI_{(k)})$. First consider $k = 1$. For $t = 0\in [0, \calI_{(1)})$ we know by definition that $P(0)=0$ and $\calS\cap \{0\} = \emptyset$. For $t\in (0, \calI_{(1)})$, we have $\prev{t} = 0$ and hence $ P(t) = ( \mathds{1}(P(0)\leq 1) P(0)\exp(-t/h))_+ = 0$ which confirms that $P$ is well-defined on $[0, \calI_{(1)})$ and that $\calS\cap [0, \calI_{(1)}) = \emptyset$,  proving the base case. 

    Now, suppose that the claim holds on $[0, \calI_{(k)})$ for some $k\in [|\calI|]$. We need to show that it also holds for $t\in [0, \calI_{(k+1)})$. To this end, first let $t= \calI_{(k)}$ which implies $\prev{t} = \calI_{(k-1)}$. Then, by the induction hypothesis, $P(\prev{t}) = P(\calI_{(k-1)})$ is well-defined, and thus %
    \begin{align*}
        P(\calI_{(k)}) &= \left( \mathds{1}( P(\calI_{(k-1)})\leq 1) P(\calI_{(k-1)})\exp\big(-(\calI_{(k)}-\calI_{(k-1)})/h\big)   + \sum_{j : \calI_{(k)} \in \calI_j} w_j\right)_{+},
    \end{align*}
    is also well-defined. Moreover, for $t\in (\calI_{(k)}, \calI_{(k+1)})$, we have $\prev{t} = \calI_{(k)}$, and hence 
    \begin{align*}
        P(t) &= \mathds{1}(P(\calI_{(k)})\leq 1) P(\calI_{(k)})\exp\big(-(t-\calI_{(k)})/h\big)
    \end{align*}
    is well-defined. In particular, if $P(\calI_{(k)}) >1$, then $\calI_{(k)} \in \calS$ and thus $P(t) = 0$ for $t\in (\calI_{(k)}, \calI_{(k+1)})$, consequently, $\calS \cap [\calI_{(k)}, \calI_{(k+1)}) \subseteq \{\calI_{(k)}\}$. Meanwhile, if $P(\calI_{(k)})\leq 1$, then $\calI_{(k)} \not\in \calS$ and $P(t) = P(\calI_{(k)}) \exp(-(t-\calI_{(k)})/h)\leq P(\calI_{(k)})\leq 1$, implying that $\calS\cap [\calI_{(k)}, \calI_{(k+1)}) = \emptyset$. This concludes the induction step and the proof. 
\end{proof}

\subsection{Casting SNNs as Recurrent ANNs with State-Space Block}\label{app:SNNasANNs}

 For the representation of the SNN layer in terms of an equivalent recurrent ANN with state block we make the following conventions. In the SNN layer, the input consists of spike trains $\calI_1, \dots, \calI_{p_{\operatorname{in}}}$ which are all dominated by $\calN_{\delta}$. 
For the ANN, the different input spike trains are encoded as a sequential input vector $(\bx^{(t)})_{t \in \mathbb{N}}$ with
    $$\bx^{(t)} := (\mathds{1}(t\delta \in \calI_j))_{j \in [p_{\operatorname{in}}]}\in \{0,1\}^{p_{\operatorname{in}}}, \quad t\in \NN.$$     
Conversely, the sequential input vector $(\bx^{(t)})_{t \in \mathbb{N}}$ can be used to reconstruct the input spike trains via $$\calI_j = \{\delta t: t \in \mathbb{N}, x_j^{(t)} = 1\} \quad \text{for }j \in [p_{\operatorname{in}}].$$

To model an SNN layer $\boldsymbol{\psi}_{h,W}: \mathbb{T}^{p_{\operatorname{in}}} \to \mathbb{T}^{p_{\operatorname{out}}}$ with $h\in [0,\infty]$ and weight matrix $W := (\mathbf{w}_1, \ldots, \mathbf{w}_{p_{\operatorname{out}}})^{\top} \in \RR^{p_{\operatorname{out}} \times p_{\operatorname{in}} }$, we introduce a nonlinear state-space blocks. This state-space block consists of $2p_{\operatorname{out}}$ hidden neurons, where for each output neuron, one hidden neuron encodes the potential and the other models the output spike train.  Concretely, it is parametrized by matrices $W'\in \RR^{2p_{\operatorname{out}}\times p_{\operatorname{in}}}, V\in \RR^{2p_{\operatorname{out}}\times 2p_{\operatorname{out}}}$, $U\in \RR^{p_{\operatorname{out}}\times 2p_{\operatorname{out}}}$. Given a sequence of $p_{\operatorname{in}}$-dimensional input vectors $(\bx^{(t)})_{t=1}^{\infty}$, the activations of the values of the hidden neurons vector $\bh^{(t)}$ and the output neurons vector $\bo^{(t)}$ are updated according to 
\begin{align*}
    \bh^{(t)} &:= \bsigma\Big( W'  \bx^{(t)} + V \bh^{(t-1)}\Big),\qquad 
    \bo^{(t)} := U \bh^{(t)}
\end{align*}
where $\bh^{(0)} := \mathbf{0}_{2p_{\operatorname{out}}}$ and the activation function $\bsigma: \mathbb{R}^{2p_{\operatorname{out}}} \to \mathbb{R}^{2p_{\operatorname{out}}}$ acts component-wise and is defined for $i \in [2p_{\operatorname{out}}]$ by 
\begin{align*}
    \sigma_i(z) := \begin{cases}
    \mathds{1}(z>1), & \text{if } i \leq p_{\operatorname{out}}, \\
    z \mathds{1}(0 \leq z \leq 1), & \text{if } i > p_{\operatorname{out}}.
    \end{cases}
\end{align*}   
To imitate the behavior of the SNN layer, we set the parameters as 
\begin{align*}
    W' := \begin{pmatrix}W\\
        W
    \end{pmatrix} ,& \quad
    V := \begin{pmatrix}
    0_{p_{\operatorname{out}}\times p_{\operatorname{out}}} & e^{-\delta/h} I_{p_{\operatorname{out}}} \\
    0_{p_{\operatorname{out}}\times p_{\operatorname{out}}} & e^{-\delta/h} I_{p_{\operatorname{out}}}
    \end{pmatrix}, \quad 
    U := \begin{pmatrix}
    I_{p_{\operatorname{out}}} & 0_{p_{\operatorname{out}}\times p_{\operatorname{out}}}
    \end{pmatrix}.
\end{align*}
We now prove that the $j$-th entry of the output neuron vector $\bo^{(t)}$ of the state-space block models equals one if and only if there is a spike in the $j$-th output neuron of the SNN layer at time $t\delta$. To see this, note that the $j$-th entry for $j \in [p_{\operatorname{out}}]$ of the  hidden neuron activation vector $h^{(t)}$ at time $t$ is given by  
\begin{align*}
    h^{(t)}_{j} &= \begin{cases}
    \mathds{1}(r_j^{(t)} > 1), & j \in [p_{\operatorname{out}}],\\
    r_j^{(t)} \mathds{1}(0\leq r_j^{(t)} \leq 1), & j \in [p_{\operatorname{out}}]+p_{\operatorname{out}},
    \end{cases}
    \quad  r_j^{(t)} := \bw_j^{\top} \bx^{(t)} + e^{-\delta/h} h^{(t-1)}_{j+p_{\operatorname{out}}}.
\end{align*}
Hence, $h^{(t)}_{j+p_{\operatorname{out}}}$ for $j \in [p_{\operatorname{out}}]$ equals the neuronwise potential among each computational unit $P(t)$ in \eqref{eq:potential_definition}. Finally, we get 
$$\bo^{(t)} = (h^{(t)}_{j})_{j\in [p_{\operatorname{out}}]} = \mathds{1}\Big(t\delta \in \phi_{h,\mathbf{w}}(\calI_1, \ldots, \calI_{p_{\operatorname{in}}})\Big).$$
Altogether, this confirms that the SNN layer can be represented by the above recurrent ANN with state-space block. Feedforward SNNs are defined by compositions of SNN layers, while 
recurrent neural networks are defined by compositions of nonlinear state-space blocks. 
We thus confirm that 
any feedforward SNN with $(L,\mathbf{p})$ architecture with $\mathbf{p} := (p_0,p_1,\ldots,p_L)$ can be represented by recurrent neural networks with $L$ state-space blocks, and for $\ell \in [L]$, the $\ell$-th state-space block contains $2p_\ell$ hidden units.

\section{Extensions of the  Single-Input Universal Representation Theorem}\label{sec:representation_spike_trains}

In the following, we derive variations of the single-input universal representation theorem (\Cref{thm:universal_representation_single_input}), by restricting to shallow SNNs, proving the necessity of negative weights, and deriving the universal representation theorem with bounded out-degree.

\subsection{Universal Representation of Shallow Networks}

Regarding the expressive power of shallow SNNs, i.e., one-hidden layer SNNs ($L=2$), we first show that there exist input-output spike train pairs that cannot be represented by shallow SNNs. This represents a fundamental limitation of shallow SNNs, which is different from the classical universal approximation theorems for ANNs. %
Complementary to this result, we then provide sufficient conditions on the input spike train that allow shallow SNNs to represent all dominated output spike trains. Note that this is equivalent to the representation of $m$-finite functions for input spike trains restricted to consist of only $m$~spikes. 

\begin{proposition}\label{prop:onelayer_counter}
    For any finite $h\in (0, \infty)$ there exist two spike trains $\calI, \calO\in \TT$ with $\calO\dom \calI$ such that no shallow SNN $f \in \SNN_h(2,(1, \ell, 1), s)$ with $\ell,s \in \NN$ satisfies $f(\calI) = \calO$. 
\end{proposition}

\begin{proof}
    Consider the input spike train $\calI = \{\tau_1,\tau_2,\tau_3\}$ with spiking times $\tau_1 = 1, \tau_2 = 2, \tau_3 = 3 + h\ln(1+e^{-1/h}),$ and define $\phi_1 = e^{-(\tau_2-\tau_1)/h} = e^{-1/h}$ and $\phi_2 = e^{-(\tau_3-\tau_2)/h} = \frac{e^{-1/h}}{1+ e^{-1/h}}$. Let the output spike train be $\calO \coloneqq \{\tau_1\}$.

    To show the assertion denote by $H_1,\ldots, H_\ell$ the neurons in the hidden layer. Each neuron is characterized by the respective weight $w_i$. If $w_i >1$, then the spike train of $H_i$ equals $\calI$. Moreover, if $w_i \in ( 1/(1+\phi_1),1]$, then $H_i$ spikes only at $\tau_2$, since then 
    \begin{align*}
        P(\tau_1) = w_i\leq 1, \quad P(\tau_2) = w_i(1+ \phi_1) > 1, \quad P(\tau_3) = w_i\leq 1,
    \end{align*}
    where $P$ is the membrane potential of $H_i$. Finally, if $w_i \leq 1/(1+ \phi_1)$, then $H_i$ does not spike at all, since 
    \begin{align*}
        P(\tau_1) &= \left(w_i\right)_+\leq 1, \quad
        P(\tau_2) = (w_i)_+(1+ \phi_1) \leq \frac{1+\phi_1}{1+\phi_1} = 1,\\
        P(\tau_3) &= (w_i)_+ (1+ \phi_2 +  \phi_1\phi_2) = (w_i)_+ \left(1+ \phi_1\frac{1+\phi_1}{1+ \phi_1}\right) = (w_i)_+ \left( 1+ \phi_1\right) \leq 1.
    \end{align*}
    Hence, all hidden neurons either admit the spike train $\{\tau_1,\tau_2,\tau_3\}$ or $\{\tau_2\}$ or do not spike at all. In particular, for an output spike at time $\tau_1$, all neurons with spike train $\{\tau_1,\tau_2,\tau_3\}$ must have output weights that accumulate a value greater than one. However, this leads to a spike at time $\tau_3$, which cannot be compensated by the neurons with spike train $\{\tau_2\}$, as they do not contribute to the potential at time $\tau_3$. 
\end{proof}

We now derive a positive statement in the case that the arrival times of the input spike trains are sufficiently homogeneous. 

\begin{proposition}\label{prop:onelayer_positive}
    Let $h\in(0,\infty]$ and $\calI\in\TT$ be a spike train with $m:=|\calI|\in\NN$ spikes. For $m\ge2$ and $h\in (0,\infty)$, set
    \[
        \delta_{\min}:=\min_{k=2,\ldots,m}\big(\calI_{(k)}-\calI_{(k-1)}\big)>0,\qquad
        \delta_{\max}:=\max_{k=2,\ldots,m}\big(\calI_{(k)}-\calI_{(k-1)}\big)>0,\qquad
        q:=e^{-\delta_{\min}/h}.
    \]
    If $h\in (0,\infty)$ (for $h=\infty$, no assumption is made), assume that
    \begin{align}\label{eq:condition_homogenous_spikes}
    	 \delta_{\max}-\delta_{\min}
        <
        h\log\Big(\frac{1-q^{m}}{1-q^{m-1}}\Big).
    \end{align}
    Then, for every output spike train $\calO\dom\calI$, there exists a shallow SNN $f\in\SNN_h(2,(1,m,1),2m)$ such that $f(\calI)=\calO$.
    In particular, \eqref{eq:condition_homogenous_spikes} is met if $\delta_{\max}\le -h\log(1-1/m)$ or in case of constant gap  size $\calI_{(k+1)}-\calI_{(k)}\equiv\delta$ for all $k \in [m-1]$ for some fixed $\delta >0$. 
\end{proposition}

\begin{proof}
    First fix $h\in (0,\infty)$. For $m=1$, the construction involves choosing all weights as either $2$ or $0$ if $|\calO|=1$ or $0$, respectively. For  $m\ge2$, fix $k\in [m]$. Consider a hidden neuron $H$ which such that the connection from input to $H$ is weighted with $w\in \RR$.
    Define $A_1(w):=(w)_+$ and, for $\ell\ge2$,
    \[
        A_\ell(w):=(w)_++e^{-(\calI_{(\ell)}-\calI_{(\ell-1)})/h}A_{\ell-1}(w).
    \]
    This equals the potential of $H$ at time $\calI_{(\ell)}$ provided $H$ has not spiked before.
    Since $\delta_{\min}\le \calI_{(\ell)}-\calI_{(\ell-1)}\le \delta_{\max}$, we have $e^{-(\calI_{(\ell)}-\calI_{(\ell-1)})/h}\le q$ and $1-e^{-(\calI_{(\ell)}-\calI_{(\ell-1)})/h}\le 1-e^{-\delta_{\max}/h}$.
    Hence, $A_\ell(w)\le w\sum_{i=0}^{\ell-1}q^i$ for all $\ell \ge 1$ by a simple induction. Moreover, a straight-forward computation shows by definition of $q$ that Assumption \eqref{eq:condition_homogenous_spikes} is equivalent to  
    $1-e^{-\delta_{\max}/h} < \frac{1-q}{1-q^m}$ and to $(1-e^{-\delta_{\max}/h})\sum_{i=0}^{m-1}q^i<1$.    
    Thus, for $2\le j\le m$,
    \[
        A_j(w)-A_{j-1}(w)
        =w-\bigl(1-e^{-(\calI_{(j)}-\calI_{(j-1)})/h}\bigr)A_{j-1}(w)
        \ge w-(1-e^{-\delta_{\max}/h})\,w\sum_{i=0}^{m-1}q^i
        >0.
    \]
    Therefore, $A_1(w)<\cdots<A_m(w)$. Further, for each fixed $j\in [m]$, the map $w\mapsto A_j(w)$ is continuous, strictly increasing, and satisfies $\lim_{w\searrow 0} A_j(w)= 0$ and $\lim_{w\nearrow \infty} A_j(w)= \infty$. Hence, by the intermediate value theorem there exists a weight $w_k'>0$ such that $A_k(w_k')=1$ while $A_j(w_k')< 1$ for all $j<k$. 
    Hence, by continuity $w\mapsto A_j(w)$ and strict monotonicity, there exists $w_k>w_k'$, with $A_k(w_k)>1$ while $A_j(w_k)<1$ for all $j<k$. For this choice of $w_k$, the neuron $H$ spikes at $\calI_{(k)}$ for the first time. Repeating this construction for $k=1,\ldots,m$ yields $m$ hidden neurons $H_1,\ldots,H_m$ such that $H_k$ first spikes at $\calI_{(k)}$. The same conclusion applies for $h = \infty$ by choosing $w_k \in (\frac{1}{k}, \frac{1}{k-1}]$ without using any assumption on the gap size.
    
    To obtain the output spike train $\calO\dom\calI$, we use the following construction. If $\calI_{(k)}\in\calO$, we connect $H_k$ to the output neuron with weight $2^k$ and otherwise, if $\calI_{(k)}\not\in\calO$, we consider weight $-2^k$. We now prove by induction that this produces $\calO$. For the base case, i.e.,  $k = 1$, note that weight $2$ between neuron $H_1$ and the output neuron triggers an output spike while weight $-2$ results in potential equal to $0$ and does not trigger a spike. For the induction step, assume that the output neuron spikes on $(0, \calI_{(k)})$ exactly at each time points  $\calI_{(j)}$ for $j\in [k-1]$ with  $\calI_{(j)}\in \calO$. If $\calI_{(k)} \in \calO$, then, at time $\calI_{(k)}$, the output neuron receives weights at least $2^k + \sum_{j=1}^{k-1} -2^j = 2^k -2^k + 2 = 2$, resulting in a spike of the output neuron at $\calI_{(k)}$. If $\calI_{(k)} \not\in \calO$, then the output neuron has potential $P(\calI_{(k)})\leq (-2^k + \sum_{j=1}^{k-1} 2^j)_+ = (-2^k + 2^k - 2)_+ = (-2)_+ = 0$ and no output spike is released. This proves the claim by mathematical induction. 

    Finally, if $\delta_{\max}\le -h\log(1-1/m)$, then $(1-e^{-\delta_{\max}/h})m < 1$. Hence, since for any $q\in(0,1)$ it holds that $\sum_{i=0}^{m-1}q^i < m$, we get $(1-e^{-\delta_{\max}/h})\sum_{i=0}^{m-1}q^i < 1$ and thus \eqref{eq:condition_homogenous_spikes} is satisfied. Moreover, if $\calI_{(k+1)}-\calI_{(k)}\equiv\delta$ for all $k\in [m-1]$ and some fixed $\delta>0$, then the assumption holds because the right-hand side of \eqref{eq:condition_homogenous_spikes} is positive.
\end{proof}

\subsection{On the Necessity of Negative Weights in the Universal Representation Theorem}

It is of genuine interest to wonder whether universal representation can hold with only excitatory neurons, that is, only non-negative network weights. We provide a negative answer in Lemma \ref{lem:two_negative_weights}, showing that at least two negative weights are necessary. As a second result, \Cref{thm:log_T_neg_weights} implies that for one input neuron and $T$ output neurons, universal representation can be achieved with at most $ 2\lceil \log(T)\rceil$ negative weights and $O(\log^2(m))$ layer and $O(m\log^2(m))$ neuron and weights. 

\begin{lemma}[Necessity of two negative weights]\label{lem:two_negative_weights}
    Let $h\in (0, \infty)$ and consider the spike train $\calI$ (consisting of $3$ spikes) and output spike train $\calO= \{\calI_{(1)}\}$ from Proposition~\ref{prop:onelayer_counter}. Then, any SNN with memory $h$ mapping $\calI$ to $\calO$ has at least two negative weights.
\end{lemma}

\begin{proof}
We begin by proving the following claim: In  a single-input feedforward SNN with depth $L \in \NN$, non-negative weights, and input spike train $\calI$, every neuron $N$ has spike train $\calS_N\in\{\emptyset,\{\calI_{(2)}\},\{\calI_{(1)},\calI_{(2)},\calI_{(3)}\}\}$. We prove the assertion by mathematical induction on the SNN layers $\ell \in [L]$. The base case for $\ell = 1$ was already shown in the proof of Proposition~\ref{prop:onelayer_counter}. Now assume the claim holds for all layers up to $\ell-1\ge 1$, and let $N$ be a neuron in layer $\ell$ connected to $n\in \NN$ neurons in layer $\ell-1$ whose incoming weights $v_j$ are all non-negative. Define
    \[
        A := \{j \in [n] \colon \calS_j = \{\calI_{(1)},\calI_{(2)},\calI_{(3)}\}\},\quad
        B := \{j \in [n] \colon \calS_j = \{\calI_{(2)}\}\},\quad
        a := \sum_{j\in A} v_j,\quad
        b := \sum_{j\in B} v_j.
    \]
    Let $P$ be the membrane potential of $N$. Then, $P(\calI_{(1)}) = a$, so $\calI_{(1)} \in \calS_N$ if and only if $a > 1$. Further, 
    upon defining $\phi_1 = e^{-(\calI_{(2)}-\calI_{(1)})/h}$ and $\phi_2 = e^{-(\calI_{(3)}-\calI_{(2)})/h}$, it holds 
    $P(\calI_{(2)}) = \big( \mathds{1}(P(\calI_{(1)})\le 1)\, P(\calI_{(1)})\phi_1 + a + b \big)_+$ and $P(\calI_{(3)}) = \big( \mathds{1}(P(\calI_{(2)})\le 1)\, P(\calI_{(2)})\phi_2 + a \big)_+$.   We now distinguish two cases.

    \emph{Case 1:} If $a\leq 1$, it follows that $P(\calI_{(1)})=a\le 1$, so no spike is triggered at time $\calI_{(1)}$ and \(P(\calI_{(2)}) = a(1+\phi_1) + b\). If $P(\calI_{(2)})\le 1$, then it follows from $1 + \phi_2 + \phi_1 \phi_2 = 1 + \phi_1$ (see Proposition~\ref{prop:onelayer_counter}) and the following computation that $\calS_N=\emptyset$ and 
    \begin{align*}
        P(\calI_{(3)}) &= P(\calI_{(2)})\phi_2 + a = a(1+\phi_1)\phi_2 + b\phi_2 + a = a(1+\phi_1) + b\phi_2\\
        &\le a(1+\phi_1)+b = P(\calI_{(2)})\le 1.
    \end{align*}
   Else, if $P(\calI_{(2)})>1$, then the neuron $N$ spikes at $\calI_{(2)}$. As a result, the potential resets, hence $P(\calI_{(3)})=a\le 1$ and it follows that $\calS_N=\{\calI_{(2)}\}$.

    \emph{Case 2:} If $a>1$, it follows that    
    $P(\calI_{(1)})=a>1$, as well as $P(\calI_{(2)})=(a+b)_+ = a+b>1$, and $P(\calI_{(3)})=a>1$, asserting that $\calS_N=\{\calI_{(1)},\calI_{(2)},\calI_{(3)}\}$. 
    
    This proves the preliminary claim. In particular, no SNN with only non-negative weights can map $\calI$ to $\calO=\{\calI_{(1)}\}$, so any such SNN must have at least one negative weight.

    Now assume, for a contradiction, that there existed an SNN with input $\calI$ and output $\calO$ and exactly one negative weight. Let this weight be the weight from neuron $N_1$ to neuron $N_2$. Since the above claim remains valid for general $a\in\RR$ --  note that we only used $b\ge 0$ -- the same three spike patterns $\{\emptyset,\{\calI_{(2)}\},\{\calI_{(1)},\calI_{(2)},\calI_{(3)}\}\}$ occur for $N_2$ if $N_1$ spikes at all three times or never spikes. Hence, we may assume that $N_1$ has spike train $\{\calI_{(2)}\}$.  Since the output neuron spikes at time $\calI_{(1)}$, there exists a directed path from the input neuron to the output neuron such that at each neuron on this path the total incoming weight from its presynaptic neurons that spike at time $\calI_{(1)}$ exceeds $1$. As $N_1$ only spikes at time $\calI_{(2)}$ it follows that along this path no neuron receives a unique negative contribution at time $\calI_{(1)}$ or $\calI_{(3)}$, and every neuron on the path spikes at time $\calI_{(1)}$ and $\calI_{(3)}$. Hence, the output neuron also spikes at time $\calI_{(3)}$, a contradiction. Concluding,  at least two negative weights are required to realize $\calO$.
\end{proof}

In the following, we prove that we can always find a network with only $2\lceil \log(m) \rceil$ negative weights that takes $m$ input spikes and can represent simultaneously all possible outputs with a single spike.

\begin{theorem}\label{thm:log_T_neg_weights}
    For every $h\in (0, \infty]$ and $m\in \NN$,
    there exists a feedforward SNN $f\in \SNN_{h}(L, \bp, s)$ with one input neuron, $m$ output neurons, memory coefficient $h\in (0,\infty]$, $L = O(1+\log^2(m))$ and  $\|\bp\|_1 +s = O(m(1+\log^2(m)))$ that uses at most $2\lceil \log(m) \rceil$ negative weights and satisfies the following: For any input spike train $\calI \in \TT$ with $|\calI| = m$, the $m$ output spike trains are $\{\calI_{(1)}\}, \{\calI_{(2)}\}, \ldots, \{\calI_{(m)}\}$. 
\end{theorem}

\Cref{thm:log_T_neg_weights} implies that 
by attaching another layer of a single output neuron to the above SNN, with weight $0$ or $2$, we can represent any output spike train $\calO \dom \calI$ with at most $2\lceil \log(m) \rceil$ negative weights. This construction shows that it is possible to use fewer negative weights than in \Cref{thm:universal_representation_single_input}, while the overall network has depth $O(\log^2(m))$ and $O(m\log^2(m))$ neurons and weights. The proof relies on the following lemma which provides a blueprint for the construction of the SNN in \Cref{thm:log_T_neg_weights}. In particular, by composing suitable functions it utilizes negative weights remarkably efficient. 

\begin{lemma}\label{lem:set_operations}
    Let $m\in \NN$ and recall the \texttt{SKIP}, \texttt{EVEN}, \texttt{OR}, and \texttt{MINUS} function from \Cref{sec:examples_functions}. Then, there exists a compositional function $\bF \colon \TT\to \TT^m$, $\bF = \bF_q \circ \bF_{q-1} \circ \cdots \circ \bF_1$ 
    with $\bF_i\colon \TT^{k_i} \to \TT^{k_{i+1}}$ for some $k_i \in \NN$, where $\bF_i = (\bF_{i, j})_{j=1}^{k_{i+1}}$ with $\bF_{i,j} \in \{\textup{\texttt{SKIP}}, \textup{\texttt{EVEN}}, \textup{\texttt{OR}}, \textup{\texttt{MINUS}}\}$. The function $\bF$ outputs all singleton spike trains of an input $\calI\in \TT$ with $|\calI| = m$, i.e., 
    \begin{align*}
        \bF(\calI) = (\{\calI_{(1)}\}, \{\calI_{(2)}\}, \ldots, \{\calI_{(m)}\}).
    \end{align*}
    Moreover, the composition has at most $2\lceil \log(m) \rceil$ \texttt{MINUS} functions, $q = O(\log^2(m))$ compositions, and $\sum_{i=1}^q k_{i+1}=O(m\log^2(m))$.
\end{lemma}

    \begin{figure}[t]
    	\centering
\includegraphics[width=\textwidth]{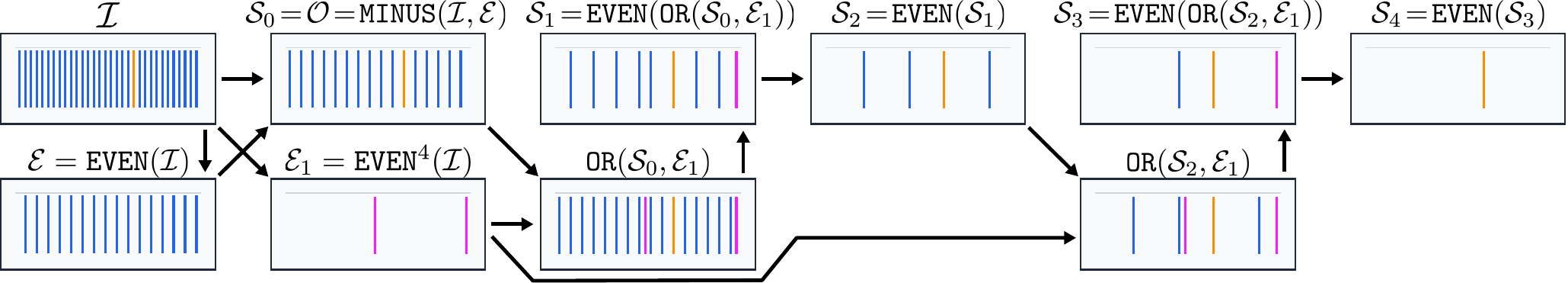}
\caption{\textbf{Procedure from proof of \Cref{lem:set_operations} with respective spike trains to produce $\{\calI_{(k)}\}$ for $k = 21$ and $m = 32$.} The $k$-th spike $\calI_{(k)}$ is depicted in orange, the spikes of $\calE_1$ are depicted in magenta, while all other spikes are shown in blue. The arrows describe how the spike trains are defined in terms of each other. Only a single \texttt{MINUS}-function, jointly with multiple \texttt{SKIP}, \texttt{EVEN}, \texttt{OR} functions, is used to produce spike train $\calS_{4}=\{\calI_{(k)}\}$. 
}\label{fig:neg_functional_representation_architecture}
    \end{figure}

\begin{proof}
    The proof is constructive. Let $\ell := \lceil \log(m) \rceil$. In a first step, we describe how to obtain each of the singletons $\{\calI_{(k)}\}$ for the second-half-indices $k\in \{2^{\ell -1}+1, \ldots, m\}$ using only one \texttt{MINUS} operation. The second step consists of using another \texttt{MINUS} operation to obtain the first-half-indices via $\{\calI_{(1)}, \ldots, \calI_{(2^{\ell -1})}\}=\texttt{MINUS}(\calI, \bigcup_{k=2^{\ell -1}+1}^m \calI_{(k)}).$ After that, we apply the procedure from the first step to the first-half-indices. Repeating the procedure $\lceil \log(m) \rceil$ times gives all singleton spikes while only using $2\lceil \log(m)\rceil$ \texttt{MINUS} functions in~total.
    
    Fix an index $k\in \{2^{\ell-1}+1, \ldots, m\}.$ We first construct the spike train $\{\calI_{(k)}\}$ if $k$ is odd, an exemplary visual illustration of the procedure is detailed in \Cref{fig:neg_functional_representation_architecture}.  
    Apply \texttt{EVEN} to $\calI$ once to obtain the even indices $\calE=\{\calI_{(2)}, \calI_{(4)}, \ldots\}$ and set $\calO:=\texttt{MINUS}(\calI, \calE)=\{\calI_{(1)},\calI_{(3)},\ldots\}$ which are the odd indices (here we use the \texttt{MINUS} operation). By applying the \texttt{EVEN} function $\ell-1$ times to the input spike train $\calI$ we obtain the spike train
    \[
      \calE_1:=\texttt{EVEN}^{\,\ell-1}(\calI)=
      \begin{cases}
        \{\calI_{(2^{\ell-1})}\}, & m<2^{\ell},\\
        \{\calI_{(2^{\ell-1})},\calI_{(2^{\ell})}\}, & m = 2^{\ell}.
      \end{cases}
    \]
    The spike trains $\calE_1$ and $\calO$ are forwarded in the function composition via $\ell-1$ \texttt{SKIP} functions which makes them accessible at subsequent layers. Starting from $\calO$, we now construct a sequence of spike trains containing $\calI_{(k)}.$ The idea is that if $k$ is at a position with even index in the ordered spike train, applying the \texttt{EVEN} function halves the set of indices and keeps $k$. If $k$ is at a position with odd index, we first use the \texttt{OR} function with $\calE_1$ to move $k$ to a position with even index and then apply the \texttt{EVEN} function. This procedure will eventually lead to the singleton $k$. To make this precise, let $\text{rank}_\calS(x)$ be the position of $x\in \calS$ in the increasing list of a finite ordered set $\calS\dom\calI$. Define $\calS_0:=\calO$ and for $t=0,1,2,\ldots, \ell-2$,
    \[
      \calS_{t+1}:=
      \begin{cases}
        \texttt{EVEN}(\calS_t), & \text{if }\text{rank}_{\calS_t}(\calI_{(k)})\text{ is even},\\
        \texttt{EVEN}(\texttt{OR}(\calS_t,\calE_1)), & \text{if }\text{rank}_{\calS_t}(\calI_{(k)})\text{ is odd},
      \end{cases}
    \]
     where $\text{rank}_{\calS_t}(\calI_{(k)})$  is well-defined as we see below.    
    Recall that we consider here the case that $k$ is odd. By induction on $t\in \{0,1,\ldots, \ell-2\}$, we now prove that
    \begin{align*}
        &\calI_{(k)}\in \calS_t, \quad \calI_{(2^{\ell-1})}\notin \calS_t \quad\text{ and } \quad L_t:=|\{x\in \calS_t:\ x \le \calI_{(2^{\ell-1})}\}| \text{ is a power of two.}
    \end{align*}
    In particular, this implies that $\text{rank}_{\texttt{OR}(\calS_t, \calE_1)}(\calI_{(2^{\ell-1})})$ is odd for $L_t \ge 2$. For $t=0$, we have $\calI_{(2^{\ell-1})}\notin \calS_0$, $L_0=2^{\ell-2}$, and $\calI_{(k)}\in \calS_0$, which shows the base case. For the induction step, if $\text{rank}_{\calS_t}(\calI_{(k)})$ is even, then \texttt{EVEN} keeps exactly the even-ranked elements, so $\calI_{(k)}\in \calS_{t+1}$ and $L_{t+1}=L_t/2$. If $\text{rank}_{\calS_t}(\calI_{(k)})$ is odd, then exactly one new element smaller than $\calI_{(k)}$ is added to $\texttt{OR}(\calS_t, \calE_1)$, namely $\calI_{(2^{\ell-1})}$, so $\text{rank}_{\texttt{OR}(\calS_t, \calE_1)}(\calI_{(k)})$ is even and applying the \texttt{EVEN} function keeps $\calI_{(k)}$. Moreover, $\calI_{(2^{\ell-1})}$ has the odd rank $L_t+1$ in $\texttt{OR}(\calS_t, \calE_1)$ and is therefore removed by \texttt{EVEN}. Hence, $L_{t+1}=\lfloor(L_t+1)/2\rfloor=L_t/2$ since $L_{t}\ge 2$ for all $t\in \{0,\ldots, \ell-3\}$ and because $L_t$ is a power of two by induction hypothesis. Thus, in each step $\calI_{(k)}$ remains and $L_t$ halves. %
    Since also $|\calS_{t+1}|\le\lfloor(|\calS_t|+2)/2\rfloor$, it follows that after the $\ell - 2$ iterations we reach $2\le |\calS_t|\le 3$ and $\text{rank}_{\calS_t}(\calI_{(k)}) = 2$ or $\text{rank}_{\calS_t}(\calI_{(k)}) = 3$. If $\text{rank}_{\calS_t}(\calI_{(k)})=2$, then $\texttt{EVEN}(\calS_t)=\{\calI_{(k)}\}$. Otherwise, set $\calS_{t+1}:=\texttt{EVEN}(\texttt{EVEN}(\texttt{OR}(\calS_t,\calE_1)))$ which gives the correct element since $\text{rank}_{\texttt{OR}(\calS_t,\calE_1)}(\calI_{(k)}) = 4$. This ends the case for $k$ being odd.  
    
    Now assume that $k$ is even. The argument is similar, but we have to pay attention to the fact that $\calI_{(2^{\ell-1})}$ might remain in the set after an $\texttt{EVEN}$ operation which leads to an additional invariant in the induction. If $k= m = 2^\ell$, then repeating the $\texttt{EVEN}$ operation, we get $\{\calI_{(m)}\}$. Otherwise, we first apply $(\ell-2)$-times the \texttt{EVEN} function to the odd spikes $\calO,$
    \[
      \calO_1 := \texttt{EVEN}^{\ell-2}(\calO) =
      \begin{cases}
        \{\calI_{(2^{\ell-1}-1)}\}, & m < 2^\ell-1,\\[2mm]
        \{\calI_{(2^{\ell-1}-1)},\calI_{(2^\ell-1)}\}, & m \ge 2^\ell-1.
      \end{cases}
    \]
    Set $\calT_0 := \calE$ and for $t=0,1,2,\ldots$ define
    \[
      \calT_{t+1} :=
      \begin{cases}
        \texttt{EVEN}(\calT_t), &
          \text{if }\text{rank}_{\calT_t}(\calI_{(k)})\text{ is even},\\[1mm]
        \texttt{EVEN}\big(\texttt{OR}(\calT_t,\calO_1)\big), &
          \text{if }\text{rank}_{\calT_t}(\calI_{(k)})\text{ is odd and }
          \calI_{(2^{\ell-1})}\in \calT_t,\\[1mm]
        \texttt{EVEN}\big(\texttt{OR}(\calT_t,\calE_1)\big), &
          \text{if }\text{rank}_{\calT_t}(\calI_{(k)})\text{ is odd and }
          \calI_{(2^{\ell-1})}\notin \calT_t.
      \end{cases}
    \]
    We prove by induction on $t\in \{0,1,\ldots, \ell-2\}$ that $\text{(a) } \calI_{(k)}\in \calT_t$,
    \begin{align*}
      \text{(b) } L_t:=|\{x\in \calT_t: x \le \calI_{(2^{\ell-1})}\}|  \text{ is a power of two},\quad \text{(c) } |\{ \calI_{(2^{\ell-1}-1)},\calI_{(2^{\ell-1})}\}\cap \calT_t| \le 1.
    \end{align*}
    For $t=0$, $\calT_0=\calE$ contains $\calI_{(k)}$, $L_0=2^{\ell-2}$ is a power of two, and (c) holds because $\calI_{(2^{\ell-1}-1)}$ is not in $\calE$ showing the base case. Assume (a), (b) and (c) hold for some $t\in \NN$ for the induction step.
    
    \emph{Case 1:} $\text{rank}_{\calT_t}(\calI_{(k)})$ is even. Then \texttt{EVEN} keeps exactly the even-ranked elements, so $\calI_{(k)}\in \calT_{t+1}$ and (a) holds. Further, $L_{t+1}=L_t/2$ is again a power of two and since we had at most one of $\calI_{(2^{\ell-1}-1)},\calI_{(2^{\ell-1})}$ in $\calT_t$, the same is true in $\calT_{t+1}$.
    
    \emph{Case 2:} $\text{rank}_{\calT_t}(\calI_{(k)})$ is odd and $\calI_{(2^{\ell-1})}\in \calT_t$. By (c) this implies $\calI_{(2^{\ell-1}-1)}\notin \calT_t$. In $\texttt{OR}(\calT_t,\calO_1)$ exactly one new element $\le \calI_{(2^{\ell-1})}$ is added, namely $\calI_{(2^{\ell-1}-1)}$. Since $\calI_{(2^{\ell-1}-1)}<\calI_{(2^{\ell-1})}<\calI_{(k)}$, the rank of $\calI_{(k)}$ in $\texttt{OR}(\calT_t,\calO_1)$ is even, therefore applying the \texttt{EVEN} function keeps $\calI_{(k)}$ and (a) holds for $t+1$. Exactly one of $\calI_{(2^{\ell-1}-1)}$ and $\calI_{(2^{\ell-1})}$ has even rank in $\texttt{OR}(\calT_t,\calO_1)$, so exactly one is left after applying the \texttt{EVEN} function, proving (c) for $t+1$. Moreover, $L_{t+1}=\left\lfloor \frac{L_t+1}{2}\right\rfloor = \frac{L_t}{2}$ (since $L_t\ge 2$ for $t<\ell-2$), because $L_t$ is a power of two.

    \emph{Case 3:} $\text{rank}_{\calT_t}(\calI_{(k)})$ is odd and $\calI_{(2^{\ell-1})}\notin \calT_t$. In $\texttt{OR}(\calT_t,\calE_1)$ exactly one new element $\le \calI_{(2^{\ell-1})}$ is added, namely $\calI_{(2^{\ell-1})}$. Since $\calI_{(2^{\ell-1})}\le \calI_{(k)}$, the rank of $\calI_{(k)}$ in $\texttt{OR}(\calT_t,\calE_1)$ is even, so \texttt{EVEN} keeps $\calI_{(k)}$ and (a) holds. As before, one of the two points $\calI_{(2^{\ell-1}-1)}$ and $\calI_{(2^{\ell-1})}$ in $\texttt{OR}(\calT_t,\calE_1)$ is removed by \texttt{EVEN}, and we still have at most one such point in $\calT_{t+1}$, proving (c). Finally, again $L_{t+1}=\lfloor(L_t+1)/2\rfloor=L_t/2$, so (b) holds. By the same argument as in the odd case, we reach $\{\calI_{(k)}\}$ in $O(\log(m))$ step.
    
    Therefore, besides the single initial \texttt{MINUS} to form $\calO$, every $\{\calI_{(k)}\}$ with $k\in\{2^{\ell-1}+1,\ldots,m\}$ is obtained using only \texttt{EVEN} and \texttt{OR} functions, as well \texttt{SKIP} functions to use the inputs from past layers from. 
    
    Next, to iterate the above procedure, set $\calA_0 := \calI$ and define
    \[
      \calU_0:=\texttt{OR}\big(\{\calI_{(2^{\ell-1}+1)}\},\ldots,\{\calI_{(m)}\}\big),\qquad
      \calA_1 := \texttt{MINUS}(\calA_0, \calU_0).
    \]
    On $\calA_1$ we repeat the above procedure to generate all $\{\calI_{(k)}\}$ for $k\in \{2^{\ell-2}+1, \ldots, 2^{\ell-1}\}$ using only one further \texttt{MINUS} (to extract the spikes at odd positions). Iterating
    \[
      \calA_{i+1} := \texttt{MINUS}\Big(\calA_i, \texttt{OR}\big(\{\calI_{(2^{\ell-i-2}+1)}\},\ldots,\{\calI_{(2^{\ell-i-1})}\}\big)\Big),
    \]
    we obtain all singletons $\{\calI_{(k)}\}$ using at most $2\ell= 2\lceil \log(m) \rceil$ applications of \texttt{MINUS} because in each step, we need one to create odd indices and one to create $\calA_{i+1}$.
    
    The above full procedure takes $O(\log m)$ iterations, and each iterate uses for each individual singleton $O(\log m)$ component functions, so the total number of compositions is $q = O(\log^2(m))$. In the $i$-th iteration, we create at most $m$ singletons, and they need to be forwarded via \texttt{SKIP} operations through at most $\log(m)$ blocks (each block containing the composition of $O(\log(m))$ functions). Hence, the total number of basic operations (including \texttt{SKIP} operations), that is the total number of component functions $\bF_{i,j}$, is bounded by $O(m\log^2(m))$.
\end{proof}

\begin{proof}[Proof of Theorem \ref{thm:log_T_neg_weights}]
    All the operations \texttt{EVEN}, \texttt{OR}, \texttt{SKIP} and \texttt{MINUS} can be implemented by SNNs of bounded size (see Section \ref{sec:examples_functions} or \Cref{app:proofs_section4}) with a single layer. Noting that each use of \texttt{MINUS} requires exactly one negative weight and that all other operations use only non-negative weights, the result follows directly from Lemma \ref{lem:set_operations}. Moreover, in the SNN implementation, $q$ corresponds to the number of layers and the total number of component functions $\bF_{i,j}$ is of the same order as the total number of neurons and non-zero weights, so the claimed bounds follow.
\end{proof}

\subsection{Structural constraint: Limited out-degree}

In this section, we work under the constraint that each neuron can be connected to at most $q\in \NN$ other neurons. Under this assumption, we get the following:

\begin{proposition}[Universal representation with bounded out-degree]\label{thm:universal_representation_spike_trains_bounded_out_degree} 
    Let $m\in \NN, r \in [m],$ $h\in (0,\infty]$, and $q\in \NN$, $q\ge 4$ be the maximum out-degree. Then, there exists an architecture $(L,\bp)$ with $L = O(1+\log(m))$, $\|\bp\|_1 = O(\max\{m/q,\sqrt{m}\})$ and a positive integer $s = O(\max\{m/q,\sqrt{m}\} + r)$ such that for any $F \in \FF_{\mathrm{fin}}(m,r)$ there exists a feedforward SNN $f \in \SNN_{h}(L,\bp,s)$ with memory $h$ and maximum out-degree $q$ such that
    \begin{align}\label{eq:SingleInputRepresentation}
         f(\calI)=F(\calI), \quad \text{for all } \calI \in \TT.
    \end{align}
    If $m = 4^{k}$ for $k\in \NN_0$, and possibly $h = 0$ if $m = 1$, the claim also holds for $F\in \FF_{\mathrm{per}}(m,r)$.
\end{proposition}

\begin{proof}
    The proof modifies the construction in the proof of \Cref{thm:universal_representation_single_input} that is visualized in \Cref{fig:functional_representation_architecture}.  We rewire the network so that approximately $m/q$ neurons enter the representation block (see \Cref{fig:functional_representation_architecture}), each firing at distinct points only once for finite functions and once in a period $2^{\lfloor \log(q)\rfloor}\leq q$ for periodic functions. The second layer of the representation block consists of $2^{\lfloor \log(q)\rfloor}$ neurons which ensures the out-degree constraint. The connection weights for the modified representation block are determined by $F$, the neurons in the second layer can only spike once within a period $2^{\lfloor \log(q)\rfloor}$, and are subsequently delayed by a $2^{\lfloor \log(q)\rfloor}$-periodic \texttt{CLOCK} module, so the representation of $F$ is ensured.  
    
    Let $\overline m \coloneqq 4^{\lceil \log_4(m) \rceil}\ge m$ as in the proof of \Cref{thm:universal_representation_single_input}. Assume first that $q$ is a power of $2$ and $q\le 2\sqrt{\overline m}$ (if $q>2\sqrt{m}$ we use the architecture from \Cref{thm:universal_representation_single_input}, which already satisfies the out-degree constraint).
    Starting from the input spike train $\calI$, we replace the coarse/fine splitting block by a deeper version. We construct $\overline m/q$ neurons $v_{1,1},\ldots, v_{1,\overline m/q}$ by applying the \texttt{ODD/EVEN} module $\log(q)$ times (instead of $\frac{1}{2}\log(\overline m)$ times) while always keeping the odd output, and then another $\log(\overline m/q)$ layers of \texttt{ODD/EVEN} functions where we keep both outputs. Denoting the spike train of $v_{1,i}$ by $\calV_{1,i}$, this yields %
    $$
      \calV_{1,i} = \big\{\calI_{(\overline m k + (i-1)q + 1)} : k\in \NN_0\big\},
      \quad i\in\big[\overline m/q\big].
    $$

    In parallel, we construct the \texttt{CLOCK}$_{q}$ module consisting of $q$ neurons $u_1,\ldots,u_q$ by keeping a copy of the input for the first $\log(\overline m/q)$ layers and then applying \texttt{ODD/EVEN} blocks $\log(q)$ times, each time keeping both outputs. Writing $\calU_i$ for the spike train of $u_i$, we obtain
    $$
      \calU_i = \big\{\calI_{(\overline m k + i)} : k\in \NN_0\big\},
      \quad i\in[q].
    $$
    This is the coarse- and fine splitting step, see Figure \ref{fig:functional_representation_architecture}. From this point on, the finite/periodic modification and the representation block are constructed analogously to the proof of \Cref{thm:universal_representation_single_input}, with the only difference that the last layer in the representation block now has $q$ neurons (so that the representation block has $O(m)$ weights and each neuron has out-degree at most $q$). The remaining parts of the argument are identical to those in the proof of \Cref{thm:universal_representation_single_input}. This construction creates the output spike train $F(\calI)$ and uses $O(\frac{\overline m}{q}+q) = O(m/q)$ neurons and well as
    $O(\frac{\overline m}{q}+q + r) = O(m/q +r)$
non-zero weights since $q \leq \sqrt{\overline m}$. Note that we impose $q\geq 4$ to ensure that the \texttt{ODD/EVEN} and \texttt{CEIL}$_1$ function can be implemented with maximum out-degree $4$, see the proofs of \Cref{lem:periodic_finite_functions} and \Cref{lem:ceil_snn}. 
    
    If $q$ is not a power of $2$, we can nonetheless apply the construction above for $q' \coloneqq 2^{\lfloor \log(q) \rfloor}\leq q$. The maximum out-degree of the resulting SNN architecture is bounded by $q'\leq q$ and fulfills universal representation of single-input functions. Moreover, the number of neurons is bounded by $O(\max(\frac{m}{q'}, \sqrt{m}))= O(\max(\frac{m}{q}, \sqrt{m}))$ while the number of non-zero weights is bounded by $O(\max(\frac{m}{q'}, \sqrt{m}) + r) = O(\max(\frac{m}{q}, \sqrt{m}) + r)$.
\end{proof}

\section{Explicit Constructions for Efficiently Representable Functions}\label{app:proofs_section4}

In this section, we provide explicit constructions of SNNs for the functions from  \Cref{sec:examples_functions}. %

\subsection{Constructions to Finite or Periodic Single-Input Functions}

We first construct SNNs for the \texttt{SKIP}, \texttt{ODD/EVEN}, \texttt{CLOCK}$_m$ functions. Later in this subsection, we provide a construction for \texttt{SPIKE}$_m$. 

\begin{lemma}[\texttt{SKIP}, \texttt{ODD/EVEN}, \texttt{CLOCK}$_m$ SNNs]\label{lem:periodic_finite_functions}
    Let $h\in (0,\infty]$. Then the following holds. 
    \begin{enumerate}
        \item For any $L\in \NN$, there exists $f \in \SNN_h(L, (1,1,\ldots,1,1), L)$ such that $f(\calI) = \texttt{SKIP}(\calI)$ for any input spike train $\calI \in \mathbb{T}$.
        \item There exists $\bbf \in \SNN_h(2, (1,2,2), 2)$ such that $\bbf(\calI) = (\texttt{ODD}(\calI), \texttt{EVEN}(\calI))\in \TT^2$ for any input spike train $\calI \in \mathbb{T}$.
        \item For any $m = 2^q$ with $q\in \NN_0$ there exists $\bbf \in \SNN_h(2q, (1,2,2,4,4,\ldots,2^q,2^q), 5\cdot 2^q)$ such that $\bbf(\calI) = \texttt{CLOCK}_m(\calI)\in \TT^{m}$ for any input spike train $\calI \in \mathbb{T}$.
    \end{enumerate}
\end{lemma}

\begin{proof}
    
    To verify Assertion $(i)$ consider an SNN with one input neuron, $L-1$ hidden neurons, and one output neuron, each connected in a chain. A construction for $L =1$ is shown in \Cref{fig_skip}$(a)$. All weights are assigned the value $2.$ This ensures that any input spike is transmitted to the output neuron. No membrane potential is built up in any neuron.

    For Assertion $(ii)$, a figure of the SNN is shown in \Cref{fig_skip}$(b)$. We first attach two neurons $N_1$ and $N_2$ to the input neuron with respective weights $2$ and $1$. At every odd spike $\calI_{(2k-1)}$ of the input, the membrane potential of neuron $N_2$ reaches value exactly $1$ and thus does not spike. At the subsequent even spike $\calI_{(2k)}$ of the input, the membrane potential equals $1 + \exp\left(-(\calI_{(2k)} - \calI_{(2k-1)})/h\right) > 1$ and thus neuron $N_2$ spikes. Hence, $N_2$ spikes every second spikes. We then attach a neuron with weight $2$ from $N_1$ and weight $-2$ from $N_2$ to obtain the \texttt{ODD} output (this results in a spike train where the even spikes are subtracted from all input spikes) and a \texttt{SKIP} connection from $N_2$ to obtain the \texttt{EVEN} output. The \texttt{ODD} output spikes at all odd input spikes, $\calI_{(2k-1)}$ for $k \in \NN$, and the \texttt{EVEN} output spikes at even input spikes, $\calI_{(2k)}$ for $k \in \NN$.
    
   Assertion $(iii)$ is proved by induction on the exponent $q\in \NN_0.$ The cases $q=0$ and $q=1$ correspond to the \texttt{SKIP} function and the \texttt{ODD/EVEN} function, both of which have already been established.
    For the induction step, we assume that the claim holds for some $q\in \mathbb{N}$ with $m = 2^q$, and we will show that it then also holds for $q+1$. 
    By the induction hypothesis, there exists an SNN $\bbf_q \in 
    \SNN_h(2q,(1,2,2,4,4,\ldots,2^q,2^q),5\cdot 2^q)$ such that the claim holds. 
    Given an input spike train $\calI$, the output neurons are denoted by $O^m_1, \dots, O^m_{2^q}$ and the corresponding output    
    spike trains are denoted by $\calO_1^m, \dots, \calO_{2^q}^m\subseteq (0, \infty)$, respectively, with 
    \begin{align}\label{eq:clock_induction_hypothesis}
        \calO^m_i := \{\calI_{(r)} : r \equiv i \mod 2^q, r\in \NN_0, r \leq |\calI|\}.
    \end{align}
    We now construct another SNN $\bg_{q+1} \in \SNN_h(2,(2^q, 2^{q+1}, 2^{q+1}), 5 \cdot 2^q)$ with $m = 2^q$ input neurons, corresponding input spike trains $\calO_1^m, \dots, \calO_{2^q}^m$, and $2m = 2^{q+1}$ output neurons $O^{2m}_1, \dots, O^{2m}_{2^{q+1}}$ with spike trains  $\calO_1^{2m}, \dots, \calO_{2^{q+1}}^{2m}$ defined as in \eqref{eq:clock_induction_hypothesis}.

    To this end, we apply an \texttt{ODD/EVEN} SNN to each of the $2^q$ inputs, which results in $2 \cdot 2^{q}$ neurons in the first and the second layer, and $5 \cdot 2^{q}$ non-zero weights. 
    Inserting a spike train $\calO_j^m$ into an \texttt{ODD/EVEN} SNN we obtain two output neurons with spike trains
    \begin{align*}
        \{\calI_{(r)} : r \equiv j \mod 2^{q+1}, r\in \NN_0, r \leq |\calI|\}&=\calO^{2m}_{j},\\
        \{\calI_{(r)} : r \equiv j + 2^q \mod 2^{q+1}, r\in \NN_0, r \leq |\calI|\}&=\calO^{2m}_{j+2^q}.
    \end{align*}
    Thus, $\bg_{q+1}:(\calO_1^m, \dots, \calO_{2^q}^m)\to (\calO_1^{2m}, \dots, \calO_{2^{q+1}}^{2m}).$ By \eqref{property_comp}, $\bbf_{m+1} := \bg_{m+1} \circ \bbf_{m} \in \SNN_h(2m+2,(1,2,2,4,4,\ldots,2^{m+1},2^{m+1}),5\cdot 2^{m+1})$, completing the proof. \qedhere

\end{proof}

One way to explain the construction of the \texttt{CLOCK} module is to think of a rooted binary tree of depth $m$ where each node is replaced by an \texttt{ODD/EVEN} function. The next lemma characterizes how to reach a specific output spike train by following a path in this tree based on the binary expansion of the desired output index.

\begin{lemma}[Bit-path characterization of the \texttt{ODD/EVEN}-tree]\label{lem:odd_even_path}
    Let $\calI \in \TT$. Consider the depth-$q$ binary tree obtained by consecutively applying \texttt{ODD} or \texttt{EVEN} functions to the outputs at every neuron.  Specifically, take a neuron $N$ of the tree at depth $\ell\in [q]$ and define a binary vector $\bb \coloneqq \{0,1\}^\ell$ by following the path from the root to $N$ where an \texttt{ODD} function is coded as $b_\ell=0$ and an \texttt{EVEN} function is coded as $b_\ell=1$. Then, the spike train of $N$ is
    \[
        \Big\{\calI_{(r)} : r-1 \equiv \textstyle \sum_{t=1}^{\ell} b_t 2^{t-1}\!\! \mod 2^{\ell}, r \leq |\calI|\Big\}.
    \]
    In particular, for each $j=1+\sum_{\ell=1}^{q} b_\ell \, 2^{\ell-1}\in [2^q]$, with $b_\ell\in\{0,1\}$ for all $\ell \in [q]$, there exists a unique output neuron with spike train $\{\calI_{(r)} : r \equiv j \mod 2^q, r \leq |\calI|\}$ that can be reached by following the path determined by the bits $b_1, \dots, b_q.$ 
\end{lemma}
\begin{proof}
    We prove the statement by induction on the depth $\ell \in [q]$. For the base case, $\ell = 1$, note that after one application of the \texttt{ODD/EVEN} function, one output contains precisely the spikes at the odd positions of $\calI$, i.e., $\{\calI_{(1)},\calI_{(3)},\ldots \}=\{\calI_{(r)} : r-1 \equiv 0 \mod 2, r \leq |\calI|\}.$ The other output contains precisely the spikes at the even positions, i.e., $\{\calI_{(2)},\calI_{(4)},\ldots \}=\{\calI_{(r)} : r-1 \equiv 1 \! \mod 2, r \leq |\calI|\}.$ This shows the claim for $\ell=1.$

    For the induction step, assume the claim holds for some $\ell\in\{1,\ldots,q-1\}$. Fix bits $b_1,\ldots,b_\ell$ and let
    \[
        i_\ell \coloneqq 1 + \sum_{t=1}^{\ell} b_t 2^{t-1}.
    \]
    By the induction hypothesis, there exists a neuron at depth $\ell$ with spike train $\calI(b_1,\ldots,b_\ell):=\{\calI_{(r)} : r-1 \equiv i_\ell-1 \mod 2^{\ell}, r \leq |\calI|\}.$ For $k=0,1,\ldots$, define $r_k := i_\ell + k\,2^{\ell}$. Apply now an \texttt{ODD/EVEN} function to $\calI(b_1,\ldots,b_\ell).$ The \texttt{ODD} function outputs $\{\calI_{(r_k)}: k \text{\ even}\}$, whereas the \texttt{EVEN} function outputs $\{\calI_{(r_k)}: k \text{\ odd}\}$ (since we start at index $k=0$). Consequently, for the \texttt{ODD} child ($b_{\ell+1}=0$) and $k=0,1,\ldots$ we have
    \[
        r_{2k} - 1 = (i_\ell-1) + (2k)\,2^{\ell} \equiv i_\ell - 1 \! \mod 2^{\ell+1},
    \]
    and for the \texttt{EVEN} child ($b_{\ell+1}=1$) and $k=0,1,\ldots$ we have
    \[
        r_{2k+1} - 1 = (i_\ell-1) + (2k+1)2^{\ell} \equiv (i_\ell - 1) + 2^{\ell}\! \mod 2^{\ell+1}.
    \]
    In either case, the new residual modulo $2^{\ell+1}$ is
    \[
        i_{\ell+1}-1 \coloneqq \Big(1 + \sum_{t=1}^{\ell+1} b_t 2^{t-1}\Big) - 1
        = (i_\ell - 1) + b_{\ell+1} 2^{\ell},
    \]
    verifying the claim for depth $\ell+1$. This completes the induction.
    Taking $\ell=q$ gives that the neuron along the path $b_1,\ldots,b_q$ has spike train $\{\calI_{(r)}: r-1\equiv j-1 \ \mathrm{mod}\ 2^q,\ r\le|\calI|\}$ with $j:=1+\sum_{t=1}^q b_t2^{t-1}$, i.e.\ it outputs exactly the spikes with $r\equiv j\ \mathrm{mod}\ 2^q$.
\end{proof}

\begin{lemma}[\texttt{SPIKE}$_m$ SNNs]\label{lem:drop_spike}
Let $h\in (0,\infty]$. Then, the following holds.
\begin{enumerate}
     \item There exists $f_1 \in \SNN_h(4, (1,2,2,4,4, 5, 1), 36)$ such that $f(\calI) = \texttt{SPIKE}_1(\calI)$ for any input spike train $\calI \in \mathbb{T}$. 
    \item For $m\in \NN,$ there exists an SNN $f_m\in \SNN_{h}(L, \bp, s)$ with architecture satisfying $L= O(1+\log(m))$ and $\|\bp\|_1 +s = O(\sqrt{m})$ such that $f_m(\calI) = \texttt{SPIKE}_m(\calI)$ for any input spike train $\calI \in \mathbb{T}$. %
\end{enumerate}
\end{lemma}

    For the proof of \Cref{lem:drop_spike}$(i)$ we rely on the following lemma which we prove at the end of this subsection. %

    \begin{lemma}\label{lem:translate}%
    For $h \in (0, \infty]$, the SNN $f \in \SNN_h(1, (3,1), 3)$ defined by a single output neuron and three input neurons $T,S,C$ with weights $1,1,$ and $-1$ has the following~property:

    For infinite spike trains $\calT, \calS$, $\calC\in \TT$ for $T,S,C$, respectively, which are pairwise disjoint and with convention $\calS_{(0)}, \calC_{(0)}\coloneqq 0$, assume that one of the following conditions is met. 
    \begin{enumerate}
        \item $\calS_{(j)}< \calC_{(j)} < \calS_{(j+1)}$ with $|\calT\cap (\calC_{(j-1)}, \calS_{(j)})|\leq 1$ and $|\calT \cap (\calS_{(j)}, \calC_{(j)})| = 0$ for all $j\in \NN$. 
        \item $\calC_{(j)} < \calS_{(j)} < \calC_{(j+1)}$ with $|\calT \cap (\calS_{(j-1)}, \calC_{(j)})| = 0$ and $|\calT \cap (\calC_{(j)}, \calS_{(j)})| \leq 1$ for all $j\in \NN$.
    \end{enumerate}
    Then, the output spike train of $f$ when applied to the input spike trains $\calT, \calS, \calC$ is given by
    \begin{align*}
        f(\calT, \calS, \calC) = \{\calS_{(j)} : |\calT \cap (\calS_{(j-1)}, \calS_{(j)})|= 1\}.
    \end{align*}  
    \end{lemma}

The distinctive feature of the SNN in \Cref{lem:translate} is that there is an output spike whenever there is a spike from $\calS$ that is preceded by a spike from $\calT$. One way to interpret the role of the spike train $\calC$ is to view it as a reset mechanism, as it makes sure that no potential buildup occurs between two consecutive spikes of $\calS$.

We now provide the proof of \Cref{lem:drop_spike}. 

\begin{proof}[Proof of \Cref{lem:drop_spike}]
    Without loss of generality, we may assume that $\calI$ is an infinite spike train, as the output at the $k$-th input spike only depends on the first $k$ spikes of $\calI$. 
    
 For Assertion $(i)$, we consider an SNN which implements the \texttt{CLOCK}$_4$ function. The output neurons are denoted by $O_1, O_2, O_3, O_4$ and the corresponding output spike trains are denoted by $\calO_1, \calO_2, \calO_3, \calO_4\subseteq (0, \infty)$, respectively. We then attach another layer of neurons $N_0, N_1, N_2, N_3, N_4$, where $N_0$ is connected to all $O_i$ with weight $2$, while $N_i$ for $i\in [4]$ is connected to $O_i$ with weight $1$ and to $O_{i+1}$ with weight $1$ if $i+1 \leq 4$ or to $O_{1}$ if $i+1 > 4$ with weight $1$. Further $N_i$ is connected to $O_{i+2}$ with weight $-1$ if $i+2 \leq 4$ or to $O_{i+2 - 4}$ if $i+2 > 4$ with weight $-1$. By \Cref{lem:translate}, the spike train of $N_0$ equals $\calI$ whereas the spike train of $N_i$ for $i\in [4]$ equals 
    \begin{align*}
        \{\calI_{(r+1)} : r \equiv i \mod 4, r\in \NN\}.
    \end{align*}
    We then attach another neuron to all neurons $N_0, N_1, N_2, N_3, N_4$ with weights $2$, $-2$, $-2$, $-2$, and $-2$, respectively. The output spike train of this neuron equals $\{\calI_{(1)}\} = \texttt{SPIKE}_1(\calI)$, and satisfies the asserted construction properties. %

    For Assertion $(ii)$ we rely on \Cref{thm:universal_representation_single_input} which states that there exists such an SNN since \texttt{SPIKE}$_m$ is $m$-finite and $1$-sparse. \qedhere

\end{proof}

\begin{proof}[Proof of \Cref{lem:translate}]
    Let $P$ be the potential of the output neuron. Since all spike trains $\calT, \calS, \calC$ are pairwise disjoint, and since the weight from $C$ to the output neuron is $-1$, it follows that $P(\calC_{(k)}) = 0$ for all $k \in \NN$ by Equation \eqref{eq:potential_definition}, and further $P(C_{(0)}) = P(0)=0$ by convention. In particular, it follows that $\calC\cap f(\calT, \calS, \calC) = \emptyset$.

    Under $(i)$, we have $\calC_{(j-1)}<\calS_{(j)}< \calC_{(j)} < \calS_{(j+1)}$ for all $j\in \NN$ with $|\calT\cap (\calC_{(j-1)}, \calS_{(j)})|\leq 1$ and $|\calT \cap (\calS_{(j)}, \calC_{(j)})| = 0$. If there exists $\tau \in \calT\cap (\calC_{(j-1)}, \calS_{(j)})$, then $P(\tau)=1$, hence there is no output spike at $\tau$, while $P(\calS_{(j)}) = 1+ \exp(-(\calS_{(j)} - \tau)/h) > 1$, leading to an output spike at $\calS_{(j)}$. Immediately afterward, the potential is set to zero and stays zero (at least) until $\calC_{(j)}$. If there is no such $\tau$, then $P(\calS_{(j)}) = 1$, resulting in no output spike at $\calS_{(j)}$, and the potential is reset to zero at $\calC_{(j)}$. 
    
    Meanwhile, under $(ii)$, we have $\calC_{(j)} < \calS_{(j)} < \calC_{(j+1)}< S_{(j+1)}$ for all $j\in \NN$. Now, recall that $f(\calT, \calS, \calC) \cap [0,\calC_{(1)}] = \emptyset$ with $P(\calC_{(1)}) = 0$, hence it follows for $\calC' \coloneqq \{\calC_{(j+1)} : j\in \NN\}$ that
    $f(\calT, \calS, \calC) \cap (\calC_{(1)}, \infty) = f(\calT, \calS, \calC') \cap (\calC_{(1)}, \infty)$, which reduces the analysis to case $(i)$. 
\end{proof}

\begin{lemma}[\texttt{REPRESENT}$_{\br}$ SNN] \label{lem:represent_function} Let $h\in (0, \infty]$ and $\br\in \{0,1\}^m$ for some $m \in \NN$. Then, there exists an SNN $f_{\br}\in \SNN_h(L, \bp,s)$ with architecture satisfying $L = O(1+\log(m)), \|\bp\|_1 = O(\sqrt{m})$ and $s = O(\sqrt{m} + \|\br\|_1)$ such that $f_{\br}(\calI) = \texttt{REPRESENT}_{\br}(\calI)$ for any spike train $\calI\in \TT$. 	
\end{lemma}

\begin{proof}
The assertion follows immediately from \Cref{thm:universal_representation_single_input}	by noticing that the function $\texttt{REPRESENT}_{\br}\colon \TT\to \TT$ is $m$-finite and $\|\br\|_1$-sparse. 
\end{proof}

\subsection{Constructions for Memoryless Multiple Input Functions}

In this subsection, we construct SNNs for the memoryless functions from \Cref{subsec:expl:memoryless_functions}.

\begin{lemma}[Memoryless SNNs] \label{lemma:boolean_construction}
    Let $h\in (0, \infty]$. Then the following holds.
\begin{enumerate}
    \item For any $d\in \NN$ there exists $f \in  \SNN_h(1,(d,1),d)$ such that $f(\calI_1,\dots, \calI_d) = \texttt{OR}(\calI_1, \dots, \calI_d)$ for any input spike trains $\calI_1, \dots, \calI_d \in \mathbb{T}$. 
    \item There exists $f \in  \SNN_h(2,(d,2,1),d+3)$ such that $f(\calI_1, \dots, \calI_d) = \texttt{AND}(\calI_1, \dots, \calI_d)$ for any input spike trains $\calI_1, \dots, \calI_d  \in \mathbb{T}$. 
    \item There exists $f \in  \SNN_h(1,(2,1),2)$ such that $f(\calI, \calJ) = \texttt{MINUS}(\calI, \calJ)$ for any input spike trains $\calI, \calJ \in \mathbb{T}$.
    \item There exists $f \in \SNN_h(2,(2,2,1),3)$ such that $f(\calI, \calJ) =\texttt{XOR}(\calI, \calJ)$ for any input spike trains $\calI, \calJ \in \mathbb{T}$.
    \item There exists $f \in  \SNN_h(2,(4, 4,3,2,1),18)$ such that $f(\calI, \calJ) = \texttt{IS-EQUAL}(\calI, \calJ, \calD)$ for any input spike trains $\calI, \calJ, \calD \in \mathbb{T}$.
\end{enumerate}
\end{lemma}

\begin{proof}
    For the \texttt{OR} function with $d$ inputs, we consider a single output neuron with weight $2$ from all input neurons. Whenever there is a spike from any of the input neurons, the potential of the output neuron becomes $2$ and thus triggers a postsynaptic spike.

    For the \texttt{AND} function with $d$ inputs, we consider two hidden neurons, $N_1$ and $N_2$, and one output neuron $O$. Neuron $N_1$ is connected to the first input neuron with weight $2$, neuron $N_2$ is connected to the first input neuron with weight $2d-2$ and to all other input neurons with weight $-2$. The output neuron $O$ is connected to neuron $N_1$ with weight $2$ and to neuron $N_2$ with weight $-2$. Now, if all input neurons spike, then neuron $N_1$ spikes at the first input spike and neuron $N_2$ does not spike, resulting in an output spike. If the first input neuron does not spike, then neither neuron $N_1$ nor neuron $N_2$ spike, and no output spike occurs. If one of the other input neurons does not spike, while the first input neuron spikes, then both neuron $N_1$ and neuron $N_2$ spike, leading to no output spike. 

    For the \texttt{MINUS} function, we consider a single output neuron connected to the first input neuron with weight $2$ and to the second input neuron with weight $-2$. If there is a spike from the first input neuron that is not coinciding with a spike from the second input neuron, the potential of the output neuron equals $2$ and leads to a postsynaptic spike. Meanwhile, if there is a spike from the second input neuron coinciding with a spike from the first input neuron, then the potential becomes zero and thus no postsynaptic spike occurs.

    Regarding the \texttt{XOR} function, we consider the SNN construction induced by 
    \begin{align*}
        \texttt{XOR}(\calI, \calJ) = \texttt{OR}(\texttt{MINUS}(\calI, \calJ), \texttt{MINUS}(\calJ, \calI)).
    \end{align*}
    
    For the \texttt{IS-EQUAL} function, we use the SNN construction induced via
    \begin{align*}
        \texttt{IS-EQUAL}(\calI, \calJ, \calD) = \texttt{OR}(\texttt{AND}(\texttt{MINUS}(\calD, \calI), \texttt{MINUS}(\calD, \calJ)), \texttt{SKIP}(\texttt{AND}_3(\calI, \calJ, \calD))),
    \end{align*}
    where the \texttt{SKIP} function is used to synchronize the depth of all inputs of the \texttt{OR}-function.\qedhere
\end{proof}

\subsection{Multiple-Input Compositional Functions with Bounded Memory}

In this subsection we provide the SNN constructions for all compositional functions from \Cref{subsec:expl:bounded_memory_functions} with bounded memory. We start with the \texttt{CEIL}$_m$ function. 

\begin{lemma}[\texttt{CEIL}$_m$ SNN]\label{lem:ceil_snn}
    For each $h\in (0, \infty]$ and $m \in \NN$ there exists $f\in \SNN_h(L,\bp, s)$ satisfying $L = O(1+\log(m)), \|\bp\|_1 +s = O(\sqrt{m})$ such that 
    \begin{align*}
        f(\calI, \calD) = \texttt{CEIL}_m(\calI, \calD) \quad \text{ for all } \calI, \calD \in \mathbb{T}.
    \end{align*}    
\end{lemma}

By putting different \texttt{CEIL}$_m$ functions in parallel, and combining them with the  \texttt{IS-EQUAL} function and \texttt{AND} (which are both memoryless) we obtain a bound on the construction of the \texttt{IS-APPROX-EQUAL}$_m$ function. 

\begin{corollary}[\texttt{IS-APPROX-EQUAL}$_m$ SNN] For each $h \in (0, \infty]$ and $m \in\NN$ there exists $f\in \SNN_h(L, \bp, s)$ satisfying $L = O(1+\log(m)), \|\bp\|_1+s = O(m^{3/2})$ such~that 
\begin{align*}
    f(\calI, \calJ, \calD)= \texttt{IS-APPROX-EQUAL}_m(\calI, \calJ, \calD) \quad \text{for all } \calI, \calJ, \calD\in \TT.
\end{align*}
\end{corollary}

For the proof of \Cref{lem:ceil_snn} we rely on the following mechanism which allows us to reset the potential of an SNN using an exogenous spike train. 

\begin{lemma}[Reset mechanism]\label{lem:reset_mechanism}
    For each $h \in (0, \infty]$ and $f\in \SNN_h(L, \bp, s)$ there exists $\tilde f \in \SNN_h(L, \bp + (1, \dots, 1, 0), s + \| \bp \|_1 - p_0+L-1)$ such that for any spike trains $\calI_1, \ldots, \calI_d, \calR \in \mathbb{T}$ with $\calR_{(0)} =0$ and $\calR_{(|\calR|+1)}\coloneqq \infty$ and any $i \in \NN_0$ with $i \leq |\calR|$,
    \begin{align*}
        \tilde f(\calI_1, \ldots, \calI_d, \calR) \cap (\calR_{(i)}, \calR_{(i+1)}]= 
            f\Big(\calI_1\cap  (\calR_{(i)}, \calR_{(i+1)}), \ldots, \calI_d\cap (\calR_{(i)}, \calR_{(i+1)})\Big).
    \end{align*}
\end{lemma}

\begin{proof}
    We construct the SNN $\tilde f$ by modifying the original SNN $f$ to include an additional input neuron $R_0$ with spike train $\calR$, and add \texttt{SKIP} connections from this new input neuron so that at every layer $\ell \in [L-1]$ there is a neuron $R_\ell$ with spike train $\calR$. This leads to $L-1$ additional neurons and non-zero weights. In addition, at every hidden and output neuron of the original SNN in layer $\ell \in [L]$ we add a connection starting at $R_{\ell-1}$ with weight $= - \, (\text{sum of all absolute incoming weights of that neuron})-1$. This leads to $\|\bp\|_1 - p_0$  additional weights. In particular, the additional connections ensure that all neurons get reset whenever there is a spike from $\calR$, and that spikes from $\calI_1, \dots, \calI_d$ that occur at the same time as a spike from $\calR$ do not contribute to the potential of any neuron in the SNN.     
\end{proof}

\begin{proof}[Proof of \Cref{lem:ceil_snn}]
    We begin the proof by considering $m=1$. For any $t \in \NN$ and $\calI, \calD \in \TTT{\NN}$ with $[t] \subseteq \calI \cup \calD$, it holds 
    $$t \in \texttt{CEIL}_1(\calI, \calD) \quad \Longleftrightarrow \quad  t \in \calI \cap \calD \text{ or } \big(t-1 \in \calI \text{ and } t-1 \notin \calD \text{ and } t \in \calD\big).$$ Hence, since $\texttt{CEIL}_1$ is input-dominated, causal and fulfills the monotone scaling property, we infer that $\FF_{\mathrm{mm}}(2,2,4)$, and yields \Cref{thm:function_representation} the assertion. In particular, by the sparsity constraint, the construction in the proof of \Cref{thm:function_representation} ensures that the SNN can be chosen such that every neuron is connected to at most $4$ neurons in the next~layer.

    For $m \geq 2$, we note that for any input spike train $\calI, \calD \in \mathbb{T}$ and $n \in \NN$ it holds, upon defining $\tilde \calI^{n}\coloneqq \calI \cap (\calD_{(n-1)}, \calD_{(n)})$, 
    \begin{align}\label{eq:ceil_case_distinction}
        \texttt{CEIL}_m(\calI, \calD) \cap (\calD_{(n-1)}, \calD_{(n)}]= \begin{cases}
         \calD_{(n)}, \!\!\!\!\!\!\!&\text{ if } |\tilde \calI^{n}| \geq m \text{, or } |\tilde \calI^{n}| = m-1 \text{ and } \calD_{(n)}\in \calI,\\
         \emptyset, \!\!\!\!\!\!\!&\text{ if } |\tilde \calI^{n}| < m-1 \text{, or } |\tilde \calI^{n}| = m-1 \text{ and } \calD_{(n)}\notin \calI. 
        \end{cases}
    \end{align}
This follows directly from the definition of the \texttt{CEIL}$_m$ function, and allows us to use the reset mechanism in \Cref{lem:reset_mechanism} to construct the desired SNN. 

Concretely, we consider two SNNs $f_1$ and $f_2$ with input spike train $\calI$ that implement the \texttt{SPIKE}$_{m}$ and \texttt{SPIKE}$_{m-1}$ function, respectively according to \Cref{lem:drop_spike}$(ii)$. Using the reset mechanism from \Cref{lem:reset_mechanism}, we extend both SNNs $f_1$ and $f_2$ such that they have input $\calI$ and are reset at every spike from $\calD$. 
Additionally, we consider an SNN $f_3$ that implements the \texttt{AND} function of $\calI$ and $\calD$ according to \Cref{lemma:boolean_construction}$(ii)$, and attach \texttt{SKIP} connections  to $f_2$ and $f_3$ to synchronize their depth with $f_1$. Finally, we also consider an SNN $f_4$ with input $\calD$ which only consists of \texttt{SKIP} connections to synchronize its depth with the other SNNs. 
The corresponding output neuron of $f_i$ is called $N_i$. So far, this requires an architecture of depth $L = O(1+\log(m))$ and $O(\sqrt{m})$ neurons and non-zero weights determined by \Cref{lem:drop_spike}. %

We now consider a $4$-input SNN $f_5$ which implements the $2$-Markovian function $F\colon \TT^4\to \TT$ defined as follows: for any input spike trains $\calT_1, \calT_2, \calT_3, \calT_4 \in \TTT{\NN}$ with $\bigcup_{i\in [4]}\calT_i= \NN$ it holds that 
\begin{align*}
    t \in  F(\calT_1, \calT_2, \calT_3, \calT_4) \quad \Longleftrightarrow \quad t \in \calT_4 \text{ and } \big( t-1 \in \calT_1 \text{ or } (t-1 \in \calT_2 \text{ and } t\in \calT_3) \big).
\end{align*} %
This can be realized according to \Cref{thm:function_representation} with a bounded architecture. 
Composing $(f_1, \dots, f_4)$ with $f_5$ then yields by the characterization in \eqref{eq:ceil_case_distinction} and the reset mechanism (\Cref{lem:reset_mechanism}) an SNN which implements the \texttt{CEIL}$_m$ function as desired. \qedhere

\end{proof}

\begin{lemma}\label{lem:if_then}
    For each $h\in (0, \infty]$,  $m \in \NN$, and reference pattern $\br \coloneqq (r_1, \ldots, r_m)\in \{0,1\}^m$ there exists $f_{\br} \in \SNN_h(L,\bp, s)$ with $L = O(1+\log(m))$, $\|\bp\|_1 = O(\sqrt{m})$, and $s = O(m)$ such that for any input spike trains $\calI, \calD \in \mathbb{T}$,
    \begin{align*}
        f_{\br}(\calI, \calD) = \texttt{IF-THEN}_m(\calI, \br, \calD).
    \end{align*}
\end{lemma}
\begin{proof}

    We first devise a two-input SNN $f_1$ that implements the \texttt{CEIL}$_1$ function according to \Cref{lem:ceil_snn}. This can be realized with a bounded number of neurons and the input to this SNN are the spike trains $\calI$ and $\calD$.     
    Using \Cref{lem:represent_function}, we construct single-input SNNs $f_2$ and $f_3$ implementing the $m$-finite functions $\texttt{REPRESENT}_{\br}$ and $\texttt{REPRESENT}_{\mathbf{1}_m}$, which are $\|\br\|_1$- and $m$-sparse respectively. The SNNs can be realized with depth of order $O(1+\log(m))$, $O(\sqrt{m})$ neurons, and at most $O(m)$ nonzero weights. The inputs to both of these SNNs will be $\calD$.    
    We then attach $O(1+\log(m))$ layers of \texttt{SKIP} connections to the output of the first SNN $f_1$ to synchronize the depth of all three SNNs. We then attach an SNN $f_4$ that implements the \texttt{IS-EQUAL} function (\Cref{lemma:boolean_construction}$(vi)$) to the outputs of SNNs $f_1$, $f_2$, and $f_3$. The \texttt{IS-EQUAL} SNN can be implemented with a bounded architecture. 
    Finally, we attach another SNN that implements the \texttt{SPIKE}$_m$ function (\Cref{lem:drop_spike}$(ii)$) to the output of the \texttt{IS-EQUAL} SNN. This requires another $O(1+\log(m))$ layers and $O(\sqrt{m})$ neurons and weights. 
    The overall architecture then implements the \texttt{IF-THEN}$_m$ function as desired, and the asserted bounds on the architecture follow by construction.
\end{proof}

\begin{lemma}\label{lem:delay_repeat}
    Let $h \in (0, \infty]$ and $m \in \NN$. Then, there exists $f_i \in \SNN_h(L_i,\bp_i, s_i)$ with $L_1+ L_2= O(1+\log(m)), \|\bp_1\|_1+s_1 = O(m)$, and $ \|\bp_2\|_1+s_2 = O(m^2)$ such that for any two input spike trains $\calI, \calD \in \mathbb{T}$,
    \begin{align*}
        f_1(\calI, \calD) = \texttt{DELAY}_m(\calI, \calD) \quad \text{ and }\quad f_2(\calI, \calD) = \texttt{REPEAT}_m(\calI, \calD).
    \end{align*}
\end{lemma}

\begin{proof} 
    For the \texttt{DELAY}$_m$ function, we attach a single neuron $J$ with weight $2$ to the two input neurons $I$ and $D$, producing the spike train $\calJ = \calI\cup \calD$. For $\overline m \coloneqq 2^{\lceil \log(m)\rceil +2}$ we then attach an SNN representing the \texttt{CLOCK}$_{\overline m}$ function to neuron $J$ and denote the output neurons by $N_1, \dots, N_{\overline m}$. In parallel, we attach \texttt{SKIP} connections to the input neuron $I$ until the depth is synchronized with the output of the \texttt{CLOCK}$_m$ SNN. We denote the output by $N_0$. We now attach to each pair $N_i$ and $N_0$ for $i\in [\overline m]$ an SNN representing the \texttt{AND} function, and call the output $M_i$. In parallel, we attach \texttt{SKIP} connections to all neurons $N_1, \dots, N_{\overline m}$ until the depth is synchronized and call the output $K_{1}, \dots, K_{\overline m}$. We now attach another layer of neurons $H_1, \dots, H_{\overline m}$ where $H_i$ is connected to $M_i$ with weight $1$, as well as $K_{i+m-1}$ if $i+m-1 \leq \overline m$ or $K_{i+m-1-\overline m}$ if $i+m-1> \overline m$ with weight $1$. In addition, we attach to $H_i$ the neuron $K_{i+m}$ if $i+m\leq \overline m$ and $K_{i+m-\overline m}$ if $i+m>\overline m$ with weight $-1$. Based on the choice of $\overline m$ it follows that $i \not \equiv i+m \mod \overline m$ and $i \not \equiv i+m+1 \mod \overline m$ since $m+1< \overline m$. By \Cref{lem:translate}, it follows that the output spike train of $H_i$ is given by 
    \begin{align*}
        \calH_{i}\coloneqq \big\{\calJ_{(k+m-1)} \colon  k \equiv i, k \in \NN_0, k+m-1\leq |\calJ|, \calJ_{(k)}\in \calI\big\}.  
    \end{align*}
    Attaching another neuron with weight $2$ to all neurons $K_i$ yields the desired output spike train for the \texttt{DELAY}$_m$ functions. Overall, this leads to depth $O(1+ \log(m))$ and $O(m)$ neurons and weights. 

    For the \texttt{REPEAT}$_m$ function, we connect $m+1$ SNNs  to the input neurons $I$ and $D$ in parallel where each SNN model one \texttt{CEIL}$_k$ functions for $k \in [m+1]$. This requires maximal depth $O(1+\log(m))$ and $O(m^2)$ neurons and non-zero weights. 
     We attach to all SNN outputs for $k\in [m]$  enough \texttt{SKIP}-connections to synchronize the depth of all SNNs, requiring another $O(m(1+\log(m)))$ neurons and non-zero weights. Finally, composing all these SNNs with a $1$-layer SNN representing the \texttt{OR} function yields the assertion. 
\end{proof}

\subsection{Explicit Constructions for Functions used in the Proof of Theorem \ref{thm:function_representation}}\label{subsec:proof:constructions}

 We now provide the proof for the construction of the \texttt{MEMORY}$_m$ function (\Cref{lem:memory_module}) and the representation of memoryless functions determined by a Boolean function (\Cref{prop:boolean_function_representation}).

\begin{proof}[Proof of \Cref{lem:memory_module}] 
As a first step of the hierarchical construction  for the \texttt{MEMORY}$_m$ function   we consider an SNN with $d$ input neurons whose spike trains are denoted by $\calI_1, \ldots, \calI_d$. If $m = 1$, then we attach a single layer of \texttt{SKIP} connections to these input neurons and are done. Hence, we assume $m \geq 2$ throughout the remainder of the proof. 

The first layer consists of $d+1$ neurons. For each $i \in [d]$, we add a neuron $v_{1,i}$ which is connected to the respective input neuron with weight $2$. In addition, we define the neuron $v_{1,0}$ which receives weight $2$ from all input neurons. This implements an \texttt{OR} function attached to all input neurons. The corresponding spike train of neuron $v_{1,0}$ is given by $\calI\coloneqq \bigcup_{i = 1}^{d} \calI_i$. This leads to $O(d)$ additional neurons and non-zero weights. 

We then attach for $\overline m \coloneqq 2^{\lceil \log(m) \rceil+2}$ a \texttt{CLOCK}$_{\overline m}$ network to the neuron $v_{1,0}$. This can be done in $O(\log(m))$ layers using the construction from \Cref{lem:periodic_finite_functions}. We denote the output neurons by $v_{2,1}, \ldots, v_{2,\overline m}$ and the corresponding spike trains by $\calI^{1}, \ldots, \calI^{\overline m}$ with $$\calI^{\ell} = \{\calI_{(\overline m k  + \ell)} \colon k \in \NN_0, \overline m k  + \ell\leq |\calI|\}\quad \text{ for }\ell \in [\overline m].$$ To the remaining neurons from the first layer, we attach \texttt{SKIP} connections until the depth is synchronized with the neurons $v_{2,j}$ for $j \in [d]$. We denote the output neurons by $v_{2,1}', \ldots, v_{2,d}'$. This leads to $O(\log(m))$ additional layers, $O(m +  d\log(m))$ additional  neurons and non-zero weights. 

We then define for each $i \in [d]$ and $\ell \in [\overline m]$ an \texttt{AND}-function with input neurons $v_{2,i}'$ and $v_{2,\ell}$. The output neurons are denoted by $v_{3,i, \ell}$ and the corresponding spike trains are given by  $$\calI_{i, \ell} = \{ \calI_{(\overline m k+\ell)} \colon \calI_{(\overline m k+\ell)} \in \calI_i, \overline m k+\ell\leq |\calI|, k \in \NN \}.$$%
In addition, we attach \texttt{SKIP} connections to the neurons $v_{2,1}, \ldots, v_{2,\overline m}$ to synchronize the depth of this layer. These output neurons are denoted by $v_{3,1}, \ldots, v_{3,\overline m}$. This leads to $O(1)$ additional layers, $O(dm)$ additional neurons and non-zero weights. 

We then attach another layer consisting of $d \overline m m$ neurons, where for each $i \in [d]$, $\ell\in [\overline m]$, and $j \in [m]$ we define a neuron $v_{4,i,\ell,j}$ to have weight $1$ from neuron $v_{3,i,\ell}$ and $v_{3,\ell+j-1}$ if $\ell+j-1\leq \overline m$ or $v_{3,\ell+j-1-\overline m}$ if $\ell+j-1> \overline m$, and additionally weight $-1$ from neuron $v_{3,\ell+j}$ if $\ell+j\leq \overline m$ or $v_{3,\ell+j-\overline m}$ if $\ell+j>\overline m$. By \Cref{lem:translate} and since $m +1\leq \overline m$, the spike train of neuron $v_{4,i,\ell,j}$, denoted by $\calI_{i, \ell,j}$, is given by
    \begin{align*}
        \calI_{i, \ell,j} = \{\calI_{(\overline m k + \ell + j -1)} \colon \calI_{(\overline m k + \ell)} \in \calI_i, \overline m k + \ell + j -1\leq |\calI|, k \in \NN_0\}.
    \end{align*}
This leads to $O(dm^2)$ additional neurons and non-zero weights.

Finally, for each $i \in [d]$ and $j \in [m]$ we attach an \texttt{OR}-function to the neurons $v_{4,i,\ell,j}$ for $\ell \in [\overline m]$. The output neurons are denoted by $v_{5,i,j}$ and their spike trains are given by 
\begin{align*}
    \calO_{i}^{j} \coloneqq \bigcup_{\ell\in [\overline m]}\calI_{i, \ell,j} = \{\calI_{(k)} \colon \calI_{(k-j+1)} \in \calI_i, k \in \NN, k - j + 1 \geq 1\}.
\end{align*}
This leads to a single additional layer, and $O(dm^2)$ additional neurons and non-zero weights. Altogether, this construction requires $O(\log(m))$ layers, $O(dm^2)$ neurons, and $O(dm^2)$ non-zero weights, which concludes the proof.
\end{proof}

\begin{proof}[Proof of \Cref{prop:boolean_function_representation}]

We construct the SNN layer by layer. The input layer consists of $d$ input neurons $v_{0,1}, \dots, v_{0,d}$ and the corresponding spike trains are denoted by $\calI_1, \ldots, \calI_d$.

The first hidden layer consists of $d+1$ neurons. The first $d$ neurons, denoted $v_{1,1}, \dots, v_{1,d}$ are \texttt{SKIP} connections from the input, i.e., each neuron $v_{1,i}$ receives weight $2$ from the corresponding input $v_{0,i}$. The $(d+1)$-th neuron, denoted by $v_{1, \cup}$, is connected to all input neurons with weight $2$. This acts as an \texttt{OR}$_d$-function of the $d$ inputs. 

The second layer consists of $2d+1$ neurons. For each $i \in [d]$, the neuron $v_{2,i}^+$ receives weight $2$ from $v_{1,i}$ and thus acts as a \texttt{SKIP}-connection, and, additionally, the neuron $v_{2,i}^-$ receives weight $2$ from the union neuron $v_{1, \cup}$ and weight $-2$ from neuron $v_{1,i}$. This acts as a \texttt{MINUS} function so that $v_{2,i}^-$ spikes only if $v_{1,\cup}$ spikes and $v_{1,i}$ does not. This ensures that $v_{2,i}^+$ spikes if and only if $v_{1, \cup}$ spikes and $v_{1,i}$ does not spike. The final neuron, $v_{2, \cup}$, in the second layer is connected to $v_{1, \cup}$ and has weight $2$, thus also serving as a \texttt{SKIP}-connection.

To construct the third layer, we consider the set $S= \{\bx\in \{0,1\}^d \colon F(\bx)=1\}$. For each $\by\in S$, we construct a neuron $v_{3,\by}$ which spikes at any $t\in \RR$ if and only $(\mathds{1}(t\in \calI_i))_{i \in [d]} = \by$. In particular, $\mathbf{0} \notin S$. Therefore, an output spike will only occur if at least one input neuron spikes at time $t$.
For the construction, we connect for each $i\in [d]$ the neuron $v_{2,i}^+$ to $v_{3,\by}$ with weight $2$ if $y_i=1$, or the neuron $v_{2,i}^-$ to $v_{3,\by}$ with weight $2$ if $y_i=0$. We additionally connect the union neuron $v_{2, \cup}$ to $v_{3,\by}$ with an inhibitory weight $2 - 2d$. These weights ensure the desired behavior and realize an \texttt{AND} function between the neurons $v_{2,i}^{\pm}$ which are connected to $v_{3,\by}$. This involves $O(r)$ neurons and $O(dr)$ non-zero weights. 

Finally, all neurons $v_{3,y}$ are connected to a single output neuron with weight $2$, which serves as an \texttt{OR}$_r$-function. Overall, the output neuron spikes at time $t$ if and only if there exists a $\by \in S$ such that $(\mathds{1}(t\in \calI_i))_{i \in [d]} = \by$, and the desired function on the spike trains determined by $F$ is realized. In particular, it holds $L = 4$, the total number of neurons is $O(d + r)$,  the total number of non-zero weights is $O(dr)$, and the architecture is independent of $h\in [0,\infty]$.
\end{proof}

\subsection{Real Weight SNNs are more Expressive than Integer Weight SNNs}

Even though the previous constructions use integer weights, real valued weights allow for additional expressiveness.

\begin{lemma}\label{thm:counterexample-integer-weights}
    Let $\II := \{(\calI, \calJ) \in \TT^2 \colon |\calI| = 3, |\calJ| = 2, \calI_{(3)} = \calJ_{(1)}\}$. Then, the function $F: \II \to \TT$ given by $F(\calI, \calJ) = \{\calI_{(3)}\}$ is causal, input dominated and has the monotone scaling property (on $\II$). Further, $F$ can be represented by an SNN in the class $\SNN_h(1,(2,1),2)$ with real valued weights but not by an SNN in $\SNN_h(1,(2,1),2)$ with integer values.
\end{lemma}

\begin{proof}%
    For $(\calI, \calJ) \in \II$ write $\calS := \calI \cup\calJ = \{\calS_{(1)},\calS_{(2)},\calS_{(3)},\calS_{(4)}\}$. Hence, $\calI = \{\calS_{(1)},\calS_{(2)}, \calS_{(3)}\}$ and $\calJ = \{\calS_{(3)}, \calS_{(4)}\}$. We want the output spike train to be $\{\calS_{(3)}\}$ for input spike trains $(\calI, \calJ)$. This can be achieved by an SNN with two inputs, one output and weights $w = (w_1,w_2)$ satisfying
    \[
    0<w_1\le\frac12,\qquad 1-w_1<w_2\le 1.
    \]
    At $\calS_{(1)}$ the potential is 
    $P(\calS_{(1)})=w_1<1$ and no spike occurs. At $\calS_{(2)}$,
    \[
        P(\calS_{(2)})= e^{-(\calS_{(2)}-\calS_{(1)})/h} w_1 + w_1
        = w_1\bigl(1+e^{-(\calS_{(2)}-\calS_{(1)})/h}\bigr)
        < 2w_1 \le 1,
    \]
    so again no spike occurs. At $\calS_{(3)}$, both input neurons spike and the potential exceeds  the threshold,
    \[
        e^{-(\calS_{(3)}-\calS_{(2)})/h} P(\calS_{(2)}) + w_1 + w_2
        \ge w_1 + w_2 > 1.
    \]
    Therefore, the output spikes. At $\calS_{(4)}$ the potential equals $w_2 \in (1-w_1,1] \subset (1/2,1]$, and no spike occurs. Thus, the output equals $\{\calS_{(3)}\}$ for all $(\calI, \calJ) \in \II$. 
    
    However, the same cannot be achieved with integer weights. To see this, consider weights $v=(v_1,v_2)\in\ZZ^2$ and suppose, for a contradiction, that the output spike train is $\{\calS_{(3)}\}$ for this input. As the output should not spike at $\calS_{(1)}$, we must have $v_1\le 1$. The absence of a spike at $\calS_{(2)}$ implies $v_1\le 0$ (for $v_1=1$ the potential at $\calS_{(2)}$ equals $1+e^{-(\calS_{(2)}-\calS_{(1)})/h}>1$ and $\calO$ would spike). At $\calS_{(3)}$ there is a spike, hence $v_2>1$. At $\calS_{(4)}$ there is no spike, hence $v_2\le1,$ contradicting $v_2>1$. Therefore, no integer vector $v\in\ZZ^2$ can realize this specific output spike pattern.
\end{proof}

\section{Expressivity Bounds of SNNs } \label{subsec:proof:SNN_expressivity}

In this section, we provide the proof for the quantitative statements on the expressiveness of SNNs (\Cref{subsec:app:mainExpressivenenessStatements}) as well as the cardinality lower bound on causal, input-dominated functions with monotone scaling property (\Cref{subsec:app:proof_cardinality}). To this end, we first establish some auxiliary statements (\Cref{subsec:app:auxiliary_expressiveness}) whose proofs are deferred to \Cref{app:proofs_technical_lemmas}.

\subsection{Auxiliary Statements for Expressiveness of SNNs}\label{subsec:app:auxiliary_expressiveness}

To quantify the expressiveness of general feedforward SNNs with a fixed architecture, we first consider the simpler case of a single SNN computational unit. To analyze its expressiveness, we start by deriving an explicit formula for the membrane potential of the output neuron. 

Recall that the membrane potential from \eqref{eq:potential_definition} at time $t$ of an SNN computational unit $\phi_{h,\mathbf{w}}$ with memory $h\in [0, \infty]$ and weight vector $\mathbf{w} \in \RR^d$ receiving $d$ input spike trains $\calI_1, \dots, \calI_d\subseteq (0, \infty)$ is defined recursively as
\begin{align}
    {P}(t) :=\Big(  \mathds{1}({P}(\prev{t})\leq 1)  {P}(\prev{t})e^{-\frac{t-\prev{t}}h}   + \sum_{j : t \in \calI_j} w_j\Big)_{+}, \label{eq:potential_definition_2}
\end{align}
where $\prev{t}$ denotes the presynaptic last spike time strictly before $t$ among all input spike trains and the output spike train. Crucial for our analysis of the expressiveness of SNN units is the following unfolding of the recursion in \eqref{eq:potential_definition_2}, which entails the contribution of each individual input neuron to the membrane potential for a given time frame $[r,t]$. Given input spike trains $\calI_1, \dots, \calI_d\in \TT$ and time points $0 < r \leq t,$ we define the vector of contributions of the $i$-th input neuron, for $i \in [d]$, to the membrane potential at time $t$ as
\begin{align}
    \bp^{[r,t]} \coloneqq (p_i^{[r,t]})_{i \in [d]} \in \mathbb{R}_{\geq 0}^d,\qquad  p_i^{[r,t]} := \textstyle \sum_{\tau \in \calI_i \cap [r,t]} e^{-\frac{t-\tau}h}. \label{eq:potential_vectors}
\end{align}

\begin{lemma}[Unfolded membrane potential]\label{lem:meta_potential_formula}%
Let $\calI_1, \dots, \calI_d\in \TT$ be $d$ spike trains and let $\phi_{h,\mathbf{w}}$ be an SNN unit with memory $h\in [0, \infty]$, $d$ input neurons, and weight  vector $\mathbf{w} = (w_1,\ldots,w_d)$. %
Enumerate the spike times of $\calI \coloneqq \bigcup_{j\in [d]} \calI_j$ by $\tau_1< \tau_2 < \dots$, set $\tau_0\coloneqq 0$, and let $r\in [|\calI|]$. Set $j\coloneqq \max\{i \in \{0,1,\ldots,r-1\}: P(\tau_i) \notin (0,1]\}$ for $P$ from \eqref{eq:potential_definition_2}. Then, the vectors $\bp^{[\tau_j, \tau_r]}$ defined in \eqref{eq:potential_vectors}, it holds that %
\begin{align*}
    P(\tau_r) = \textstyle \left(\sum_{k = j+1}^{r} \sum_{l\colon \tau_k \in \calI_l} w_l \exp\big(-(\tau_r - \tau_k)/h\big)\right)_+ = \left(\sum_{l=1}^{d} w_l p_l^{[\tau_{j+1}, \tau_r]}\right)_+ = \left(\langle \mathbf{w}, \bp^{[\tau_{j+1}, \tau_r]}\rangle\right)_+, 
\end{align*} 
where $\langle \cdot, \cdot \rangle$ denotes the standard inner product on $\RR^d$.
\end{lemma}

The above result reveals a close relation between the membrane potential at the spike times and the vectors $\bp^{[r, t]}$. In particular, it allows us to devise a geometric condition for two SNN computational units with different weight vectors to produce identical output spike trains for the same inputs. To this end, we define the set of all possible $\bp^{[r, t]}$ for given input spike trains $\calI_1, \dots, \calI_d\in \TT$ as
    \begin{align}
        \calP(\calI_1, \dots, \calI_d) := \Big\{\bp^{[r, t]} \;\Big|\;   \text{ for } r, t\in \textstyle\bigcup_{j = 1}^{d}\calI_j, r\leq t \Big\}.\label{eq:potential_vectors_class}
    \end{align}

\begin{lemma}[Identifiability of outputs of SNN units]\label{lem:identifiability}
    Let $\calI_1, \dots, \calI_{d}\in \TT$ be spike trains and let $\phi_{h,\mathbf{w}_1}, \phi_{h,\mathbf{w}_2}$ be two SNN computational units with memory $h\in [0, \infty]$ and $d$ input neurons, which are parametrized by the respective weights vectors $\mathbf{w}_1$ and $\mathbf{w}_2$.  
    Define the function $\chi : s \mapsto \mathds{1}(s>1)- \mathds{1}(s\leq 0)$. If  
    \begin{align}\label{eq:identifiability_condition}
        \chi(\langle \mathbf{w}_1, \bp\rangle) = \chi(\langle \mathbf{w}_2, \bp\rangle) \quad \text{ for all } \bp\in \calP(\calI_1, \dots, \calI_d),
    \end{align}
    then $\phi_{h,\mathbf{w}_1}(\calI_1, \dots, \calI_d)=\phi_{h,\mathbf{w}_2}(\calI_1, \dots, \calI_d).$
\end{lemma}

We now proceed with an upper bound on the number of distinct functions that can be represented by a single SNN computational unit. 

\begin{proposition}[expressiveness of one SNN computational unit]\label{thm:single_neuron_expressivity}
    Consider a finite collection of $d$-tuples of spike trains $\II \subseteq \TT^d$,  and denote by $\phi_{h,\mathbf{w}}$ an SNN computational unit with memory $h\in [0, \infty]$ and weight vector $\mathbf{w} \in \RR^d$. Define $\calP(\II)\coloneqq \bigcup_{(\calI_1, \dots, \calI_d)\in \II} \calP(\calI_1, \dots, \calI_d)$ with $\calP(\calI_1, \dots, \calI_d)$ from \eqref{eq:potential_vectors_class} and let $H \coloneqq  |\calP(\bbI)|$. Then, it follows that
    \begin{align}\label{eq:single_neuron_expressivity}
    \begin{aligned}
       &\Big|\big\{\bbB_{\bw}\colon \bbI \to \mathbb{T}, (\calI_1, \dots, \calI_d)\mapsto \phi_{h,\mathbf{w}}(\calI_1, \dots, \calI_d) \;\colon \;\mathbf{w} \in \RR^d\big\}\Big|
       \leq (8e H)^{d}.
    \end{aligned}\end{align}
    Moreover, $H\leq |\bbI|(dT)^2$ for $T\coloneqq \max_{(\calI_1, \dots, \calI_d)\in \bbI}\max_{j \in [d]}|\calI_j|$.
\end{proposition}

\subsection{Proof of the Expressivity bound}\label{subsec:app:mainExpressivenenessStatements}

\begin{proof}[Proof of \Cref{thm:SNN_expressivity}]
    We write $p_0 := d$.
    By definition, 
    $\bbf_{\theta} \in \SNN_h(L,\mathbf{p}, s)$ is parametrized by a tuple of weight matrices $\theta := (W_1,W_2,\ldots, W_L) \in \Theta$ with
    $$\Theta := \Big\{(W_1,W_2,\ldots, W_L) : W_{\ell} \in \RR^{p_{\ell} \times p_{\ell-1}} \text{ for } \ell \in [L], \textstyle \sum_{\ell\in [L]} \|W_{\ell}\|_0 \leq s \Big\}.$$
    Let $D:=\smash{\sum_{\ell\in [L]}} p_{\ell-1} p_{\ell}$ be the total number of parameters in $\SNN_h(L,\mathbf{p}, s)$. 
    For each choice of $D-s$ zero positions among the $D$ parameters, we define $\Theta_m \subset \Theta, \quad m = 1,2,\ldots, \tbinom{D}{s}$ 
    as the subset of weight tuples whose zero entries include those $D-s$ positions.
    Then, 
    \begin{align*}
        \bigcup_{m=1}^{\tbinom{D}{s}} \Theta_m = \Theta,
    \end{align*}
    and $\{\Theta_m, m \in [\tbinom{D}{s}]\}$ are not disjoint since networks with fewer than $s$ nonzero entries belong to multiple sets $\Theta_i$.   
    In the following, we will show that $m \in [\tbinom{D}{s}]$,
    \begin{align}\label{eq:to_prove_SNN_expressivity}
       \Big|\left\{\bbB_w\colon \bbI \to \mathbb{T}^{p_L}, (\calI_1, \dots, \calI_{d})\mapsto \bbf_{\theta}(\calI_1, \dots, \calI_{d}) \colon \theta \in \Theta_m  \right\}\Big|
       \leq (8 e T_{\operatorname{sum}})^{s}.
    \end{align}
    Then, from union bound, combined with $\tbinom{D}{s} \leq D^s$ and
    $$D = \textstyle \sum_{\ell=1}^{L} p_{\ell-1} p_{\ell}
    \leq \sum_{\ell=1}^{L} (p_{\ell-1} \vee p_{\ell})^2 
    \leq \left(\sum_{\ell=1}^{L} (p_{\ell-1} \vee p_{\ell})\right)^2
    \leq s^2,$$
    the assertion follows.

    To prove \eqref{eq:to_prove_SNN_expressivity}, we denote by $\calI_{\ell, j}$ the spike train of the $j$-th node of the $\ell$-th layer for $\ell \in \{0,1,\ldots,L\}$ and $j \in [p_{\ell}]$. 
    We fix the sparsity pattern $m \in [\tbinom{D}{s}].$
    For $\ell \in [L]$ and $j \in [p_{\ell}]$,
    $s_{\ell, j}$ denotes the number of non-zero parameters in the $j$-th column of $W_{\ell}$.
    Then, $\calI_{\ell,j}$ is produced by a SNN computational unit receiving input from $s_{\ell,j}$ spike trains among $\{\calI_{\ell-1,1}, \ldots, \calI_{\ell-1,p_{\ell-1}}\}$.

    Recall the definition of $\calP$ in \eqref{eq:potential_vectors_class}. %
    Since   
    $|\calP(\calI^{(i)}_1, \dots, \calI^{(i)}_{d})| \leq (T_i+1)^2$ for every $i \in [n]$, we have
    $|\calP(\mathbb{I})| \leq T_{\operatorname{sum}}.$ 
    By \Cref{thm:single_neuron_expressivity}, the number of possible $1$-layer SNN induced functions $(\calI_{0,1}, \ldots, \calI_{0,{d}}) \mapsto (\calI_{1,1}, \ldots, \calI_{1,{p_1}})$ on the domain $\mathbb{I}$ is bounded by
    $$ \prod_{j=1}^{p_1}  (8eT_{\operatorname{sum}})^{s_{1,j}} =  (8eT_{\operatorname{sum}})^{\sum_{j=1}^{p_1} s_{1,j}}.$$

    We now proceed by induction from the first to the $L$-th SNN layer.
    Assume that for $\ell \in \{1,\ldots,L-1\}$, the total number of possible SNN induced functions 
    $(\calI_{0,1}, \ldots, \calI_{0,{d}}) \mapsto (\calI_{\ell,1}, \ldots, \calI_{\ell,{p_\ell}})$ on the domain $\mathbb{I}$ based on parameter space $\Theta_m$ is upper bounded by
    $ (8 e T_{\operatorname{sum}})^{\sum_{k=1}^\ell \sum_{j=1}^{p_k} s_{k,j}}.$
    We fix one of these possible $\ell$-layer SNN induced functions, $\boldsymbol{\psi}_{\ell} \circ \ldots \circ \boldsymbol{\psi}_{1}: \mathbb{I} \to \mathbb{T}^{p_\ell}$.
    By Corollary \ref{cor:inputPotential}, for each $i \in [n]$ the spike train $\boldsymbol{\psi}_{\ell} \circ \ldots \circ \boldsymbol{\psi}_{1}(\calI^{(i)}_1, \dots, \calI^{(i)}_{d})$ is dominated by $\bigcup_{j=1}^{d} \calI^{(i)}_j$, and hence has at most $T_i$ spikes.
    We thus get
    \begin{align*}
        &\Big|\calP\big(\{ \boldsymbol{\psi}_{\ell} \circ \ldots \circ \boldsymbol{\psi}_{1}(\calI_{0,1}, \ldots, \calI_{0,{d}}) : (\calI_{0,1}, \ldots, \calI_{0,{d}}) \in \mathbb{I} \}\big)\Big|\\ 
        &\leq \sum_{i=1}^n         
        \Big|\calP\big( \boldsymbol{\psi}_{\ell} \circ \ldots \circ \boldsymbol{\psi}_{1}(\calI^{(i)}_1, \dots, \calI^{(i)}_{d}) \big)\Big|\leq \sum_{i=1}^n (T_i+1)^2 = T_{\operatorname{sum}}.
    \end{align*}
 By \Cref{thm:single_neuron_expressivity}, the number of possible $1$-layer SNN induced functions $(\calI_{\ell,1}, \ldots, \calI_{\ell,{p_\ell}}) \mapsto (\calI_{\ell+1,1}, \ldots, \calI_{\ell+1,{p_{\ell+1}}})$ on the domain $\{ \boldsymbol{\psi}_{\ell} \circ \ldots \circ \boldsymbol{\psi}_{1}(\calI_{0,1}, \ldots, \calI_{0,{d}}) : (\calI_{0,1}, \ldots, \calI_{0,{d}}) \in \mathbb{I} \}$ is upper bounded by
    $$ \prod_{j=1}^{p_{\ell+1}}  (8eT_{\operatorname{sum}})^{s_{\ell+1,j}} =  (8eT_{\operatorname{sum}})^{\sum_{j=1}^{p_{\ell+1}} s_{\ell+1,j}}.$$
   The number of $(\ell+1)$-layer SNN induced functions 
    $(\calI_{0,1}, \ldots, \calI_{0,{d}}) \mapsto (\calI_{\ell+1,1}, \ldots, \calI_{\ell+1,{p_{\ell+1}}})$ on domain $\mathbb{I}$ and parameter space $\Theta_m$ is therefore upper bounded by
    $$ (8 e T_{\operatorname{sum}})^{\sum_{k=1}^{\ell+1} \sum_{j=1}^{p_k} s_{k,j}}.$$

    Iterating this process until $\ell = L$, the total number of possible $L$-layer SNN induced functions 
    $(\calI_{0,1}, \ldots, \calI_{0,{d}}) \mapsto (\calI_{L,1},  \ldots, \calI_{L,{p_L}})$ on domain $\mathbb{I}$ and for the parameter space $\Theta_m$ is upper bounded by 
    \begin{align*}
         (8 e T_{\operatorname{sum}})^{\sum_{\ell=1}^L \sum_{j=1}^{p_\ell} s_{\ell,j}} = (8 e T_{\operatorname{sum}})^{s},
    \end{align*}
    which proves \eqref{eq:to_prove_SNN_expressivity} and hence the assertion.
\end{proof}

\subsection{Proof of the Cardinality Bound on Causal Functions}\label{subsec:app:proof_cardinality}

\begin{proof}[Proof of \Cref{lem:cardinality_causal_functions_2}] 
If $(\calI_1, \ldots, \calI_d) \in \II$, 
then for each $t \in [m]$ at least one of the spike trains $\calI_1,\ldots, \calI_d$ includes $t$, and hence the number of possible values of $(\mathds{1}(t \in \calI_j))_{j \in [d]}$ is $2^d-1$.
This yields $|\II| = (2^d-1)^m$.
 We enumerate the elements of $\II$ by $\mathbf{I}_{1}, \ldots, \mathbf{I}_{(2^d-1)^m}$.  

For $k \in [(2^d - 1)^m]$, we define a map $\nu_{k} : \FF_{\mathrm{mm}}(d,m,r) \to \{0,1\}$ as 
$$\nu_{k}(F) := \mathds{1} \Big(m \in F(\mathbf{I}_{k}) \Big).$$
Based on this perspective, it follows that 
\begin{enumerate}
    \item For any $\bb \in \{0,1\}^{(2^d - 1)^m}$ with $\|\bb\|_1 \leq r$, there exists $F \in \FF_{\mathrm{mm}}(d,m,r)$ such that $(\nu_{k}(F))_{k \in [(2^d - 1)^m]} = \bb$. 
    \item For $F_1, F_2 \in \FF_{\mathrm{mm}}(d,m,r)$, if 
    $$\big(\nu_{k}(F_1)\big)_{k \in [(2^d - 1)^m]} \neq \big(\nu_{k}(F_2)\big)_{k \in [(2^d - 1)^m]},$$
    then there exists $\bI \in \II$ such that $F_1(\bI) \neq F_2(\bI)$.
\end{enumerate}
Moreover, we have
\begin{align}
    \Big|\big\{\bb \in \{0,1\}^{(2^d - 1)^m}: \|\bb\|_1 \leq r\big\} \Big|
    &= \sum_{i=0}^r \binom{(2^d - 1)^m}{i} \nonumber \\
    &= \sum_{i=0}^r \binom{(2^d - 1)^m}{i} \Big(\mathds{1}\big(r \leq (2^{d}-1)^{m}\big) + \mathds{1}\big(r \geq (2^{d}-1)^{m}+1\big)\Big)\nonumber \\
    &\geq 
    2^r \land 2^{(2^{d}-1)^{m}}. \label{tmp_001}
\end{align}
By $(i)$, $(ii)$ and \eqref{tmp_001}, we obtain the assertion. \qedhere
\end{proof}

\subsection{Proofs for Auxiliary Statements on Expressiveness}\label{app:proofs_technical_lemmas}

\begin{proof}[Proof of \Cref{lem:meta_potential_formula}]
    The proof follows by unrolling the definition of the membrane  potential in \eqref{eq:potential_definition}. 
    Specifically, we show via (reverse) induction for every $s \in \{j+1, \dots, r\}$,  
    \begin{align*} 
        P(\tau_r) = \bigg( \mathds{1}\big(P(\tau_{s-1}) \leq 1 \big)
        P(\tau_{s-1}) \exp\Big(-\frac{\tau_r - \tau_{s-1}}h\Big) + \sum_{k =s}^{r} \sum_{l\colon \tau_k \in \calI_l} w_l \exp\Big(-\frac{\tau_r - \tau_k}h\Big)\bigg)_+.
    \end{align*}
    Choosing $s = j+1$ yields the desired result. 

    The case $s = r$ follows directly from \eqref{eq:potential_definition}. For the induction step, we assume that the claim holds for some $s\in \{j+2, \dots, r\}$, and we need to show that it also holds for $s-1$. Since $s-1 \geq j+1$ we know by definition of $j$ that $P(\tau_{s-1}) \in (0,1]$ and this asserts 
    \begin{align*}\textstyle 
        \mathds{1}\big(P(\tau_{s-1}) \leq 1 \big) P(\tau_{s-1}) &= P(\tau_{s-1}) \\*
        &= 
        \mathds{1}\big(P(\tau_{s-2}) \leq 1 \big)
        P(\tau_{s-2})\exp\Big(-\frac{\tau_{s-1}-\tau_{s-2}}h\Big) + \sum_{l\colon \tau_{s-1} \in \calI_l} w_l,
    \end{align*} 
    where we can drop the expression $(\cdot)_+$ as the argument is positive. Plugging this into the induction hypothesis yields the claim for $s-1$ and finishes the proof by mathematical induction: 
    \begin{align*}
        \textstyle  P(\tau_r) = \Bigg(&\left(\mathds{1}\big(P(\tau_{s-2}) \leq 1 \big) P(\tau_{s-2}) \exp\Big(-\frac{\tau_{s-1} - \tau_{s-2}}h\Big) + \sum_{l\colon \tau_{s-1} \in \calI_l} w_l  \right)\exp\Big(-\frac{\tau_r - \tau_{s-1}}h\Big)\\* \textstyle 
         & + \sum_{k =s}^{r}\sum_{l\colon \tau_k \in \calI_l} w_l \exp\big(-\frac{\tau_r - \tau_k}h\Big) \Bigg)_+\\\textstyle 
            = \Bigg(&\mathds{1}\big(P(\tau_{s-2}) \leq 1 \big) P(\tau_{s-2}) \exp\Big(-\frac{\tau_r - \tau_{s-2}}h\Big) + \sum_{k =s}^{r}\sum_{l\colon \tau_k \in \calI_l} w_l \exp\Big(-\frac{\tau_r - \tau_k}h\Big) \Bigg).   \qedhere
   \end{align*}
\end{proof}

\begin{proof}[Proof of \Cref{lem:identifiability}]
    We prove the claim by mathematical induction with respect to the collection of all input spike times $\calI = \bigcup_{j = 1}^{d} \calI_j \subset (0, \infty)$. Indeed, by \Cref{prop:potential_well_defined} an output spike can only occur at an input spike time $\tau \in \calI$. It thus suffices to compare whether the membrane potentials from \eqref{eq:potential_definition} at $\tau$ for $\phi_{h,\mathbf{w}_1}$ and $\phi_{h,\mathbf{w}_2}$ either both exceed $1$ or are both below $0$. We enumerate the elements of $\calI$ by $0<\tau_1 < \tau_2 < \dots$, set $\tau_0 \coloneqq 0$, and denote the potential for the output neuron associated to $\phi_{h,\mathbf{w}_i}$ as $P_i$ for each $i \in \{1, 2\}$.

    For the base case, we consider the time point $\tau_1$, and we need to show that $P_1(\tau_1)>1$ if and only if $P_2(\tau_1)>1$ and $ P_1(\tau_1)= 0$ if and only if $P_2(\tau_1)= 0$. To this end, note that 
    \begin{align*}
       \textstyle  P_1(\tau_1) =
       \left(\langle \bw_1, \bp^{[\tau_0, \tau_1]}\rangle\right)_+, \quad 
       \textstyle  P_2(\tau_1) =
       \left(\langle \bw_2, \bp^{[\tau_0, \tau_1]}\rangle\right)_+,
    \end{align*}
    and the desired equivalence follows from the assumption in \eqref{eq:identifiability_condition}. 

    Let us now suppose that $\chi( P_1(\tau_i))_{i = 1, \dots, k-1} = \chi(P_2(\tau_i))_{i = 1, \dots, k-1}$ for some $k\geq 2$, and thus that the two output spike trains are identical up to time $\tau_{k-1}$. We need to show that the output spike trains are also identical at time $\tau_k$. 
    To this end, consider 
    $$j\coloneqq \max\big(i \in \{0,1,\ldots,k-1\} \colon P_1(\tau_i) \notin (0,1]\big).$$ 
    Hence, at $\tau_j$ either $ P_1(\tau_j) = 0$ or $ P_1(\tau_j)>1$. In the first case, we have by the induction hypothesis that $ P_2(\tau_j)=0$, 
    whereas in the second case we have $ P_2(\tau_j)>1$. In both cases, we thus have that $P_2(\tau_{j}) \notin (0,1]$. Moreover, for every subsequent time point $i \in \{j+1, \ldots, k-1\}$ it holds that $P_1(\tau_i)\in (0,1]$ and thus by the induction hypothesis also $P_2(\tau_i)\in (0,1]$. 
    Hence, we obtain
    $$\max\big(i \in \{0,1,\ldots,k-1\} \colon P_2(\tau_i) \notin (0,1]\big) = j.$$     
     By applying \Cref{lem:meta_potential_formula}  for both $\phi_{h,\mathbf{w}_1}$ and $\phi_{h,\mathbf{w}_2}$ at time $\tau_k$ (i.e., $r=k$),  we get
    \begin{align*}
        P_1(\tau_k) = \left(\langle \mathbf{w}_1, \bp^{[\tau_{j+1}, \tau_k]}\rangle\right)_+, \quad P_2(\tau_k) = \left(\langle \mathbf{w}_2, \bp^{[\tau_{j+1}, \tau_k]}\rangle\right)_+.
    \end{align*}
    The desired equivalence at time $\tau_k$ thus follows from \eqref{eq:identifiability_condition}, completing the induction. 
\end{proof}

\begin{proof}[Proof of \Cref{thm:single_neuron_expressivity}]
    By \Cref{lem:identifiability}, two SNN computational units $\phi_{h,\mathbf{w}_1}$ and $\phi_{h,\mathbf{w}_2}$ with weight vectors $\mathbf{w}_1, \mathbf{w}_2 \in \RR^d$ have identical output spike trains for all combinations of input spike trains $(\calI_1, \ldots, \calI_d) \in \bbI$, and thus satisfy $\phi_{h,\mathbf{w}_1} = \phi_{h,\mathbf{w}_2}$ on $\bbI$, if 
    \begin{align}
        \chi(\langle \mathbf{w}_1, \bp\rangle) = \chi(\langle \mathbf{w}_2, \bp\rangle) \quad \text{ for all } \bp\in \calP(\bbI),\label{eq:identifiability_condition_2}
    \end{align}
    where $\chi(x) = \mathds{1}(x>1)-\mathds{1}(x\leq 0)\in \{-1,0,1\}$ and where $\calP(\bbI)$ is defined in \eqref{eq:potential_vectors_class}. 

   Condition \eqref{eq:identifiability_condition_2} has a specific geometric interpretation. For each $\bp\in \calP(\bbI)$ we associate two hyperplanes $\calH_0(\bp)$ and $\calH_1(\bp)$ in $\RR^d$ defined by $\calH_i(\bp)\coloneqq \{ \mathbf{w} \in \RR^d \colon \langle \mathbf{w}, \bp\rangle = i\}$ for $i\in \{0, 1\}$. 
   These sets define proper hyperplanes because $\bp \neq \mathbf{0}$ for all $\bp\in \calP(\bbI)$, since each vector $\bp$ involves at least one contribution from a spike.  
   The hyperplanes divide the space into three regions: the region $\calR_{-1}(\bp)\coloneqq \{\mathbf{w}\in \RR^d \colon \langle \mathbf{w}, \bp\rangle \leq 0\}$ below $\calH_0(\bp)$, 
   the region $\calR_{0}(\bp)\coloneqq \{\mathbf{w} \in \RR^d \colon 0<\langle \mathbf{w}, \bp\rangle \leq 1\}$ between the two hyperplanes, and the region $\calR_{1}(\bp)\coloneqq \{\mathbf{w} \in \RR^d\colon 1<\langle \mathbf{w}, \bp\rangle\}$ above $\calH_1(\bp)$. The condition on $\mathbf{w}_1$ and $\mathbf{w}_2$ in \eqref{eq:identifiability_condition_2} can thus be rephrased as follows: For every $\bp\in \calP(\bbI)$, the weight vectors $\mathbf{w}_1$ and $\mathbf{w}_2$ must either both lie in $\calR_{-1}(\bp)$, or both lie in $\calR_{0}(\bp)$, or both lie in $\calR_{1}(\bp)$. 
   
   Hence, the number of different SNN units $\phi_{h,\mathbf{w}}$ acting on $\bbI$ for $\mathbf{w} \in \RR^d$ is upper bounded by the number of different ways to assign each $\bp \in \calP(\bbI)$ to one of the three regions $\calR_{-1}(\bp)$, $\calR_{0}(\bp)$, or $\calR_{1}(\bp)$. Equivalently, it reduces the problem to counting the number of different intersections of the form $\bigcap_{\bp \in \calP(\bbI)} \calR_{r(\bp)}(\bp)$ for $r(\bp)\in \{-1, 0, 1\}$. Consequently, we have
    \begin{align*}
        &\left|\left\{\bbB_{\bw}\colon \bbI \to \mathbb{T}, (\calI_1, \dots, \calI_d)\mapsto \phi_{h,\mathbf{w}}(\calI_1, \dots, \calI_d), \mathbf{w} \in \RR^d\right\}\right|\\
        &\leq \left|\left\{ \bigcap_{\bp \in \calP(\bbI)} \calR_{r(\bp)}(\bp) \;\colon\; r\colon \calP(\bbI)\to \{-1,0,1\}\right\}\right|.
    \end{align*}
    The latter is upper bounded by so-called  hyperplane arrangements, see \Cref{lem:bound_Halfspaces_Hyperplanes} below. By e.g.\ \citet[Theorem 28.1.1]{halperin2017arrangements}, the number of different regions formed by the hyperplanes $\calH_0(\bp)$ and $\calH_1(\bp)$ for $\bp \in \calP(\bbI)$
    is upper bounded by 
    \begin{align*}
        \sum_{k = 0}^{d} \sum_{i = 0}^{k} \binom{d-i}{k-i} \binom{2H}{d-i}  =\sum_{i = 0}^{d} \binom{2r}{d-i}  \sum_{k = i}^{d}  \binom{d-i}{k-i}  &=\sum_{i = 0}^{d} \binom{2H}{d-i}  2^{d-i} \notag =\sum_{i = 0}^{d} \binom{2H}{i}  2^{i} \notag \\&
        \leq  2^{d} \sum_{i = 0}^{d} \binom{2H}{i}\notag
           \leq (d+1) \left(\frac{4eH}{H \land d}\right)^{d} %
         \leq (8eH)^{d},  
    \end{align*}
where for the second inequality we used that for all $i\in \{0, \dots, d\}$,
$$\binom{2H}{i}\leq \binom{2H}{d} \mathds{1}(H \geq d) + \binom{2H}{H} \mathds{1}(H < d) = \binom{2H}{H \land d} \leq \left(\frac{2eH}{H \land d}\right)^{H \land d} \leq \left(\frac{2eH}{H \land d}\right)^{d}.$$
The last inequality follows by noticing that  $(d+1)/d^{d} \leq 2^d$ and $d+1 \leq 2H^d$ for $d, H \in \NN$. 
This proves the assertion. 
\end{proof}
\begin{lemma}\label{lem:bound_Halfspaces_Hyperplanes}
The number of partition cells formed by intersections of given half-spaces $\calH_1, \dots, \calH_r$ in $\RR^d$ and their complements is upper bounded by the number of different (connected) non-empty regions formed by intersecting the hyperplanes $\partial \calH_1, \dots, \partial \calH_r$.
\end{lemma}

\begin{proof}[Proof of \Cref{lem:bound_Halfspaces_Hyperplanes}]
    For $i\in [r]$, define the sets $\calH_i^{0} \coloneqq {\calH_i^{c}},$ $\calH_i^{1}\coloneqq \calH_i,$ $\tilde \calH_i^{-1}\coloneqq \textup{int}(\calH_i^{c})$, $\tilde \calH_i^{0}\coloneqq \partial \calH_i$, and $\tilde \calH_i^{1}\coloneqq \calH_i$, where the redundant definition $\calH_i^{1} = \tilde \calH_i^{1}=\calH_i$ is chosen for notational convenience below. Next, define 
    \begin{align*}\textstyle 
        \calA \coloneqq \left\{ \alpha \in \{0,1\}^{r} \;\colon\; \bigcap_{i=1}^{r} \calH_i^{\alpha_i} \neq \emptyset \right\}, \quad \calB \coloneqq \left\{ \beta \in \{-1,0, 1\}^{r} \;\colon\; \bigcap_{i=1}^{r} \tilde \calH_i^{\beta_i} \neq \emptyset\right\},
    \end{align*}
    where $\calA$ (resp.\ $\calB$) encodes all non-empty, and thus convex and connected, intersections of $\calH_i^\gamma$ for $\gamma \in \{0,1\}$ (resp.\ $\tilde \calH_i^\gamma$ for $\gamma \in \{-1,0,1\}$). By $\calH_i^{0} = \tilde \calH_i^{0}\dot \cup \tilde \calH_i^{-1}$ for all $i \in\{1, \dots r\}$, for every $\alpha \in \calA$, we find
    \begin{align*}
        \emptyset \neq \textstyle \bigcap_{i\in [r]} \calH_i^{\alpha_i} = \bigcup_{\substack{\beta_1,\ldots,\beta_r \in \{-1,0,1\}\\ \mathds{1}(\beta_i= 1) = \alpha_i}} \, \bigcap_{i =1}^{r} \, \tilde \calH_i^{\beta_i}
    \end{align*}
    and thus there exists some $\beta_1(\alpha),\ldots,\beta_r(\alpha) \in \{-1,0,1\}$ such that $\emptyset \neq \bigcap_{i =1}^{r} \tilde \calH_i^{\beta_i}\subseteq \bigcap_{i =1}^{r} \calH_i^{\alpha_i}$. 
    Further, since two distinct $\alpha, \alpha'\in \calA$ lead to disjoint sets $\bigcap_{i =1}^{r} \calH_i^{\alpha_i}$ and $\bigcap_{i =1}^{r} \calH_i^{\alpha_i'}$, it follows that  $(\beta_1(\alpha),\ldots,\beta_r(\alpha))\neq (\beta_1(\alpha'),\ldots,\beta_r(\alpha'))$, 
    and we conclude that there is an injection $\iota \colon \calA \mapsto \calB$, asserting $|\calA| \leq |\calB|$. 
\end{proof}

\section{Proofs for Additional Statements from Section 3}\label{app:proofs_section3}

\begin{lemma}\label{lem:finite_and_markovian}
    Let $F:\TT \to\TT$ be an \inputdominated, causal function satisfying the monotone scaling property. Then, the following implications hold. 
    \begin{enumerate}
        \item If $F$ is $m$-finite, then $F$ has $(m+1)$-Markovian memory.
        \item Conversely, if $F$ has $(m+1)$-Markovian memory, then either $F$ or its input-wise complement $\overline F \colon \TT\to \TT, \calI \mapsto \calI\backslash F(\calI)$ is $m$-finite.
    \end{enumerate}
\end{lemma}

\begin{proof}
    By virtue of the reduction in \eqref{eq:single_input_reduction}, the function $F$ is identified by  $F(\NN)$. In particular, $m$-finiteness is then equivalent to $F(\NN) \subseteq [m]$, while $(m+1)$-Markovian memory is equivalent to the property that for all $t\in\NN$ one has $t\in F(\NN)$ if and only if $t\in F((t-(m+1), t]\cap \NN)$. 

    For $(i)$, note for any $t\in \NN$,  $t\leq m+1$ that $[t] = (t-(m+1), t] \cap \NN$, hence by causality it holds $t\in F(\NN)$ if and only if $t\in F(\{1, \dots, t\}) = F((t-(m+1), t]\cap \NN)$. 
    Moreover, by $m$-finiteness we have $t\notin F(\NN)$ for every $t\geq m+1$. Hence, it follows for $t\geq m+1$ that by monotone scaling that $t\notin F((t-(m+1), t]\cap \NN) = F([m+1])\oplus \{t-(m+1)\}$, and in particular, 
        $\mathds{1}(t\in F(\NN)) = 0 = \mathds{1}(t\in F((t-(m+1), t]\cap \NN)),$
    which confirms the $(m+1)$-Markov property.
    
    To show $(ii)$, observe that both $F$ and $\overline F$ are $(m+1)$-Markovian, and assume without loss of generality that $m+1\notin F(\NN)$ (otherwise, the argument applies to $\overline F$). Then, by $(m+1)$-Markovian memory and monotone scaling, it holds for any $t\geq m+1$ that $F(\NN) = F((t-(m+1), t]\cap \NN) = F([m+1])\oplus \{t-(m+1)\}.$
    This spike train does also not contain $t$ since $m+1\notin F([m+1]).$ Hence, $F(\NN) \subseteq [m]$. \qedhere
\end{proof}

\begin{lemma}\label{lem:mm_periodic_composition}
    Let $F_M\colon \TT^2 \to \TT$ be $m_M$-Markovian and $F_P\colon \TT\to \TT$ be $m_P$-periodic, and define the input dominated, causal, monotone scaling function $F\colon \TT \to \TT, F(\calI) := F_M(\calI, F_P(\calI))$. Then, 
    it follows 
    for all $t\ge m_M+1$ that 
    $t\in F(\NN)$ if and only if $t+m_P\in F(\NN)$.
\end{lemma}

\begin{proof}
    By \eqref{eq:single_input_reduction} it suffices to consider $\calI := \NN$. For $t>m_M$ define $\calI_t := \NN\cap(t-m_M,t]$ and $\calJ_t := F_P(\NN)\cap(t-m_M,t]$. By the $m_M$-Markovian property of $F_M$ we have $t\in F_M(\NN,F_P(\NN))$ if and only if $t\in F_M(\calI_t,\calJ_t)$. Further, note $\calI_{t+m_P} = \calI_t\oplus\{m_P\}$ and, since $F_P$ is $m_P$-periodic, it holds $\calJ_{t+m_P} = \calJ_t\oplus\{m_P\}$. Moreover, 
the monotone scaling property of $F_M$ implies $F_M(\calI_{t+m_P},\calJ_{t+m_P}) = F_M(\calI_t,\calJ_t)\oplus\{m_P\}$, so $t+m_P\in F_M(\calI_{t+m_P},\calJ_{t+m_P})$ if and only if $t\in F_M(\calI_t,\calJ_t)$. Applying again the $m_M$-Markovian property at time $t+m_P$ yields $t+m_P\in F(\NN)$ if and only if $t+m_P\in F_M(\calI_{t+m_P},\calJ_{t+m_P})$. Combining these equivalences gives 
    $t\in F(\NN)$ if and only if $t+m_P\in F(\NN)$ for all $t\ge m_M+1$.
\end{proof}

\section*{Acknowledgments and Disclosure of Funding} 
Shayan Hundrieser acknowledges funding from the German National Academy of Sciences Leopoldina under grant number LPDS 2024-11. Philipp Tuchel is supported by the DFG (project number 516672205). Insung Kong and Johannes Schmidt-Hieber are supported by ERC grant A2B (grant agreement number 101124751). 
We would like to thank Niklas Dexheimer and Sascha Gaudlitz for helpful discussions in the early phase of this project.

\end{document}